\documentclass[10pt, oneside]{article} 
\usepackage{geometry}
\usepackage[utf8]{inputenc} % allow utf-8 input
\usepackage[T1]{fontenc}    % use 8-bit T1 fonts
\usepackage{amssymb}% hyperlinks
\usepackage{hyperref}
\usepackage{url}            % simple URL typesetting      
\usepackage{amsfonts}       % blackboard math symbols
\usepackage{nicefrac}       % compact symbols for 1/2, etc.
\usepackage{graphicx}				% Use pdf, png, jpg, or eps§ with pdflatex; use eps in DVI mode
\usepackage{setspace}							% TeX will automatically convert eps --> pdf in		
\usepackage{enumitem}
\usepackage{amssymb}
\usepackage{amsmath}
\usepackage{float}
\usepackage{mathtools}
\usepackage{tikz}
\usepackage{algorithm}
\usepackage{algorithmic}
\usepackage{amsthm}
\usepackage{authblk}

\newtheorem{theorem}{Theorem}

\theoremstyle{plain}
\newtheorem{proposition}{Proposition}

\newtheorem{lemma}{Lemma}
\newtheorem{corollary}{Corollary}

\theoremstyle{plain}
\newtheorem{assumption}{Assumption}

\theoremstyle{plain}
\newtheorem{definition}{Definition}

\theoremstyle{remark}

\theoremstyle{definition}

\DeclareMathOperator*{\argmin}{arg\,min}

\DeclareMathOperator{\sign}{sign}

\DeclarePairedDelimiter\abs{\lvert}{\rvert}%
\DeclarePairedDelimiter\norm{\lVert}{\rVert}%
\DeclareMathOperator{\erfc}{erfc}
\DeclareRobustCommand{\deq}[0]{\overset{\rm d}{=}}

\begin{document}

\title{Asymptotic Errors for Teacher-Student Convex Generalized Linear Models \\ (or : How to Prove Kabashima's Replica Formula)}

\author{Cedric Gerbelot\footnote{cedric.gerbelot@cims.nyu.edu},
Alia Abbara\footnote{alia.abbara@epfl.ch},
Florent Krzakala\footnote{florent.krzakala@epfl.ch}\\
\begin{small}
$^{*}$ Courant Institute of Mathematical Sciences, NYU, New York, USA \end{small}\\ 
$^{\dagger}$\begin{small}Laboratoire IBI-SV, EPFL, Switzerland\end{small} \\
\begin{small}$^{\ddagger}$IdePHICS Laboratory, EPFL, Switzerland\end{small}}
\date{}
\maketitle

\begin{abstract}
There has been a recent surge of interest in the study of asymptotic reconstruction performance in various cases of generalized linear estimation problems in the teacher-student setting, especially for the case of i.i.d standard normal matrices. Here, we go beyond these matrices, and prove an analytical formula for the reconstruction performance of convex generalized linear models with rotationally-invariant data matrices with arbitrary bounded spectrum, rigorously confirming, under suitable assumptions, a conjecture originally derived using the replica method from statistical physics. The proof is achieved by leveraging on message passing algorithms and the statistical properties of their iterates, allowing to characterize the asymptotic empirical distribution of the estimator. For sufficiently strongly convex problems, we show that the two-layer vector approximate message passing algorithm (2-MLVAMP) converges, where the convergence analysis is done by checking the stability of an equivalent dynamical system, which gives the result for such problems. We then show that, under a concentration assumption, an analytical continuation may be carried out to extend the result to convex (non-strongly) problems. We illustrate our claim with numerical examples on mainstream learning methods such as sparse logistic regression and linear support vector classifiers, showing excellent agreement between moderate size simulation and the asymptotic prediction.
\end{abstract}

\section{Introduction}

\subsection{Background and motivation}
In the modern era of statistics and machine learning, data analysis often requires solving high-dimensional estimation problems with a very large number of parameters. Developing algorithms for this task and understanding their limitations has become a major challenge. In this paper, we consider this question in the framework of supervised learning under the teacher-student scenario: (i) the data is synthetic and labels are generated by a \textquotedblleft teacher\textquotedblright rule and (ii) training is done with a convex Generalized Linear Model (GLM) . Such problems are ubiquitous in  machine learning, statistics, communications, and signal processing. 

The study of asymptotic (i.e. large-dimensional) reconstruction performance of generalized linear estimation in the teacher-student setting has been the subject of a significant body of work over the past few decades \cite{seung1992statistical,watkin1993statistical,engel2001statistical,bayati2011lasso,el2013robust,donoho2016high,zdeborova2016statistical}, 
and is currently witnessing a renewal of interest, especially for the case of identically and independently distributed (i.i.d.) standard normal data matrices, see e.g. \cite{sur2019likelihood,hastie2022surprises,mei2022generalization}. The aim of this paper is to provide a general analytical formula describing the reconstruction performance of such convex generalized linear models, but for a broader class of more adaptable matrices.

The problem is defined as follows: we aim at  reconstructing a given i.i.d. weight vector $\mathbf{x}_{0} \in \mathbb{R}^{N}$ from outputs $\mathbf{y} \in \mathbb{R}^{M}$
generated using a training set $(\mathbf{f}_{\mu})_{\mu=1,...,M}$ and the \textquotedblleft teacher\textquotedblright \thickspace rule:
\begin{equation}
  \label{teacher}
  \mathbf{y} = \varphi(\mathbf{F}\mathbf{x}_{0},\mathbf{\omega_{0}})
\end{equation}
where $\varphi$ is a proper, closed, continuous function and $\mathbf{\omega_{0}} \sim \mathcal{N}(0,\Delta_{0}\mbox{Id})$ is an i.i.d. noise vector. To go beyond the Gaussian i.i.d. case tackled in a majority of theoretical works, we shall allow matrices of arbitrary spectrum. We consider the data matrix 
$\mathbf{F} \in \mathbb{R}^{M \times N}$, obtained by concatenating the vectors of the training set, to be \emph{rotationally invariant}: its singular value decomposition reads
$\mathbf{F} = \mathbf{U}\mathbf{D}\mathbf{V}^{T}$ where $\mathbf{U} \in \mathbb{R}^{M \times M},\mathbf{V} \in \mathbb{R}^{N \times N}$ are uniformly
sampled from the orthogonal groups $O(M)$ and $O(N)$ respectively. $\mathbf{D} \in \mathbb{R}^{M \times N}$ contains the singular
values of $\mathbf{F}$ on its diagonal. Our analysis encompasses any singular value distribution with compact support. We place ourselves in the so-called high-dimensional regime, so that $M,N \to \infty$ while the ratio $\alpha \equiv M/N$ is kept finite. Our goal is to study the reconstruction performance of the generalized linear estimation method:
\begin{equation}
  \label{student}
  \mathbf{\hat{x}} \in \argmin_{\mathbf{x} \in \mathbb{R}^{N}} \left\lbrace g(\mathbf{F}\mathbf{x},\mathbf{y})+f(\mathbf{x}) \right\rbrace
\end{equation}
where $g$ and $f$ are proper, closed, convex and separable functions. This type of procedure is an instance of empirical risk minimizationa and is one of the building blocks of modern machine learning. It encompasses several
mainstream methods such as logistic regression, the LASSO or linear support vector machines. More precisely, the quantities of interest representing the reconstruction performance are the mean squared error $E = \mathbb{E}\left[\frac{1}{N}\norm{\mathbf{x}_{0}-\hat{\mathbf{x}}}_{2}^{2}\right]$ for regression problems, and the reconstruction angle $\theta_{x} = \arccos{\frac{\mathbf{x}_{0}^{T}\hat{\mathbf{x}}}{\norm{\mathbf{x}_{0}}_{2}\norm{\hat{\mathbf{x}}}_{2}}}$ for classification problems.
\subsection{Main contributions} 
\begin{itemize}[noitemsep,topsep=0pt,leftmargin=*]
\item We provide a set of equations characterizing the asymptotic statistical properties of the estimator defined by problem \eqref{student} with data generated by \eqref{teacher} in the asymptotic setup, for separable, convex losses and penalties (including for instance Logistic, Hinge, LASSO and Elastic net), for rotationally invariant sequences of matrices $\mathbf{F}$. For sufficiently strongly convex problems (in the sense of Lemma \ref{conv_lemma}), our assumptions are classical with respect to earlier work. To extend the result to convex problems however, we require a concentration assumption that we discuss further in section \ref{sec:main_res}.
  \item By doing so, we give, under the aforementioned set of assumptions, a mathematically rigorous proof, of a replica formula obtained heuristically through statistical physics for this problem, notably by Y. Kabashima\cite{kabashima2008inference}. This is a significant step beyond the setting of most rigorous work on replica results, which assume matrices to be i.i.d. random Gaussian ones.
\item Our proof method builds on a detailed mapping between alternating directions descent methods \cite{boyd2011distributed} from convex optimization and a set of algorithms called multi-layer vector approximate message-passing algorithms \cite{manoel2017multi,schniter2016vector}. This enables us to use convergence results from convex analysis and dynamical systems to study the trajectories of vector approximate message-passing algorithms.
\item Beyond the high-dimensional result on the estimator defined by the GLM, our convergence analysis provides a generic condition for the convergence of 2-layer MLVAMP, regardless of the randomness of the design matrix and of the dimensions of the problem, for sufficiently strongly convex problems.
\end{itemize}
\subsection{Related work} The simplest case of the present question, when both $f$ and $g$ are quadratic functions, can be mapped to a random matrix theory problem and solved rigorously, as in e.g. \cite{hastie2022surprises}. Handling non-linearity is, however, more challenging. A long history of research tackles this difficulty in the high-dimensional limit, especially in the statistical physics literature where this setup is common. The usual analytical approach in statistical physics of learning \cite{seung1992statistical,watkin1993statistical,engel2001statistical} is a heuristic, non-rigorous but very adaptable technique called the replica method \cite{mezard1987spin,mezard2009information}. In particular, it has been applied on many variations of the present problem, and laid the foundation of a large number of deep, non-trivial results in machine learning, signal processing and statistics, e.g.  \cite{gardner1989three,Opper1990,opper1996statistical,Kinzel2003,kabashima2009typical,ganguli2010statistical,advani2016equivalence,mitra2019compressed,emami2020generalization}. Among them, a generic formula for the present problem has been conjectured by Y. Kabashima, providing  sharp asymptotics for the reconstruction performance of the signal ${\bf x_0}$~\cite{kabashima2008inference}.

Proving the validity of a replica prediction is a difficult task altogether. There has been recent progress in the particular case of Gaussian data, where the matrix $\mathbf{F}$ is made of i.i.d. standard Gaussian coefficients. In this case, the asymptotic performance of the LASSO was rigorously derived in \cite{bayati2011dynamics}, and the existence of the logistic estimator discussed in \cite{sur2019likelihood}. A set of papers managed to extend this study to a large set of convex losses $g$, using the so-called Gordon comparison theorem \cite{thrampoulidis2018precise}. We broaden those results here by proving the Kabashima formula, valid for the set of rotationally invariant matrices introduced above and any convex, separable loss $g$ and sufficiently strongly convex regularizer $f$ under classical conditions. We extend this result to any convex, separable $g$ and $f$ under stronger assumptions.

Our proof strategy is based  on the use of approximate-message-passing~\cite{donoho2009message,rangan2011generalized}, as pioneered in \cite{bayati2011lasso}, and is similar to a recent work \cite{gerbelot2020asymptotic} on a simpler setting. This family of algorithms is a statistical physics-inspired variant of belief propagation \cite{mezard1989space,kabashima2003cdma,kabashima2004bp} where local beliefs are approximated by Gaussian distributions. A key feature of these algorithms is the existence of the state evolution equations, a scalar equivalent model which allows to track the asymptotic statistical properties of the iterates at every time step. A series of groundbreaking papers initiated with \cite{bayati2011dynamics} proved that these equations are exact in the large system limit, and extended the method to treat nonlinear problems \cite{rangan2011generalized} and handle rotationally invariant matrices \cite{rangan2019vector,takahashi2022macroscopic}. We shall use a variant of these algorithms called  multi-layer vector approximate message-passing (MLVAMP)  \cite{schniter2016vector,fletcher2018inference}.
The key technical point in our approach is an analysis of the convergence of MLVAMP. This is achieved by phrasing the algorithm as a dynamical system, and then determining sufficient conditions for convergence with linear rate. Our analysis guarantees converging trajectories above a threshold value of the strong convexity parameter of the problem, which is sufficient to complete the proof in that region. We use an analytic continuation to extend the result to convex problems, at the cost of an additional condition discussed after stating our main set of assumption.

\section{Background on MLVAMP}
\label{back_MLVAMP}
In this section, we present background on the multilayer vector approximate message-passing algorithm developed in \cite{fletcher2018inference}. In doing so, we will introduce the key quantities involved in our main theorem. MLVAMP was initially designed as a probabilistic inference algorithm in multilayer architectures. Here, we only focus on the 2-layer version for inference in GLMs, and use the notations of \cite{takahashi2022macroscopic}. The algorithm can be derived in several ways, notably from expectation-consistent variational inference frameworks such as expectation propagation \cite{minka2001family}, where the target posterior distribution is approximated by a simpler one with moment matching constraints. In the maximum a posteriori setting (MAP), the frequentist optimization framework is recovered, with additional parameter prescriptions due to the probabilistic models, as we will see below. The derivation of the algorithm is, however, not our point of interest. We focus on providing a self-contained interpretation from the convex optimization point of view, in particular in terms of variable splitting. 
\subsection{Link with variable splitting and proximal descent}
A common procedure to tackle nonlinear optimization problems involving several functions is variable splitting, so that each non-linearity may be treated independently. Augmenting the Lagrangian with a square penalty on the slack variable equality constraint leads to the family of alternating direction methods of multipliers (ADMM) \cite{boyd2011distributed}, where the objective is iteratively minimized in the direction of each initial variable and slack variable. The descent steps then take the form of proximal operators of the non-linearities. For example, on problem \eqref{student}, adding a slack variable $\mathbf{z}=\mathbf{F}\mathbf{x}$ would lead to the augmented Lagrangian:
\begin{align}
\label{ex_split}
  g(\mathbf{z},\mathbf{y})+f(\mathbf{x}) +\theta^{T}(\mathbf{z}-\mathbf{F}\mathbf{x})+\frac{\alpha}{2}\norm{\mathbf{z}-\mathbf{F}\mathbf{x}}_{2}^{2} 
\end{align}
where $\alpha>0$ is a free parameter that can enforce strong convexity of the objective if large enough and $\theta$ is a Lagrange multiplier. Updating $\mathbf{x}$ from an update on $\mathbf{z}$ amounts to a linear estimation problem, which can be solved by least squares. This is implemented, for example, in linearized ADMM \cite{boyd2011distributed}, where the proximal descent steps are coupled to least-square ones. \\
MLVAMP solves problem \eqref{student} by introducing the same splitting as in \eqref{ex_split} with an additional trivial splitting for each variable: $\mathbf{x}_{1}, \mathbf{x}_{2}, \mathbf{z}_{1}, \mathbf{z}_{2}$
    such that $\mathbf{x}_{1}=\mathbf{x}_{2}, \thickspace \mathbf{z}_{1} = \mathbf{F}\mathbf{x}_{1}, \thickspace \mathbf{z}_{2}=\mathbf{F}\mathbf{x}_{2}$.
In the convex optimization framework, parameters like gradient step sizes, or proximal parameters need to be chosen. In the expectation propagation framework, they are prescribed by expectation-consistency constraints, which leads to additional steps in the algorithm. MLVAMP thus consists in four descent steps on $\mathbf{x}_{1}, \mathbf{x}_{2}, \mathbf{z}_{1}, \mathbf{z}_{2}$, and the updates on the parameters of the functions corresponding to those descent steps. This is shown in the MLVAMP iterations (see \eqref{GVAMP} further), where $\mathbf{x}_{1}, \mathbf{z}_{1}$ are updated using the proximal operators of the loss and regularizer, while $\mathbf{z}_{2}$ and $\mathbf{x}_{2}$ are obtained through least-squares. As mentioned above, the parameters of proximal operators (or denoisers in the signal processing literature) and least-squares are set by probabilistic inference rules (here moment-matching of marginal distributions). It is shown in \cite{fletcher2016expectation} that, in the MAP setting, these updates amount to adapting the parameters to the local curvature of the cost function.
\subsection{2-layer MLVAMP and its state evolution}
We lay out the full iterations of the MLVAMP algorithm from \cite{fletcher2018inference} applied to a 2-layer network in Algorithm \ref{GVAMP}. For a given operator $T : \mathcal{X} \to \mathbb{R}^d$ where $d$ is $M$ or $N$ in our setting, the brackets $\langle T(\mathbf{x})\rangle = \frac{1}{d}\sum_{i=1}^{d} T(\mathbf{x})_{i}$ denote element-wise averaging operations. For a given matrix $\mathbf{M} \in \mathbb{R}^{d\times d}$, the brackets amount to $\langle\mathbf{M}\rangle = \frac{1}{d}\mbox{\textbf{Tr}}(\mathbf{M})$. For a given function, for example $g_{1x}$, we use the shorthand $g_{1x}(...)$ when the arguments have been made clear in a line above and are left unchanged.
\begin{algorithm}
\begin{algorithmic}
\caption{2-layer MLVAMP}
\label{GVAMP}
    \REQUIRE Initialize $\mathbf{h}_{1x}^{(0)}, \mathbf{h}_{2z}^{(0)},\hat{Q}_{1x}^{(0)},\hat{Q}_{2z}^{(0)}$, number of iterations T. 
    \FOR{t=0,1...,T}
    \vspace{0.2cm}
    \STATE \mbox{// Denoising $\mathbf{x}$} 
    \vspace{0.2cm}
    \STATE $\mathbf{\hat{x}}_1^{(t)} = g_{1x} (\mathbf{h}_{1x}^{(t)},\hat{Q}_{1x}^{(t)})$ 
    \STATE $\chi_{1x}^{(t)} = \left\langle \partial_{\mathbf{h}_{1x}^{(t)}} g_{1x} (...)\right\rangle/\hat{Q}_{1x}^{(t)}$ 
    \STATE $\hat{Q}_{2x}^{(t)} =1/\chi_{1x}^{(t)} - \hat{Q}_{1x}^{(t)}$
    \STATE $\mathbf{h}_{2x}^{(t)}=(\mathbf{\hat{x}}_1^{(t)}/\chi_{1x}^{(t)}  - \hat{Q}_{1x}^{(t)} \mathbf{h}_{1x}^{(t)}) / \hat{Q}_{2x}^{(t)}$ 
    \vspace{0.2cm}
    \STATE \mbox{// LMMSE estimation of $\mathbf{z}$}
    \vspace{0.2cm}
    \STATE $\mathbf{\hat{z}}_{2}^{(t)} = g_{2z}(\mathbf{h}_{2x}^{(t)}, \mathbf{h}_{2z}^{(t)}, \hat{Q}_{2x}^{(t)}, \hat{Q}_{2z}^{(t)})$  
    \STATE $\chi_{2z}^{(t)} = \left\langle \partial_{\mathbf{h}_{2z}^{(t)}} g_{2z}(...) \right\rangle/ \hat{Q}_{2z}^{(t)}$ 
    \STATE $\hat{Q}_{1z}^{(t)} = 1/\chi_{2z}^{(t)} -\hat{Q}_{2z}^{(t)}$ 
    \STATE $\mathbf{h}_{1z}^{(t)}=(\mathbf{\hat{z}_{2}^{(t)}}/\chi_{2z}^{(t)}  - \hat{Q}_{2z}^{(t)} \mathbf{h}_{2z}^{(t)}) / \hat{Q}_{1z}^{(t)}$ 
    \vspace{0.2cm}
    \STATE \mbox{// Denoising $\mathbf{z}$}
    \vspace{0.2cm}
    \STATE $\mathbf{\hat{z}}_1^{(t)} = g_{1z} (\mathbf{h}_{1z}^{(t)},\hat{Q}_{1z}^{(t)}),$
    \STATE $\thickspace \chi_{1z}^{(t)} =  \left\langle \partial_{\mathbf{h}_{1z}^{(t)}} g_{1z} (...)\right\rangle/\hat{Q}_{1z}^{(t)}$ 
    \STATE $\hat{Q}_{2z}^{(t+1)} =1/\chi_{1z}^{(t)} - \hat{Q}_{1z}^{(t)}$
    \STATE $\mathbf{h}_{2z}^{(t+1)}=(\mathbf{\hat{z}}_1^{(t)}/\chi_{1z}^{(t)}  - \hat{Q}_{1z}^{(t)} \mathbf{h}_{1z}^{(t)}) / \hat{Q}_{2z}^{(t+1)}$ 
    \vspace{0.2cm}
    \STATE \mbox{// LMMSE estimation of $\mathbf{x}$} 
    \vspace{0.2cm}
    \STATE $\mathbf{\hat{x}}_{2}^{(t+1)} = g_{2x}(\mathbf{h}_{2x}^{(t)}, \mathbf{h}_{2z}^{(t+1)}, \hat{Q}_{2x}^{(t)}, \hat{Q}_{2z}^{(t+1)})$ \STATE $\chi_{2x}^{(t+1)} = \left\langle \partial_{\mathbf{h}_{2x}^{(t)}} g_{2x}(...) \right\rangle/\hat{Q}_{2x}^{(t)}$ 
    \STATE $\hat{Q}_{1x}^{(t+1)} = 1/\chi_{2x}^{(t+1)} -\hat{Q}_{2x}^{(t)}$ 
    \STATE $\mathbf{h}_{1x}^{(t+1)}=(\mathbf{\hat{x}}_{2}^{(t+1)}/\chi_{2x}^{(t+1)}  - \hat{Q}_{2x}^{(t)} \mathbf{h}_{2x}^{(t)}) / \hat{Q}_{1x}^{(t+1)}$
    \ENDFOR
    \vspace{0.2cm}
    \RETURN $\hat{\mathbf{x}}_{1},\hat{\mathbf{x}}_{2}$
\end{algorithmic}
\end{algorithm}
The denoising functions $g_{1x}$ and $g_{1z}$ can be written as proximal operators in the MAP setting:

\begin{align}
    g_{1x}(\mathbf{h}_{1x}^{(t)}, \hat{Q}_{1x}^{(t)})&= \argmin_{\mathbf{x}\in \mathbb{R}^N} \left\lbrace f(\mathbf{x}) +  \dfrac{\hat{Q}_{1x}^{(t)}}{2} \norm{\mathbf{x} - \mathbf{h}_{1x}^{(t)}}_2^2 \right\rbrace \\ &=\mbox{Prox}_{f/\hat{Q}_{1x}^{(t)}}(\mathbf{h}_{1x}^{(t)})  \label{prox-f}
\end{align}
and 
\begin{align}
    g_{1z}(\mathbf{h}_{1z}^{(t)}, \hat{Q}_{1z}^{(t)})&= \argmin_{\mathbf{z}\in \mathbb{R}^M} \left\lbrace g(\mathbf{y}, \mathbf{z}) +  \dfrac{\hat{Q}_{1z}^{(t)}}{2} \norm{\mathbf{z} - \mathbf{h}_{1z}^{(t)}}_2^2 \right\rbrace\\ &=\mbox{Prox}_{g(.,\mathbf{y})/\hat{Q}_{1z}^{(t)}}(\mathbf{h}_{1z}^{(t)}). \label{prox-g}
\end{align}
The LMMSE denoisers $g_{2z}$ and $g_{2x}$ in the MAP setting read (see \cite{schniter2016vector}):
\begin{align}
    &g_{2z}(...) = \mathbf{F}\mathbf{M}_{1}^{(t)}(\hat{Q}_{2x}^{(t)} \mathbf{h}_{2x}^{(t)}+ \hat{Q}_{2z}^{(t)} \mathbf{F}^T \mathbf{h}_{2z}^{(t)}) \label{g1+} \\
    &g_{2x}(...) = \mathbf{M}_{2}^{(t)}(\hat{Q}_{2x}^{(t)} \mathbf{h}_{2x}^{(t)}+ \hat{Q}_{2z}^{(t+1)} \mathbf{F}^T \mathbf{h}_{2z}^{(t+1)}). \label{g1-}
\end{align}
where we defined the matrices $\mathbf{M}_{1}^{(t)} = (\hat{Q}_{2z}^{(t)} \mathbf{F}^T \mathbf{F}+\hat{Q}_{2x}^{(t)} \mbox{Id})^{-1}$, and $\mathbf{M}_{2}^{(t)} = (\hat{Q}_{2z}^{(t+1)} \mathbf{F}^T \mathbf{F}+\hat{Q}_{2x}^{(t)} \mbox{Id})^{-1}$.
As mentioned in the previous section, MLVAMP returns at each iteration two sets of estimators $(\mathbf{\hat{x}}_1^{(t)}, \mathbf{\hat{x}}_2^{(t)})$ and $(\mathbf{\hat{z}}_1^{(t)}, \mathbf{\hat{z}}_2^{(t)})$ which respectively aim at reconstructing the minimizer $\mathbf{\hat{x}}$ and $\mathbf{\hat{z}}=\mathbf{F}\mathbf{\hat{x}}$. At the fixed point, we have $\mathbf{\hat{x}}_1^{(t)}=\mathbf{\hat{x}}_2^{(t)}$ and $\mathbf{\hat{z}}_1^{(t)}= \mathbf{\hat{z}}_2^{(t)}$, as proven in \cite{pandit2020inference}. The intermediate vectors $\mathbf{h}_{1x}^{(t)}$, $\mathbf{h}_{2x}^{(t)}$, $\mathbf{h}_{1z}^{(t)}$ and $\mathbf{h}_{2z}^{(t)}$ have the key feature that they behave asymptotically as Gaussian centered around $\mathbf{x}_{0}$ and $\mathbf{z}_{0}=\mathbf{F}\mathbf{x}_{0}$, under the set of assumptions given in appendix \ref{SE_assumptions}. More precisely, at each iteration, they converge empirically with second order moment (PL2) towards Gaussian variables: 
\begin{subequations}
\label{SE-assumption_main}
\begin{align}
   \lim_{M,N \to \infty} \hat{Q}_{1x}^{(t)}\mathbf{h}_{1x}^{(t)} - \hat{m}_{1x}^{(t)} \mathbf{x_0} &\stackrel{PL(2)} = \sqrt{\hat{\chi}_{1x}^{(t)}} \xi_{1x}^{(t)}\\
   \lim_{M,N \to \infty} \mathbf{V}^T (\hat{Q}_{2x}^{(t)} \mathbf{h}_{2x}^{(t)} - \hat{m}_{2x}^{(t)} \mathbf{x_0}) &\stackrel{PL(2)} = \sqrt{\hat{\chi}_{2x}^{(t)}} \xi_{2x}^{(t)}
    \\
   \lim_{M,N \to \infty}  \mathbf{U}^T (\hat{Q}_{1z}^{(t)}\mathbf{h}_{1z}^{(t)} - \hat{m}_{1z}^{(t)} \mathbf{z_0}) &\stackrel{PL(2)} = \sqrt{\hat{\chi}_{1z}^{(t)}} \xi_{1z}^{(t)} \\
 \lim_{M,N \to \infty}  \hat{Q}_{2z}^{(t)}\mathbf{h}_{2z}^{(t)} - \hat{m}_{2z}^{(t)} \mathbf{z_0} &\stackrel{PL(2)} = \sqrt{\hat{\chi}_{2z}^{(t)}} \xi_{2z}^{(t)}
\end{align}
\end{subequations}
where $\xi_{1x}^{(t)}, \xi_{2x}^{(t)}, \xi_{1z}^{(t)}, \xi_{2z}^{(t)}$ are i.i.d standard normal random variables independent of all other quantities. The definition of PL(2) convergence is reminded in Appendix \ref{appendix:analysis_framework}, and we use the notation $\stackrel{PL(2)}=$ following \cite{rangan2019vector,fletcher2018inference}. We can roughly say that the $\hat{Q}, \hat{m}, \hat{\chi}$'s parameters characterize the distributions of the $\mathbf{h}$'s. Using the representation \eqref{SE-assumption_main} in the iterations of MLVAMP results in a scalar recursion that tracks the evolution of the parameters of the aforementioned Gaussian distributions. This recursion provides the so-called state evolution equations.
The existence of state evolution equations is the reason why we use 2-layer MLVAMP in our proof. Indeed, they allow the construction of iterate paths that lead to the solution of problem \eqref{teacher}, while knowing their statistical properties.
\section{Main result}
\label{sec:main_res}
Our main result characterizes the asymptotic empirical distribution of the estimator $\mathbf{\hat{x}}$ defined in~\eqref{student} with data generated by \eqref{teacher}, and of $\mathbf{\hat{z}} = \mathbf{F \hat{x}}$. We start by stating the necessary assumptions.
\begin{assumption} 
\label{main_assum}
\begin{enumerate}[label=(\alph*)]
    \item [~]
    \item the functions $f$ and $g$ are proper, closed, convex and separable functions.
    \item the cost function $g(\mathbf{F}.,\mathbf{y})+f(.)$ is coercive, i.e. $\lim_{\norm{\mathbf{x}}\to \infty}g(\mathbf{F}\mathbf{x},\mathbf{y})+f(\mathbf{x}) = +\infty$.
    \item there exists a finite constant $B_{1}$ such that $\frac{1}{N}\norm{\hat{\mathbf{x}}}^{2}_{2}\leqslant B_{1}$ almost surely as $N \to \infty$. We also assume that, for any pseudo-Lipschitz function of order $2$, if there exists a finite constant $B_{2}$ such that $\forall N \in \mathbb{N}, \frac{1}{N}\sum_{i=1}^{N}\phi(\hat{x}_{i}) \leqslant B_{2}$, then the limit $\lim_{N \to \infty}\frac{1}{N}\sum_{i=1}^{N}\phi(\hat{x}_{i})$ exists.
    \item for any $\mathbf{x} \in \mbox{dom}(f)$ and any $\mathbf{x}' \in \partial f(\mathbf{x})$, there exists a constant $C$ such that $\norm{\mathbf{x}'}_{2}\leqslant C(1+\norm{\mathbf{x}}_{2})$. The same holds for $g$ on its domain. 
    \item there exist sequences of real analytic functions $g_{\epsilon}, f_{\epsilon}$ such that for any $x$, $\lim_{\epsilon \to 0}g_{\epsilon}(x) = g(x)$, $\lim_{\epsilon \to 0}f_{\epsilon}(x) = f(x)$, and for all $\epsilon >0$, 
    $g''_{\epsilon}$ and $f''_{\epsilon}$ belong to the Schwartz space.
    \item the empirical distributions of the underlying truth $\mathbf{x}_{0}$, eigenvalues of $\mathbf{F}^T \mathbf{F}$, and noise vector $w_{0}$, respectively converge empirically with second order moments, as defined in appendix \ref{appendix:analysis_framework}, to independent scalar random variables $x_{0},w_{0},\lambda$ with distributions $p_{x_{0}}$, $p_{\lambda}$, $p_{w_{0}}$. We assume that the distribution $p_{\lambda}$ is not all-zero and has compact support. 
    \item the design matrix $\mathbf{F} = \mathbf{U}\mathbf{D}\mathbf{V}^{\top} \in \mathbb{R}^{M \times N}$ is rotationally invariant, as defined in the introduction, where the elements of the Haar distributed matrices $\mathbf{U},\mathbf{V}$ are independent of the elements of the ground truth vector $\mathbf{x}_{0}$, noise $\boldsymbol{\omega}_{0}$ and elements of $\mathbf{D}$.
    \item the solution to the set of fixed point equations \eqref{thm1-equations} exists and is unique, for any convex $g$ and $f$ verifying the assumptions above
    \item finally assume that $M,N \to \infty$ with fixed ratio $\alpha = M/N$.
\end{enumerate}
\end{assumption}
The coercivity assumption (b) ensures that the minimization problem Eq.\eqref{student} is feasible and that the estimator exists. Most machine learning cost functions verify this assumption, including any convex loss which is bounded below and regularized with a coercive term such as the $\ell_{1}$ or $\ell_{2}$ norm, see \cite{bauschke2011convex} Corollary 11.15. Non-coercive problems include unregularized logistic regression and unregularized, underspecified least-squares for example. The scaling assumptions (d) are required for the state evolution equations of the MLVAMP iteration corresponding to the optimization problem Eq.\eqref{student} to hold, as discussed in appendix \ref{SE_assumptions}.
Such conditions are often encountered in high dimensional analysis of M-estimators, see, e.g. \cite{thrampoulidis2018precise}, and are verified by the setups proposed in the experiments section. The convergence of averaged sumes of PL2 observables in assumption (c) and the analytic approximation in assumption (e) are required for our analytic continuation 
to hold, and we show that any combination of hinge, logistic and square loss with $\ell_{1}$ or $\ell_{2}$ regularization verifies the latter in Appendix \ref{analytic_continuation}, subsection \ref{subsec:app_approx_fast}. We show in Lemma \ref{main_lemma} that, for sufficiently strongly convex problems, these two assumptions are not required. The concentration assumption we require has been proven to hold for a number of convex problems with Gaussian random design regardless of the strong convexity of the problem (see the related work section), and we believe rotationally invariant matrices do not change this behaviour. However, since we are unable to prove it below the threshold value of the strong convexity parameter, it remains an assumption. Additional detail on the notion of empirical convergence is given in appendix \ref{appendix:analysis_framework}. This analysis framework is mainly due to \cite{bayati2011dynamics} and is related to convergence in Wasserstein metric as pointed out in \cite{emami2020generalization}.
We are now ready to state our main theorem.

\begin{theorem}[Fixed point equations]
  \label{main_th}
  Under assumption \ref{main_assum}, consider the ground-truth $\mathbf{x_0}$ and let $\mathbf{z_0} = \mathbf{F x_0}$, $\rho_x \equiv\norm{\mathbf{x_0}}_2^2/N$ and $\rho_z \equiv\norm{\mathbf{z_0}}_2^2/M$. 
  For a strictly convex instance of problem \eqref{student}, let $\hat{\mathbf{x}}$ be its unique solution. For a convex (non-strictly) instance of problem \eqref{student}, let $\hat{\mathbf{x}}$ be its unique least $\ell_{2}$ norm solution. Then let $\hat{\mathbf{z}} = \mathbf{F}\hat{\mathbf{x}}$. Then, for any real analytic, pseudo-Lipschitz function of order 2 $\phi$ whose second derivative belongs to the Schwartz space, the following holds
:
\begin{align}
    &\lim_{N \to \infty} \frac{1}{N}\sum_{i=1}^{N}\phi(x_{0,i},\hat{x}_{i}) \stackrel{a.s.} = \mathbb{E}[\phi(x_{0},\mbox{Prox}_{f/\hat{Q}_{1x}^{(*)}}(H_{x}))]\\ 
    &\lim_{M \to \infty}\frac{1}{M}\sum_{i=1}^{M}\phi(z_{0,i},\hat{z}_{i}) \stackrel{a.s.} = \mathbb{E}[\phi(z_{0},\mbox{Prox}_{f/\hat{Q}_{1z}^{(*)}}(H_{z}))]
\end{align}
where $H_{x}=\frac{\hat{m}_{1x}^*x_{0}+\sqrt{\hat{\chi}_{1x}^*}\xi_{1x}}{\hat{Q}^{*}_{1x}}$, $H_{z}=\frac{\hat{m}_{1z}^*z_{0}+\sqrt{\hat{\chi}_{1z}^*}\xi_{1z}}{\hat{Q}^{*}_{1z}}$ and expectations are taken with respect to the random variables $x_{0} \sim p_{x_0}$, $z_{0} \sim \mathcal{N}(0,\sqrt{\rho_z})$,  $\xi_{1x}, \xi_{1z} \sim \mathcal{N}(0,1)$. 
The parameters $\hat{Q}_{1x}^*, \hat{Q}_{1z}^*, \hat{m}_{1x}^*, \hat{m}_{1z}^*, \hat{\chi}_{1x}^*, \hat{\chi}_{1z}^*$ are determined by the fixed point of the system:
\begin{subequations}
\label{thm1-equations}
\begin{align}
 \hat{Q}_{2x}&= \hat{Q}_{1x}(\mathbb{E}\left[\eta'_{f/\hat{Q}_{1x}}\left(H_{x}\right)\right]^{-1}-1)\\
\hat{Q}_{2z} &= \hat{Q}_{1z}(\mathbb{E}\left[\eta'_{g(.,y)/\hat{Q}_{1z}}\left(H_{z}\right)\right]^{-1}-1) \\
\hat{m}_{2x}&= \frac{\mathbb{E}\left[x_{0}\eta_{f/\hat{Q}_{1x}}\left(H_{x}\right)\right]}{\rho_x \chi_{x}}-\hat{m}_{1x} \\
\hat{m}_{2z}&= \frac{\mathbb{E}\left[z_{0}\eta_{g(.,y)/\hat{Q}_{1z}}\left(H_{z}\right)\right]}{\rho_z \chi_{z}}-\hat{m}_{1z} \\
\hat{\chi}_{2x}&=\frac{\mathbb{E}\left[\eta^{2}_{f/\hat{Q}_{1x}}\left(H_{x}\right)\right]}{\chi_{x}^{2}}\\
&-\rho_x(\hat{m}_{1x}+\hat{m}_{2x})^{2}-\hat{\chi}_{1x} \\
\hat{\chi}_{2z}&=\frac{\mathbb{E}\left[\eta^{2}_{g(.,y)/\hat{Q}_{1z}}\left(H_{z}\right)\right]}{\chi_{z}^{2}}\\
&-\rho_z(\hat{m}_{1z}+\hat{m}_{2z})^{2}-\hat{\chi}_{1z}
\end{align}
\begin{align}
\hat{Q}_{1x}&= \mathbb{E}\left[\frac{1}{\hat{Q}_{2x}+\lambda \hat{Q}_{2z}}\right]^{-1}-\hat{Q}_{2x}   \\
 \vspace{1pt}
\hat{Q}_{1z} &= \alpha\mathbb{E}\left[\frac{\lambda}{\hat{Q}_{2x}+\lambda \hat{Q}_{2z}}\right]^{-1}-\hat{Q}_{2z} \\
 \vspace{1pt}
\hat{m}_{1x}&= \frac{1}{\chi_{x}}\mathbb{E}\left[\frac{\hat{m}_{2x}+\lambda \hat{m}_{2z}}{\hat{Q}_{2x}+\lambda \hat{Q}_{2z}}\right] -\hat{m}_{2x} \\
 \vspace{1pt}
\hat{m}_{1z}&= \frac{\rho_x}{\alpha \chi_{z} \rho_z}\mathbb{E}\left[\frac{\lambda(\hat{m}_{2x}+\lambda \hat{m}_{2z})}{\hat{Q}_{2x}+\lambda \hat{Q}_{2z}}\right]-\hat{m}_{2z} \\
 \vspace{1pt}
\hat{\chi}_{1x}&=\frac{1}{\chi_{x}^{2}}\mathbb{E}\left[\frac{\hat{\chi}_{2x}+\lambda\hat{\chi}_{2z}+\rho_x(\hat{m}_{2x}+\lambda\hat{m}_{2z})^{2}}{(\hat{Q}_{2x}+\lambda \hat{Q}_{2z})^{2}}\right]\\
&\hspace{2.5cm}-\rho_x(\hat{m}_{1x}+\hat{m}_{2x})^{2}-\hat{\chi}_{2x} \nonumber\\
 \vspace{1pt}
\hat{\chi}_{1z}&=\frac{1 }{\alpha \chi_{z}^{2}}\mathbb{E}\left[\frac{\lambda(\hat{\chi}_{2x}+\lambda\hat{\chi}_{2z}+\rho_x(\hat{m}_{2x}+\lambda\hat{m}_{2z})^{2})}{(\hat{Q}_{2x}+\lambda \hat{Q}_{2z})^{2}}\right]\\
&\hspace{2.7cm}-\rho_z(\hat{m}_{1z}+\hat{m}_{2z})^{2}-\hat{\chi}_{2z}, \nonumber
\end{align}
\end{subequations}
\quad \\
where $\chi_{x}=(\hat{Q}_{1x}+\hat{Q}_{2x})^{-1}$, $\chi_{z}=(\hat{Q}_{1z}+\hat{Q}_{2z})^{-1}$, and expectations are taken with respect to the random variables $x_{0} \sim p_{x_0}$, $z_{0} \sim \mathcal{N}(0,\sqrt{\rho_z})$, $y \sim \varphi(z_{0},\omega_{0})$,  $\xi_{1x}, \xi_{1z} \sim \mathcal{N}(0,1)$, and eigenvalues $\lambda \sim p_\lambda$. $\eta$ is a shorthand for the scalar proximal operator:
\begin{equation}
\eta_{\gamma f}(z) = \argmin_{x \in \mathcal{X}} \left\lbrace\gamma f(x)+\frac{1}{2}(x-z)^{2}\right\rbrace.
\end{equation}
\end{theorem}% \quad \\
The set of fixed point equations from Theorem \ref{main_th} naturally stems from the "replica-symmetric" free energy commonly used in the statistical physics community \cite{mezard1987spin,mezard2009information}. The free energy depends on a set of parameters, and extremizing it with respect to all parameters, i.e. writing the zero gradient condition for each parameter, provides the set of equations \eqref{thm1-equations}. We state this correspondence
in the following corollary to Theorem \ref{main_th} :
\begin{corollary}[The Kabashima formula] \quad \\
\label{cor_free}
The fixed point equations from theorem \ref{main_th} can equivalently be rewritten as the solution of the extreme value problem (\ref{free-energy}) defined by the replica free energy from \cite{takahashi2022macroscopic}.
\begin{figure*}[!ht]
\begin{align}
    f&=-\mathop{\rm extr}_{m_x, \chi_x, q_x, m_z, \chi_z, q_z}\{g_{\rm F} + g_{\rm G} - g_{\rm S}\}, \label{free-energy} \\
    g_{\rm F} &=\mathop{\rm extr}_{\hat{m}_{1x}, \hat{\chi}_{1x}, \hat{Q}_{1x}, \hat{m}_{1z}, \hat{\chi}_{1z}, \hat{Q}_{1z}}\left\lbrace\frac{1}{2}q_x \hat{Q}_{1x} - \frac{1}{2}\chi_x \hat{\chi}_{1x}-\hat{m}_{1x} m_x-\alpha \hat{m}_{1z} m_z+\frac{\alpha}{2}\left(q_z \hat{Q}_{1z} - \chi_z \hat{\chi}_{1z}\right)\right.\nonumber \\
    &\left.+ \mathbb{E}\left[ \phi_x (\hat{m}_{1x}, \hat{Q}_{1x}, \hat{\chi}_{1x};x_0, \xi_{1x})\right]+\alpha \mathbb{E}\left[ \phi_z(\hat{m}_{1z}, \hat{Q}_{1z}, \hat{\chi}_{1z};z_0, \xi_{1z}) \right] \right\rbrace, \nonumber \\
    g_{\rm G} &= \mathop{\rm extr}_{\hat{m}_{2x}, \hat{\chi}_{2x}, \hat{Q}_{2x}, \hat{m}_{2z}, \hat{\chi}_{2z}, \hat{Q}_{2z}}\left\lbrace\frac{1}{2}q_x \hat{Q}_{2x} - \frac{1}{2}\chi_x \hat{\chi}_{2x}- m_x \hat{m}_{2x}- \alpha m_z \hat{m}_{2z} + \frac{\alpha}{2}
    \left(q_z \hat{Q}_{2z} - \chi_z \hat{\chi}_{2z}\right)\right.\nonumber\\
    &\left.-\frac{1}{2}\left(\mathbb{E}\left[\log (\hat{Q}_{2x}+\lambda\hat{Q}_{2z})\right]-\mathbb{E}\left[\frac{\hat{\chi}_{2x} + \lambda \hat{\chi}_{2z}}{\hat{Q}_{2x} + \lambda \hat{Q}_{2z}}\right]\right. \left.-\mathbb{E}\left[\frac{\rho_x(\hat{m}_{2x} + \lambda \hat{m}_{2z})^2}{(\hat{Q}_{2x} + \lambda \hat{Q}_{2z})}\right]\right)\right\rbrace,\nonumber 
    \\
    g_{\rm S} &=  \frac{1}{2}\left(\frac{q_x}{\chi_x}- \frac{m_x^2}{\rho_x \chi_x}\right)+\frac{\alpha}{2}\left(\frac{q_z}{\chi_z}-\frac{m_z^2}{\rho_z\chi_z}\right)
    \nonumber,
\end{align}
where $\phi_x$ and $\phi_z$ are the potential functions
\begin{align}
    \phi_x(\hat{m}_{1x}, \hat{Q}_{1x}, \hat{\chi}_{1x}; x_0, \xi_{1x})=\lim_{\beta \rightarrow \infty}\dfrac{1}{\beta} \log \int e^{-\frac{\beta\hat{Q}_{1x}}{2}x^2 + \beta(\hat{m}_{1x} x_0 + \sqrt{\hat{\chi}_{1x}}\xi_{1x})x - \beta f(x)}dx, \label{phi_x}\\
    \phi_z(\hat{m}_{1z}, \hat{Q}_{1z}, \hat{\chi}_{1z};z_0, \chi_{1z}) = \lim_{\beta \rightarrow \infty}\frac{1}{\beta} \log \int e^{-\frac{\beta\hat{Q}_{1z}}{2}z^2 + \beta(\hat{m}_{1z} z_0 + \sqrt{\hat{\chi}_{1z}}\xi_{1z})z-\beta g(y,z)}dz. \label{phi_z}
\end{align}
\hrulefill
\end{figure*}
\end{corollary}
$\beta$ is a parameter that corresponds in the physics approach to an inverse temperature. In the $\beta \to \infty$ limit (the so-called zero temperature limit), the integrals defining $\phi_x$ and $\phi_z$ concentrate on their extremal value. Note that they are closely related to the Moreau envelopes $\mathcal{M}$ \cite{parikh2014proximal,bauschke2011convex} of $f$ and $g$, which represent a smoothed form of the objective function with the same minimizers:
\begin{align}
    &\phi_x(\hat{m}_{1x}, \hat{Q}_{1x}, \hat{\chi}_{1x}; x_0, \xi_{1x}) = \frac{\hat{Q}_{1x}}{2}H_{x}^{2}-\mathcal{M}_{\frac{f}{\hat{Q}_{1x}}}\left(H_{x}\right) \\
    &\mbox{where} \thickspace \forall \thickspace \gamma \geqslant 0, \thickspace \mathcal{M}_{\gamma f}(z) = \mbox{inf}_{x}\left\lbrace f(x)+\frac{1}{2\gamma}\norm{x-z}_{2}^{2}\right\rbrace, 
    \end{align}
We provide details on this correspondence in appendix \ref{app:rep_mor}.
In the zero-temperature limit we consider, it is possible to have more precise information on the geometry of 
the cost function defining the optimization problem in Corollary \ref{cor_free}. Indeed, it is composed of 
functions whose convexity or concavity are staightforward to establish : linear terms, inverses, logarithms, squares and expectation of Moreau envelopes. The 
convexity of the latter is well documented in \cite{thrampoulidis2018precise}. First, note that the parameters $\chi_{x},\chi_{z},\hat{\chi}_{1x},\hat{\chi}_{2x},\hat{\chi}_{1z},\hat{\chi}_{2z},q_{x},q_{z},\hat{Q}_{1x},\hat{Q}_{2x},\hat{Q}_{1z},\hat{Q}_{2z}$ are 
positive so we may restrict their feasibility set to $\mathbb{R}^{+}$, while $m_{x},m_{z},\hat{m}_{1x},\hat{m}_{1z},\hat{m}_{2x},\hat{m}_{2z}$ can take any value in $\mathbb{R}$. Then, $q^{*}_{x} = \frac{1}{N}\norm{\hat{\mathbf{x}}}^{2}$ and $m^{*}_{x} = \frac{1}{N}\mathbf{x}_{0}^{\top}\hat{\mathbf{x}}$.
The Cauchy-Schwarz inequality thus gives $q_{x}^{*} \geqslant \frac{(m^{*}_{x})^{2}}{\rho_{x}}$. Similarly with $\hat{\mathbf{z}}$, $q_{z}^{*} \geqslant \frac{(m^{*}_{z})^{2}}{\rho_{z}}$. We may thus restrict the feasibility sets of 
$q_{x},q_{z},m_{x},m_{z}$ such that they verify these inequalities. In these regions, the function $g_{s}$ is convex in $\chi_{x},\chi_{z}$, linear in $q_{x},q_{z}$ and concave in $m_{x},m_{z}$. The terms involving 
$q_{x},q_{z},m_{x},m_{z},\chi_{x},\chi_{z}$ in $g_{G}$ and $g_{F}$ are all linear. Moving to $g_{g}$, the cost function defining it is convex in $\hat{Q}_{2x},\hat{Q}_{2z}$ (negative logarithm and inverse function on $\mathbb{R}^{+}$),
linear in  $\hat{\chi}_{2x}, \hat{\chi}_{2z}$ and convex in $\hat{m}_{2x},\hat{m}_{2z}$. Regarding $g_{F}$, all terms are linear except for the replica potentials. Using Moreau's identity, we may write 
$\phi_x(\hat{m}_{1x}, \hat{Q}_{1x}, \hat{\chi}_{1x}; x_0, \xi_{1x}) = \mathcal{M}_{\hat{Q}_{1x}f^{*}}\left(\hat{m}_{1x}x_{0}+\sqrt{\hat{\chi}_{1x}}\xi_{1x}\right)$ where $f^{*}$ is the conjugate of $f$. Using 
the properties summarized in \cite{thrampoulidis2018precise}, the cost function defining $g_{F}$ is convex in $\hat{m}_{1x},\hat{m}_{1z},\hat{Q}_{1x},\hat{Q}_{1z}$. The convexity with respect to 
$\chi_{1x},\chi_{1z}$ is harder to characterize due to the composition of the Moreau envelope with the square root, and should be studied locally for more information. The extremization may then be rewritten as 
a maximization over the variables in which the cost function is concave and minimization over the variables in which the cost function is convex. Note that this does not give 
information on the uniqueness of the solution, which would require joint strict convexity and strict concavity. \\
\par
As immediate corollaries to Theorem \ref{main_th}, we can determine the asymptotic errors of the GLM and the optimal value of the loss function.
To characterize the asymptotic reconstruction errors and angles, we can define the norms of the estimators and their overlaps with the ground-truth vectors as the limits
\begin{align}
    m_x^* &\equiv \lim_{N \to \infty} \frac{\mathbf{\hat{x}}^{T}\mathbf{x}_{0}}{N}\hspace{1.3cm}
    m_z^* \equiv \lim_{M \to \infty} \frac{\mathbf{\hat{z}}^{T}\mathbf{z}_{0}}{M}  \\
    q_x^* &\equiv \lim_{N \to \infty} \frac{\norm{\mathbf{\hat{x}}}_{2}^{2}}{N} \hspace{1.5cm}
q_z^* \equiv\lim_{N \to \infty} \frac{\norm{\mathbf{\hat{z}}}_{2}^{2}}{M}.
\end{align}
We then have :
\begin{corollary} \quad \\
    Under the set of Assumptions \ref{main_assum}, the squared norms $m_x^*, m_z^*$ of estimator $\mathbf{\hat{x}}$ defined by~\eqref{student} and $\mathbf{\hat{z}}=\mathbf{F \hat{x}}$, and their overlaps $q_x^*, q_z^*$ with ground-truth vectors are almost surely given by:
\begin{align}
    m_x^* &= \mathbb{E}\left[x_{0}\eta_{\frac{f}{\hat{Q}_{1x}^*}}\left( H_{x} \right)\right], \thickspace
    q_x^*  = \mathbb{E}\left[\eta^{2}_{\frac{f}{\hat{Q}_{1x}^*}}\left(H_{x}\right)\right] \\
    m_z^*  &= \mathbb{E}\left[z_0  \eta_{\frac{g(.,y)}{\hat{Q}_{1z}^*}}\left(H_{z}\right)\right], \thickspace
q_z^* = \mathbb{E}\left[\eta^2_{\frac{g(.,y)}{\hat{Q}_{1z}^*}}\left(H_{z}\right)\right]
\end{align}
with $H_{x}$ and $H_{z}$ defined as in Theorem \ref{main_th}.
\end{corollary}
With the knowledge of the asymptotic overlap $m_{x}^*$, and squared norms $q_{x}^*$, $\rho_{x}$, most quantities of interest can be determined. For instance, the quadratic reconstruction error is obtained from its definition as $\mbox{E} = \rho_x + q_x^*  -2m_x^*$, while the angle between the ground-truth vector and the estimator is $\theta = \mbox{arccos}(m_{x}^*/(\sqrt{\rho_{x}q_{x}^*}))$. One can also evaluate the generalization error for new random Gaussian samples, as advocated in ~\cite{engel2001statistical}, or compute similar errors for the denoising of ${\bf z_0}$. 

\section{Numerical results}

Obtaining a stable implementation of the fixed point equations can be challenging. We provide simulation details in appendix \ref{num_append} along with a link to the script we used to produce the figures. Theoretical predictions (full lines) are compared with numerical experiments (points) conducted using standard convex optimization solvers from \cite{pedregosa2011scikit}. The comparison with finite size ($N \equiv$ a few hundreds) numerical experiments shows that, despite being asymptotic in nature, the predictions are accurate even at moderate system sizes. 
All experimental points were done with $N=200$ and averaged one hundred times.
\subsection{Validity of the replica prediction}
We start with a simple verification of the replica prediction in Figure\ref{fig_data}, on a classification problem where data is generated as $\mathbf{y} = \sign(\mathbf{F x_0})$. We consider two types of singular value distributions for $\mathbf{F}$ and three types of losses: a square loss, a linear support vector classification (SVC) loss and a logistic loss. Technical details and expressions are given in appendix \ref{num_append}. We use ridge regularization with penalty $f=\frac{\lambda_2}{2} \norm{\cdot}_2^2$. We plot the reconstruction angle $\theta$ as a function of the aspect ratio of the problem $\alpha$ in Figure \ref{fig_data}. A first plot is done with a Marchenko-Pastur eigenvalue distribution for $\mathbf{F}^T \mathbf{F}$ corresponding to $\mathbf{F}$ being i.i.d Gaussian. We then move out of the Gaussian setting and change the eigenvalue distribution for~\eqref{indicator}, which has a qualitatively similar behaviour: it has bounded support, and includes vanishing singular values at a given value $\alpha=1$ of the aspect ratio. We recover a result close to the i.i.d. Gaussian one, including the error peak for the square loss when $\alpha=1$. In both cases, the SVC and the logistic regression perform similarly. Note that error peaks can also be obtained for the max-margin solution as shown in \cite{gerace2020generalisation}, using a more elaborate teacher.
\begin{figure}[!t]
\centering
    \includegraphics[scale=0.5]{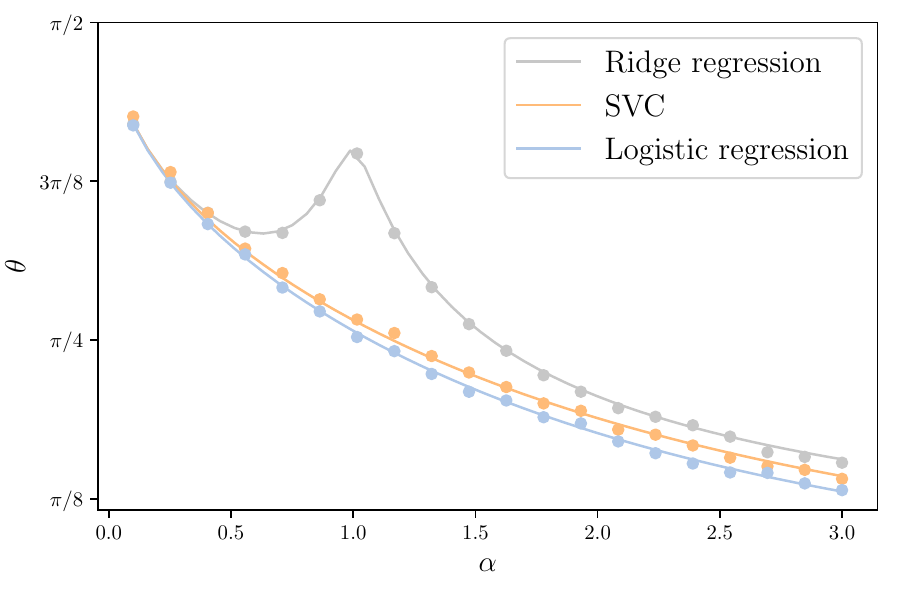}
    \includegraphics[scale=0.5]{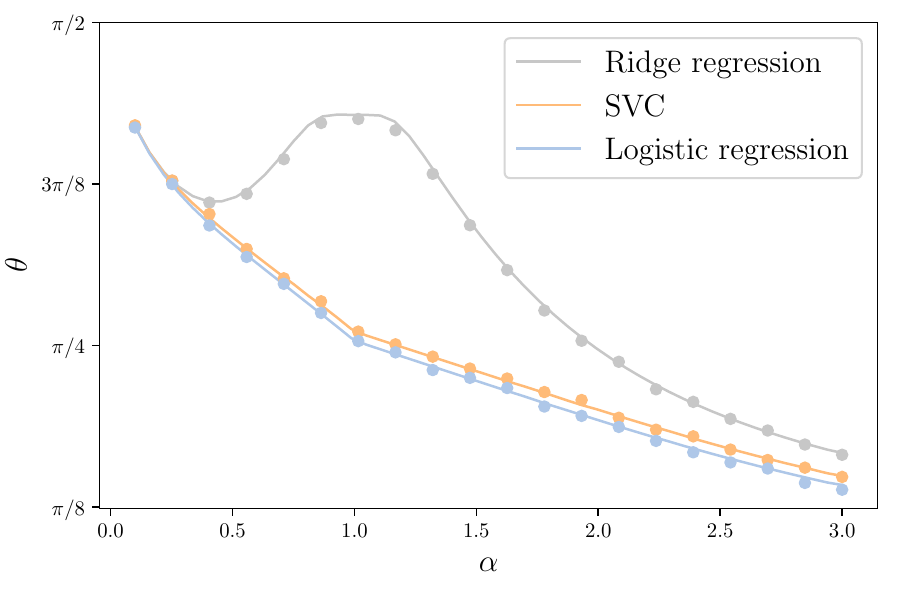}
    \caption{Illustration of Theorem \ref{main_th} in a binary classification problem with data generated as ${\mathbf{y} = \phi(\mathbf{F}\mathbf{x}_{0})}$ with the data matrix ${\bf F}$ being {\bf Left :} a Gaussian i.i.d. matrix and {\bf Right :} a random orthogonal invariant matrix with a squared uniform density of  singular values. We plot the angle between the estimator and the ground-truth vector $\theta = \mbox{arccos}(m_{x}^*/(\sqrt{\rho_{x}q_{x}^*}))$ as a function of the aspect ratio $\alpha = M/N$ with three different losses: ridge regression, a Support Vector Machine with linear kernel and a logistic regression. $f$ is a $\ell_2$ penalty with parameter $\lambda_2=10^{-3}$. The theoretical prediction (full line) is compared with numerical experiments (points) conducted using standard convex optimization solvers from \cite{pedregosa2011scikit}.}
    \label{fig_data}
\end{figure}
\subsection{Sparse logistic regression}
We now use the replica prediction to study sparse logistic regression with i.i.d Gaussian and row-orthogonal data, the latter being ubiquitous in signal processing. Row-orthogonal data gives rise to a discrete eigenvalue distribution for $\mathbf{F}^T \mathbf{F}$ of zeroes and ones:
\begin{equation}
    \lambda_{\mathbf{F}^{T}\mathbf{F}} \sim \max(0,1-\alpha)\delta(0)+\min(1,\alpha)\delta(1)
\end{equation}
and is often found to outperform Gaussian sensing matrices for recovery tasks, see e.g. \cite{kabashima2009typical} or \cite{gerbelot2020asymptotic}. In what follows, we define the sparsity $\rho$ of the ground truth vector as the fraction of non-zero components which are sampled from a standard normal distribution. Labels are generated with $\mathbf{y} = \sign(\mathbf{F x_0})$ as for Figure \ref{fig_data}. \\
\subsubsection{Effect of sparsity}
In Figure \ref{figure2}, we start by plotting the reconstruction angle against the aspect ratio of the measurement matrix for different values of the sparsity of the teacher vector, for $\ell_{2}$ regularization $f = \frac{\lambda_2}{2} \norm{\cdot}_2^2$ and $\ell_{1}$ regularization $f = \lambda_1 \norm{\cdot}_1$, and a fixed value of regularization parameters $\lambda_1, \lambda_2$.
\begin{figure*}[!t]
\centering
    \includegraphics[scale=0.45]{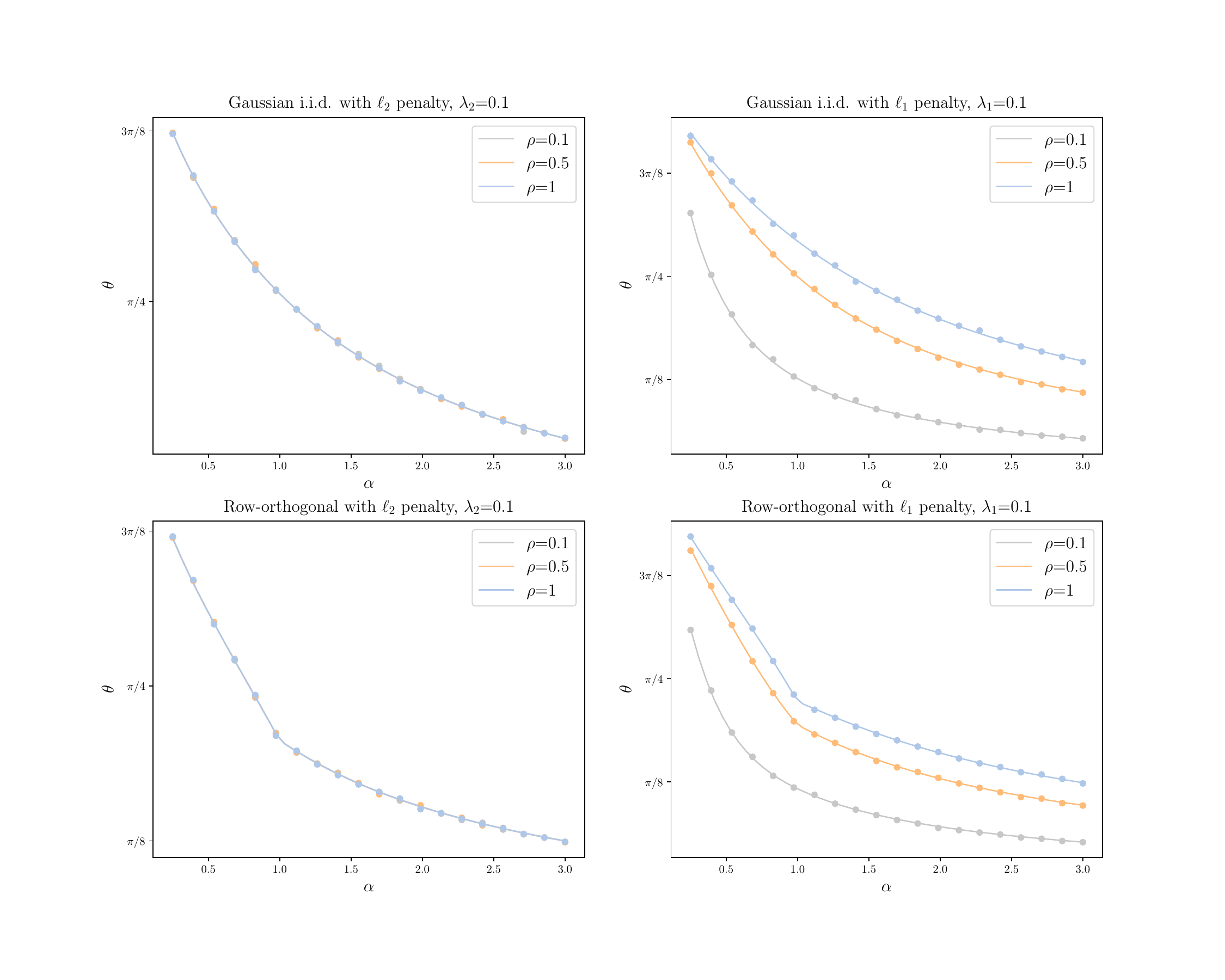}
    \caption{Effect of the sparsity of the planted vector. We plot the angle between the estimator and the ground truth in a binary classification problem with ${\mathbf{y} = \mbox{sign}(\mathbf{F}\mathbf{x}_{0})}$ as a function of $\alpha=M/N$, for different values of sparsity $\rho$. We use logistic regression. Figures in the top are for $\mathbf{F}$ Gaussian i.i.d., while figures in the bottom are for $\mathbf{F}$ row-orthogonal. {\bf Left :} we use a $\ell_2$ penalty with parameter $\lambda_2=0.1$, and notice that the angle is the same for any sparsity. {\bf Right :} we use a  $\ell_1$ penalty with parameter $\lambda_1=0.1$. The theoretical prediction (full line) is compared with numerical experiments (points) conducted using standard convex optimization solvers from \cite{pedregosa2011scikit}.}
    \label{figure2}
\end{figure*}
In the case of $\ell_{2}$-regularization, we observe that the reconstruction performance remains the same whatever the sparsity of the original teacher vector as all curves collapse together (top and bottom left). The ridge regularization is thus unable to differentiate sparse and non-sparse problems. For $\ell_{1}$, better performance is observed when the sparsity increases. Comparing the values for $\ell_{2}$ and $\ell_{1}$ also shows that, for a non-sparse signal, $\ell_{2}$ and $\ell_{1}$ reconstruction perform similarly. The largest difference is observed at $\rho = 0.1$, where the $\ell_{1}$ penalized logistic regression significantly outperforms the ridge one. We thus keep this value of the sparsity parameter for the next figures. \\
\quad \\
\subsubsection{Varying the regularization parameter at constant sparsity}
In  Figure \ref{figure3}, keeping the sparsity of the teacher constant at $\rho=0.1$, we look to tune the regularization strength. An interesting effect appears in the ridge-regularized case with row-orthogonal measurements : the curves collapse to a single one when the aspect ratio goes below $\alpha=1$. We find that the optimal regularization strength for the $\ell_{2}$ penalty lies around $\lambda_{2} = 0.01$, and for the $\ell_{1}$-penalty around $\lambda_{1} = 0.1$, for both types of matrices.
\begin{figure*}[!t]
    \centering
    \includegraphics[scale=0.45]{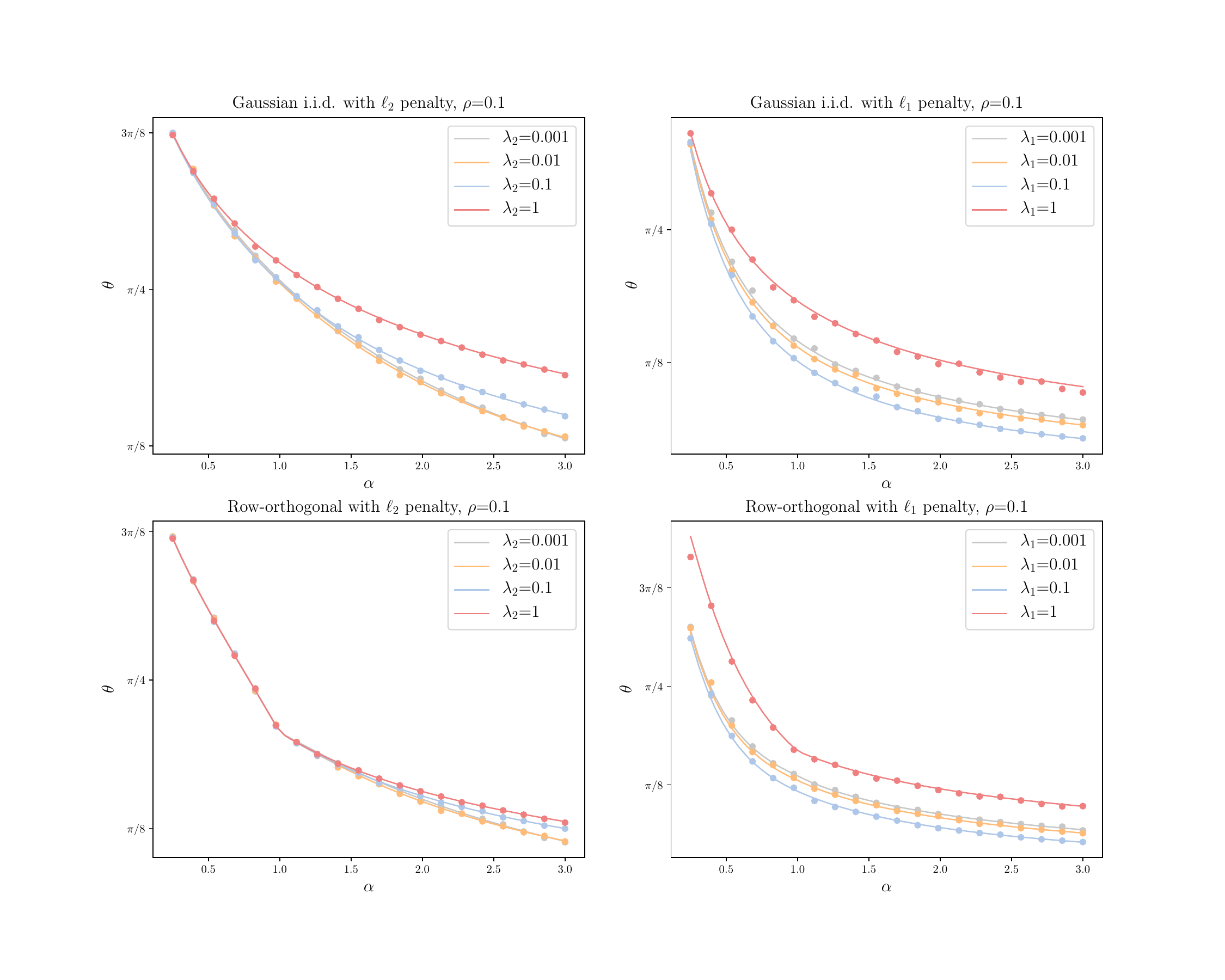}
    \caption{Tuning the regularization parameter. We still plot the angle between the estimator and the ground truth in a binary classification problem with ${\mathbf{y} = \sign(\mathbf{F}\mathbf{x}_{0})}$ as a function of $\alpha=M/N$, for a fixed sparsity of planted vector $\rho = 0.1$, for different values of regularization parameters. Figures in the top are for $\mathbf{F}$ Gaussian i.i.d., while figures in the bottom are for $\mathbf{F}$ row-orthogonal. {\bf Left :}  $\ell_2$ penalty with different values of regularization parameter $\lambda_2$. {\bf Right :}  $\ell_1$ penalty with different values of regularization parameter $\lambda_1$.}
    \label{figure3}
\end{figure*}\\

\subsubsection{Comparing case}
In Figure \ref{figure4}, we directly compare the reconstruction performance of logistic regression on a sparse problem with previously tuned regularization parameter of $\ell_{2}$ and $\ell_{1}$ penalties, with the two types of measurement matrices.
\begin{figure}[!t]
    \centering
    \includegraphics[scale=0.5]{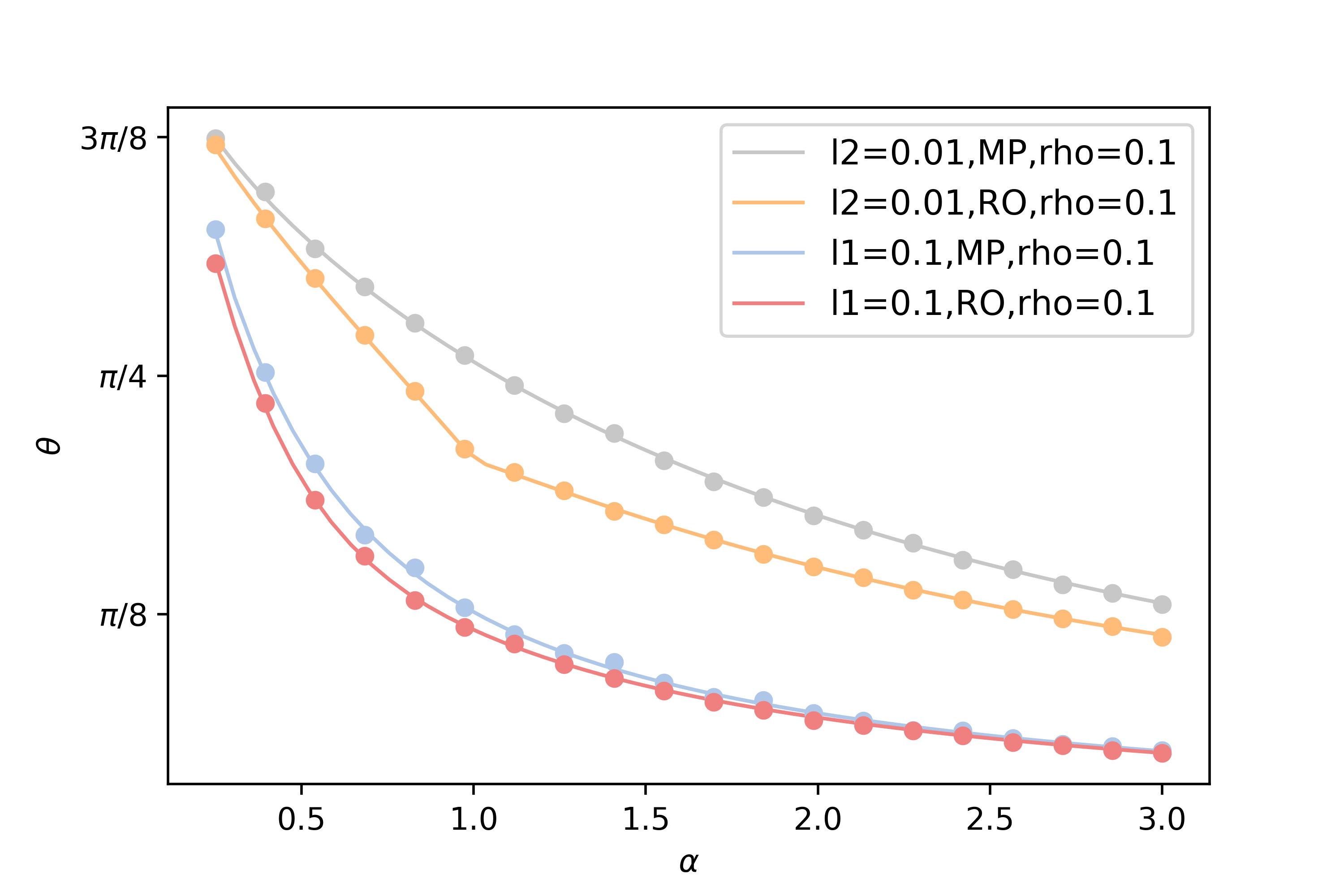}
    \caption{Comparing reconstruction performance for Gaussian i.i.d. and row-orthogonal matrices. In this figure, we compare the reconstruction angles between the estimator and the ground-truth for binary classification obtained with $\ell_1$ and $\ell_2$ penalties. We use logistic regression. The sparsity of the sparse vector is fixed to $\rho=0.1$. For both Gaussian i.i.d. and row-orthogonal data matrices, we see that $\ell_1$ penalty with $\lambda_1=0.1$ performs better than the $\ell_2$ penalty with $\lambda_2=0.01$. For those two penalties, row-orthogonal matrices allow to obtain smaller reconstruction angles than Gaussian i.i.d. matrices.}
    \label{figure4}
\end{figure}
We naturally observe that the $\ell_{1}$ penalty leads to better reconstruction of the sparse vector. Row-orthogonal matrices outperform the i.i.d. Gaussian ones with both regularization, although the gap is less significant with the $\ell_{1}$ penalty. \\
\quad \\
\subsubsection{Discussion}
Several non-trivial effects are observed when studying the interplay between eigenvalue distribution of the design matrix, loss function, regularization and structure of the underlying teacher vector. Looking for analytical simplifications of the fixed point equations from Theorem \ref{main_th} in specific cases would be interesting to understand how the key quantities interact and lead, for example, to the collapsing observed in $\ell_{2}$-penalized problems. This further motivates the use of these equations to determine reconstruction limits of generalized-linear modeling. Some examples include limits of sparse recovery for different types of measurement matrices, or finding if optimal losses can be designed to achieve performances close to Bayes optimal errors.
\section{Sketch of proof of Theorem \ref{main_th}}
Our proof follows an approach pioneered in \cite{bayati2011lasso} where the LASSO risk for i.i.d. Gaussian matrices is determined. The  idea is to build a sequence of iterates that provably converges towards the estimator $\mathbf{\hat{x}}$, while also knowing the statistical properties of those iterates through a set of equations. We must therefore concern ourselves with three fundamental aspects:

\begin{enumerate}[label=(\roman*),itemsep=0mm,nosep,leftmargin=0.6cm]
\item construct a sequence of iterates with a rigorous statistical characterization that matches their equations of Theorem \ref{main_th} at the fixed point,
\item verify that the sequence's fixed point corresponds to the estimator $\hat{\mathbf{x}}$,
\item check that this sequence is provably convergent, otherwise the iterates might drift off on a diverging trajectory, and the fixed point would never be reached. We thus make sure the statistical characterization indeed applies to the point of interest $\hat{\mathbf{x}}$.
\end{enumerate}
In short, we have a sequence of estimates $(\mathbf{x}_{k})_{k \in \mathbb{N}}$ taking values in $\mathbb{R}^{N}$, and their exact asymptotic (in N) distribution for any $k > 0$. To show that these statistics extend to $\hat{\mathbf{x}}$, we need to show that $\lim_{k \to \infty} \mathbf{x}_{k} = \hat{\mathbf{x}}$. To do so, we need the sequence to converge (i.e. point iii), and its fixed point to be $\mathbf{\hat{x}}$ (point ii). As indicated in the introduction, we will use an instance of the 2-layer MLVAMP algorithm to construct this sequence. Note that, for the sake of brevity, we do not verify that limiting points of 2-layer MLVAMP trajectories $\lim_{k \to \infty}\mathbf{x}_{k}$ converge empirically to the Gaussian distribution prescribed by the state evolution equations. This point is treated explicitly in \cite{emami2020generalization}. \\

The following lemma establishes the link between the state evolution equations and our main theorem.
\vspace{0.2cm}
\begin{lemma}(Fixed point of 2-layer MLVAMP state evolution equations)
\label{fixed_point_se_match}
  The state evolution equations of 2-layer MLVAMP from \cite{fletcher2018inference}, reminded in appendix \ref{SE_append}, match the equations of Theorem \ref{main_th} at their fixed point.
\end{lemma}
\begin{proof}See appendix \ref{SE_append}.\end{proof}
\quad \\
This confirms that 2-layer MLVAMP is a good choice to design the sequences that we seek. 
We know that the iterates of 2-layer MLVAMP can be characterized by state evolution equations which correspond, at their fixed point, to the equations of Theorem \ref{main_th} by virtue of Lemma \ref{fixed_point_se_match}. The necessary assumptions for the state evolution equations to hold are verified in appendix \ref{SE_assumptions}. We must now show that the estimator of interest defined by (\ref{teacher}) and (\ref{student}) can be reached using 2-layer MLVAMP. We thus continue with point (ii).
\vspace{0.2cm}
\begin{lemma}(Fixed point of 2-layer MLVAMP)
  \label{VAMP_fixed}
  The fixed point of algorithm \eqref{GVAMP} matches the optimality condition of the unconstrained convex problem Eq.\eqref{student} 
\end{lemma}
\begin{proof} See appendix \ref{fixed_point_proof}. \end{proof}
\quad \\
This part is a consequence of the structure of the algorithm and properties of proximal operators. We now move to point (iii) and seek to characterize the convergence properties of 2-layer MLVAMP. Instead of directly tackling the convergence of 2-layer MLVAMP on any convex GLM, we take a detour and focus on a constrained problem, where functions $f$ and $g$ are augmented by a $\ell_2$ norm with ridge parameters $\lambda_2$, $\tilde{\lambda}_2$. The called on intuition is that the algorithm will be more likely to converge in a strongly convex problem. We start by showing the convergence of MLVAMP in the constrained strongly convex setting, for values of $\lambda_{2}$ larger than a certain threshold, and any strictly positive $\tilde{\lambda}_{2}$.
\begin{lemma}(Linear convergence of 2-layer MLVAMP for strongly convex problems)
\label{conv_lemma}
Assume $f$ and $g$ are twice differentiable. Define the constrained problem
\begin{equation}
\label{smooth_student}
\mathbf{\hat{x}}(\lambda_{2},\tilde{\lambda}_{2}) = \argmin_{\mathbf{x} \in \mathbb{R}^{N}} \left\lbrace \tilde{g}(\mathbf{F}\mathbf{x},\mathbf{y})+\tilde{f}(\mathbf{x})\right\rbrace
\end{equation}
where $\tilde{f}(\mathbf{x}) = f(\mathbf{x})+\frac{\lambda_{2}}{2}\norm{\mathbf{x}}_{2}^{2}$ and $\tilde{g}(\mathbf{x}, \mathbf{y}) = g(\mathbf{x}, \mathbf{y})+\frac{\tilde{\lambda}_{2}}{2}\norm{\mathbf{x}}_{2}^{2}$.
Consider 2-layer MLVAMP applied to find (\ref{smooth_student}), from which we extract at each iteration the vector ${\mathbf{h}^{(t)} = \left[\mathbf{h}_{2z}^{(t)}, \mathbf{h}_{1x}^{(t)}\right]^{T}}$. Let $\mathbf{h^*}$ be its value at the fixed point of algorithm \eqref{GVAMP}. We then have that, for any $\tilde{\lambda_{2}}>0$, there exists a value $\lambda_{2}^{*}$ such that, for any $\lambda_{2}>\lambda_{2}^{*}$, there exists a strictly positive constant $c$ verifying $0<c<\lambda_{2}$, such that for any $t\in \mathbb{N}$:
  \begin{align}
  \norm{\mathbf{h}^{(t)}-\mathbf{h^{*}}}_{2}^{2} \leqslant \left(\frac{c}{\lambda_{2}}\right)^{t}\norm{\mathbf{h}^{(0)}-\mathbf{h^{*}}}_{2}^{2},\end{align}
The convergence of $\mathbf{h}^{(t)}$ implies that estimators $\mathbf{\hat{x}}_1^{(t)}$ and $\mathbf{\hat{x}}_2^{(t)}$ returned by 2-layer MLVAMP also converge to the desired $\mathbf{\hat{x}}(\lambda_2, \tilde{\lambda}_2)$, i.e., under the conditions listed above
\begin{align}
    &\lim_{t \to \infty} \norm{\hat{\mathbf{x}}^{(t)}-\mathbf{\hat{x}}(\lambda_2, \tilde{\lambda}_2)}_{2}^{2} = 0.
\end{align}
\end{lemma}
\begin{proof}
See appendix \ref{conv_proof}.
\end{proof} 
\quad \\
For a loss function $\tilde{g}$ with any non-zero strong convexity constant, and a regularization $\tilde{f}$ with a sufficiently strong convexity, 2-layer MLVAMP converges linearly towards its unique fixed point. Note that this convergence result is independent from the dimension. We elaborate on this lemma in the next section. An immediate consequence is the following lemma, which claims that Theorem \ref{main_th} holds when 2-layer MLVAMP converges. Since this result does not rely on an analytic continuation, the assumptions on the concentration of PL2 observables of $\hat{\mathbf{x}}$, given by the state evolution property, and approximation of the cost function by analytic functions with fast decaying higher order derivatives are not required. The result can also 
be stated for any PL2 observable, with no restriction on its derivability and decay of higher order derivatives. We summarize the necessary assumptions in the following list:
\begin{assumption} 
\label{str_conv_assum}
\begin{enumerate}[label=(\alph*)]
    \item [~]
    \item the functions $f$ and $g$ are proper, closed, convex and separable functions.
    \item the cost function $g(\mathbf{F}.,\mathbf{y})+f(.)$ is coercive, i.e. $\lim_{\norm{\mathbf{x}}\to \infty}g(\mathbf{F}\mathbf{x},\mathbf{y})+f(\mathbf{x}) = +\infty$.
    \item there exists a constant $B_{1}$ such that $\frac{1}{N}\norm{\hat{\mathbf{x}}}^{2}_{2}\leqslant B_{1}$ almost surely as $N \to \infty$.
    \item for any $\mathbf{x} \in \mbox{dom}(f)$ and any $\mathbf{x}' \in \partial f(\mathbf{x})$, there exists a constant $C$ such that $\norm{\mathbf{x}'}_{2}\leqslant C(1+\norm{\mathbf{x}}_{2})$. The same holds for $g$ on its domain. 
    \item the empirical distributions of the underlying truth $\mathbf{x}_{0}$, eigenvalues of $\mathbf{F}^T \mathbf{F}$, and noise vector $w_{0}$, respectively converge empirically with second order moments, as defined in appendix \ref{appendix:analysis_framework}, to independent scalar random variables $x_{0},w_{0},\lambda$ with distributions $p_{x_{0}}$, $p_{\lambda}$, $p_{w_{0}}$. We assume that the distribution $p_{\lambda}$ is not all-zero and has compact support. 
    \item the design matrix $\mathbf{F} = \mathbf{U}\mathbf{D}\mathbf{V}^{\top} \in \mathbb{R}^{M \times N}$ is rotationally invariant, as defined in the introduction, where the elements of the Haar distributed matrices $\mathbf{U},\mathbf{V}$ are independent of the elements of the ground truth vector $\mathbf{x}_{0}$, noise $\boldsymbol{\omega}_{0}$ and elements of $\mathbf{D}$.
    \item the solution to the set of fixed point equations \eqref{thm1-equations} exists and is unique for any convex functions $f,g$ verifying the 
    \item finally assume that $M,N \to \infty$ with fixed ratio $\alpha = M/N$.
\end{enumerate}
\end{assumption}
\begin{lemma}(Asymptotic error for the twice differentiable, sufficiently strongly convex problem) \\
\label{main_lemma}
Consider the strongly convex minimization problem with twice differentiable $f$ and $g$ \eqref{smooth_student}.
Under the set of assumptions~\ref{str_conv_assum}, for any $\tilde{\lambda}_{2}>0$, there exists a $\lambda_{2}^{*}$ such that, for any $\lambda_{2}>\lambda_{2}^{*}$, Then, for any pseudo-Lipschitz function of order 2 $\phi$, the following holds
:
\begin{align}
    &\lim_{N \to \infty} \frac{1}{N}\sum_{i=1}^{N}\phi(x_{0,i},\hat{x}_{i}) \stackrel{a.s.} = \mathbb{E}[\phi(x_{0},\mbox{Prox}_{f/\hat{Q}_{1x}^{(t)}}(H_{x}))]\\ 
    &\lim_{M \to \infty}\frac{1}{M}\sum_{i=1}^{M}\phi(z_{0,i},\hat{z}_{i}) \stackrel{a.s.} = \mathbb{E}[\phi(z_{0},\mbox{Prox}_{f/\hat{Q}_{1z}^{(t)}}(H_{z}))]
\end{align}
where the scalars $\hat{Q}_{1x},\hat{Q}_{1z}$ and the random variables $H_{x},H_{z}$ are defined as in Theorem \ref{main_th}.
\end{lemma}
\begin{proof}
Using the result from Lemma \ref{conv_lemma}, we have $\lim_{t\to \infty}\lim_{N \to \infty}\frac{1}{N}\norm{\mathbf{x}^{(t)}-\mathbf{\hat{x}}(\lambda_2, \tilde{\lambda}_2)}_{2}^{2} = 0$. As proven in \cite{emami2020generalization}, the state evolution parameters will converge to those of the fixed point of the state evolution equations along a converging trajectory of 2-layer MLVAMP. Using the assumption on the bounded averaged norm of $\hat{\mathbf{x}}$, the state evolution equations to show that the averaged norm of the iterates are bounded along a converging trajectory, and the state evolution equations to obtain the exact asymptotics of each iterate along the converging trajectory, an identical argument to that of the proof of Theorem 1.5 from \cite{bayati2011dynamics} gives Lemma \ref{main_lemma}.
\end{proof}
\quad \\
We are now left to prove Theorem \ref{main_th}, for any range of parameters $(\lambda_2, \tilde{\lambda}_2)$. $\tilde{\lambda}_{2}$ can already be chosen arbitrarily small. This means we need to relax the threshold value on $\lambda_{2}$ for the validity of the scalar quantities involved in Theorem 1. To do so, we start by introducing another modification of the original problem, where the objective functions are assumed to be real analytic. Lemma \ref{main_lemma} naturally holds for real analytic convex functions. Proving Theorem \ref{main_th} on the real analytic problem then boils down to performing an analytic continuation on the $\lambda_{2}$ parameter, and is detailed in Appendix \ref{analytic_continuation}. We thus have the following intermediate result :

\begin{lemma}(Asymptotics of the real analytic problem)
\label{analytic_error_lemma}
Consider assumption \ref{main_assum} is verified. Suppose additionally that $f$ and $g$ are real analytic. Then Theorem 1 holds for any $\tilde{\lambda}_{2} > 0$ and any $\lambda_{2}>0$.
\end{lemma}
Theorem \ref{main_th} can then be proven from Lemma \ref{analytic_error_lemma} by showing that the solutions of the original problem and of its real analytic approximation are arbitrarily close, and by carefully studying the limits $\tilde{\lambda}_{2}\to 0$ and $\lambda_{2} \to 0$. This is deferred to Appendix \ref{analytic_continuation}. Note that the proof of the analytic continuation presented here makes the one from \cite{gerbelot2020asymptotic}, which was incomplete, rigorous.  \\
The remaining technical part is the proof of the convergence Lemma \ref{conv_lemma}. For this purpose, we use a dynamical system reformulation of 2-layer MLVAMP and
a result from control theory, adapted to machine learning in \cite{lessard2016analysis} and more specifically to ADMM in \cite{nishihara2015general}.
\section{Convergence analysis of 2-layer MLVAMP}
\label{section:conv_sec}
The key idea of the approach pioneered in \cite{lessard2016analysis} is to recast any non-linear dynamical system as a linear one, where convergence will be naturally characterized by a matrix norm. For a given non-linearity $\tilde{\mathcal{O}}$ and iterate $\mathbf{v}$, we define the variable $\mathbf{u}= \tilde{\mathcal{O}}(\mathbf{v})$ and rewrite the initial algorithm in terms of this trivial transform. Any property of $\tilde{\mathcal{O}}$ is then summarized in a constraint matrix linking $\mathbf{v}$ and $\mathbf{u}$. For example, if $\tilde{\mathcal{O}}$ has Lipschitz constant $\omega$, then for all $t$:
\begin{equation}
    \norm{\mathbf{u}^{(t+1)}-\mathbf{u}^{(t)}}_{2}^{2} \leqslant \omega^{2} \norm{\mathbf{v}^{(t+1)}-\mathbf{v}^{(t)}}_{2}^{2},
\end{equation}
which can be rewritten in matrix form:
\begin{align}
    \mathbf{U}^{T}&\begin{bmatrix}\omega^{2}\mathbf{I}_{d_v} & 0 \\0 & -\mathbf{I}_{d_u}\end{bmatrix}\mathbf{U} \geqslant 0 \\
     \mbox{where} \thickspace \mathbf{U} &= \begin{bmatrix}\mathbf{v}^{(t+1)}-\mathbf{v}^{(t)} \\\mathbf{u}^{(t+1)}-\mathbf{u}^{(t)}\end{bmatrix}
\end{align}
where $\mathbf{I}_{d_v}, \mathbf{I}_{d_u}$ are the identity matrices with dimensions of $\mathbf{v},  \mathbf{u}$, i.e. $M$ or $N$ in our case. Any co-coercivity property (verified by proximal operators) can be rewritten in matrix form but yields non block diagonal constraint matrices. We will thus directly use the Lipschitz constants for our proof, as they lead to simpler derivations and suffice to prove the required result. The main theorem from \cite{lessard2016analysis}, adapted to ADMM in \cite{nishihara2015general}, then establishes a sufficient condition for convergence with a linear matrix inequality, involving the matrices defining the linear recast of the algorithm and the constraints. Let us now detail how this approach can be used on 2-layer MLVAMP.
\subsection{2-layer MLVAMP as a dynamical system : sketch of proof of Lemma 3}
We start by rewriting 2-layer MLVAMP in a more compact form: 
\begin{align}
\label{comp1}
& \text{Initialize }\mathbf{h}_{1x}^{(0)}, \mathbf{h}_{2z}^{(0)} \nonumber \\
    &\mathbf{h}_{1x}^{(t+1)} = \mathbf{W}^{(t)}_1 \tilde{\mathcal{O}}^{(t)}_1 \mathbf{h}_{1x}^{(t)} + \mathbf{W}^{(t)}_2 \tilde{\mathcal{O}}^{(t)}_2 (\mathbf{W}^{(t)}_3 \mathbf{h}_{2z}^{(t)} + \mathbf{W}^{(t)}_4 \tilde{\mathcal{O}}^{(t)}_{1}( \mathbf{h}_{1x}^{(t)})) \\
    &\mathbf{h}_{2z}^{(t+1)} = \mathbf{\tilde{\mathcal{O}}}^{(t)}_2(\mathbf{W}^{(t)}_3 \mathbf{h}_{2z}^{(t)} + \mathbf{W}^{(t)}_4 \tilde{\mathcal{O}}^{(t)}_1 (\mathbf{h}_{1x}^{(t)}))
    \label{comp2}
\end{align}
where 
\begin{align}
\label{rec_op}
\mathbf{W_1}^{(t)} &= \dfrac{\hat{Q}^{(t)}_{2x}}{\hat{Q}^{(t+1)}_{1x}} \left( \dfrac{1}{\chi^{(t+1)}_{2x}}(\hat{Q}^{(t+1)}_{2z} \mathbf{F}^T \mathbf{F}+\hat{Q}^{(t)}_{2x} \mbox{Id})^{-1} - \mbox{Id} \right)  \\
\mathbf{W_2}^{(t)} &= \dfrac{\hat{Q}^{(t+1)}_{2z}}{\chi^{(t+1)}_{2x} \hat{Q}^{(t+1)}_{1x}} (\hat{Q}^{(t+1)}_{2z} \mathbf{F}^T \mathbf{F}+\hat{Q}^{(t)}_{2x} \mbox{Id})^{-1} \mathbf{F}^T \\
\mathbf{W_3}^{(t)} &= \dfrac{\hat{Q}^{(t)}_{2z}}{\hat{Q}^{(t)}_{1z}} \left( \dfrac{1}{\chi^{(t)}_{2z}} \mathbf{F} (\hat{Q}^{(t)}_{2z} \mathbf{F}^T \mathbf{F}+\hat{Q}^{(t)}_{2x} \mbox{Id})^{-1} \mathbf{F}^T - \mbox{Id} \right) \\
\mathbf{W_4}^{(t)} &= \dfrac{\hat{Q}_{2x}^{(t)}}{\hat{Q}_{1z}^{(t)}\chi^{(t)}_{2z}} \mathbf{F} (\hat{Q}^{(t)}_{2z} \mathbf{F}^T \mathbf{F}+\hat{Q}^{(t)}_{2x} \mbox{Id})^{-1} \\
\mathbf{\tilde{O}_1}^{(t)}&= \dfrac{\hat{Q}^{(t)}_{1x}}{\hat{Q}^{(t)}_{2x}}\left( \dfrac{1}{\chi^{(t)}_{1x} \hat{Q}^{(t)}_{1x}} \mbox{Prox}_{\mathbf{f}/\hat{Q}^{(t)}_{1x}}(\cdot) - \mbox{Id}\right) \\
\mathbf{\tilde{O}_2}^{(t)}&= \dfrac{\hat{Q}^{(t)}_{1z}}{\hat{Q}^{(t+1)}_{2z}}\left( \dfrac{1}{\chi^{(t)}_{1z} \hat{Q}^{(t)}_{1z}} \mbox{Prox}_{\mathbf{g}(.,y)/\hat{Q}^{(t)}_{1z}}(\cdot) - \mbox{Id}\right).
\end{align}
For the linear recast, we then define the variables: 
\begin{align}
    &\mathbf{u}_{1}^{(t)} = \tilde{\mathcal{O}}^{(t)}_{1}(\mathbf{h}_{1x}^{(t)}), \thickspace
    \mathbf{v}^{(t)} =\mathbf{W}^{(t)}_{3}\mathbf{h}_{2z}^{(t)}+\mathbf{W}^{(t)}_{4}\mathbf{u}_1^{(t)},\\ 
   &\mathbf{u}_{2}^{(t)} = \tilde{\mathcal{O}}^{(t)}_{2}(\mathbf{v}^{(t)}), \\
    \text{s.t.} \thickspace
    &\mathbf{h}_{2z}^{(t+1)} = \mathbf{u}_{2}^{(t)},
    \mathbf{h}_{1x}^{(t+1)} = \mathbf{W}^{(t)}_{1}\mathbf{u}_{1}^{(t)}+\mathbf{W}^{(t)}_{2}\mathbf{u}_{2}^{(t)}.
\end{align}
where $\mathbf{u}_{1}, \mathbf{h}_{1x} \in \mathbb{R}^{N}$; and $\mathbf{v},\mathbf{u}_{2},\mathbf{h}_{2z} \in \mathbb{R}^{M}$. We then define as new variables the vectors
\begin{align}
\label{def_hu}
    \mathbf{h}^{(t)} = \begin{bmatrix}\mathbf{h}_{2z}^{(t)} \\ \mathbf{h}_{1x}^{(t)}\end{bmatrix}, \quad
    \mathbf{u}^{(t)} = \begin{bmatrix}\mathbf{u}_{2}^{(t)} \\ \mathbf{u}_{1}^{(t)}\end{bmatrix}, \\
    \label{def_w12}
    \mathbf{w}_{1}^{(t)} = \begin{bmatrix}\mathbf{h}_{1x}
^{(t)}\\ \mathbf{u}_{1}^{(t)}\end{bmatrix}, \quad
\mathbf{w}_{2}^{(t)} = \begin{bmatrix}\mathbf{v}^{(t)} \\ \mathbf{u}_{2}^{(t)}\end{bmatrix}.
\end{align}
This leads to the following linear dynamical system recast of \eqref{comp1}-\eqref{comp2}:
\begin{align}
\label{dyn_syst1}
  \mathbf{h}^{(t+1)} &= \mathbf{A}^{(t)}\mathbf{h}^{(t)}+\mathbf{B}^{(t)}\mathbf{u}^{(t)} \\
  \mathbf{w}_{1}^{(t)} &= \mathbf{C}^{(t)}_{1}\mathbf{h}^{(t)}+\mathbf{D}^{(t)}_{1}\mathbf{u}^{(t)} \\
  \mathbf{w}_{2}^{(t)} &= \mathbf{C}^{(t)}_{2}\mathbf{h}^{(t)}+\mathbf{D}^{(t)}_{2}\mathbf{u}^{(t)}
  \label{dyn_syst2}
\end{align}
where
\begin{align}
\mathbf{A}^{(t)} &= \mathbf{0}_{(M+N)\times (M+N)} \thickspace \mathbf{B}^{(t)} = \begin{bmatrix}\mathbf{I}_{M} & \mathbf{0}_{M \times N} \\\mathbf{W}^{(t)}_{2} & \mathbf{W}^{(t)}_{1}\end{bmatrix} \\
\mathbf{C}^{(t)}_{1} &= \begin{bmatrix}\mathbf{0}_{N \times M} & \mathbf{I}_{N} \\ \mathbf{0}_{N \times M} & \mathbf{0}_{N \times N}\end{bmatrix} \mathbf{D}^{(t)}_{1} = \begin{bmatrix} \mathbf{0}_{N \times M} & \mathbf{0}_{N \times N} \\ \mathbf{0}_{N \times M} & \mathbf{I}_{N}\end{bmatrix}\\
\mathbf{C}^{(t)}_{2} &= \begin{bmatrix}\mathbf{W}^{(t)}_{3} & \mathbf{0}_{M \times N} \\ \mathbf{0}_{M \times M} & \mathbf{0}_{M \times N}\end{bmatrix}  \mathbf{D}^{(t)}_{2} = \begin{bmatrix} \mathbf{0}_{M\times M} & \mathbf{W}^{(t)}_{4} \\ \mathbf{I}_{M} & \mathbf{0}_{M \times N}\end{bmatrix}.
\end{align}
$\mathbf{O}$ denotes a matrix with only zeros. The next step is to impose the properties of the non-linearities $\tilde{\mathcal{O}}^{(t)}_{1},\tilde{\mathcal{O}}^{(t)}_{2}$ through constraint matrices. The Lipschitz constants $\omega_{1}^{(t)},\omega_{2}^{(t)}$ of $\tilde{\mathcal{O}}^{(t)}_{1},\tilde{\mathcal{O}}^{(t)}_{2}$ can be determined using properties of proximal operators \cite{giselsson2016linear} and are directly linked to the strong convexity and smoothness of the cost function and regularization. The relevant properties of proximal operators are reminded in appendix \ref{appendix : prox_prop}, while the subsequent derivation of the Lipschitz constants is detailed in appendix \ref{conv_proof} and yields:
\begin{align}
    \label{lip_const}
    \omega_{1}^{(t)}&= \dfrac{\hat{Q}^{(t)}_{1x}}{\hat{Q}^{(t)}_{2x}} \sqrt{1+\dfrac{(\hat{Q}^{(t)}_{2x})^{2}-(\hat{Q}^{(t)}_{1x})^{2}}{(\hat{Q}^{(t)}_{1x} + \lambda_2)^2}}\\
    \omega_{2}^{(t)}&= \dfrac{\hat{Q}^{(t)}_{1z}}{\hat{Q}^{(t)}_{2z}} \sqrt{1+\dfrac{(\hat{Q}^{(t)}_{2z})^{2}-(\hat{Q}^{(t)}_{1z})^{2}}{(\hat{Q}^{(t)}_{1z} + \tilde{\lambda}_2)^2}}.
\end{align}
We thus define the constraints matrices
\begin{equation}
    \mathbf{M}^{(t)}_{1} = \begin{bmatrix}(\omega_{1}^{(t)})^{2} & 0 \\0 & -1\end{bmatrix} \otimes \mathbf{I}_{N} \quad \mathbf{M}^{(t)}_{2} = \begin{bmatrix}(\omega_{2}^{(t)})^{2} & 0 \\0 & -1\end{bmatrix} \otimes \mathbf{I}_{M}
\end{equation}
where $\otimes$ denotes the Kronecker product.
We then use a time dependent form of Theorem 4 from \cite{lessard2016analysis} in the appropriate form for 2-layer MLVAMP, as was done in \cite{nishihara2015general} for ADMM.
\newline
\begin{proposition}(Time dependent version of Theorem 4 from \cite{lessard2016analysis})
\label{find_Lyap}
    Consider, at each time step $t\in \mathbb{N}$, the following linear matrix inequality with $\tau_{(t)} \in [0,1]$:
    \begin{align}
    \label{the_LMI}
    0 &\succeq \begin{bmatrix}(\mathbf{A}^{(t)})^{T}\mathbf{P}\mathbf{A}^{(t)}-(\tau_{(t)})^{2}\mathbf{P} & (\mathbf{A}^{(t)})^{T}\mathbf{P}\mathbf{B}^{(t)} \\
    (\mathbf{B}^{(t)})^{T}\mathbf{P}\mathbf{A}^{(t)} & (\mathbf{B}^{(t)})^{T}\mathbf{P}\mathbf{B}^{(t)} \end{bmatrix}\\
    &+\begin{bmatrix}\mathbf{C}^{(t)}_{1} & \mathbf{D}^{(t)}_{1} \\\mathbf{C}^{(t)}_{2} & \mathbf{D}^{(t)}_{2}\end{bmatrix}^{T}\begin{bmatrix}\beta^{(t)}_{1} \mathbf{M}^{(t)}_{1} & \mathbf{0}_{2N \times 2M}  \\\mathbf{0}_{2M \times 2N} & \beta^{(t)}_{2}\mathbf{M}^{(t)}_{2}\end{bmatrix}\begin{bmatrix}\mathbf{C}^{(t)}_{1} & \mathbf{D}^{(t)}_{1} \\\mathbf{C}^{(t)}_{2} & \mathbf{D}^{(t)}_{2}\end{bmatrix} \notag
\end{align}
If, at each time step, (\ref{the_LMI}) is feasible for some $\mathbf{P} \succ 0$ and $\beta_1^{(t)},\beta_2^{(t)} \geqslant 0$, then for any initialization $\mathbf{h}^{(0)}$, $\mathbf{h}^{(t)}$ converges to $\mathbf{h^*}$, the fixed point of \eqref{dyn_syst1}-\eqref{dyn_syst2}:
\begin{equation}
    \forall t, \quad \norm{\mathbf{h}^{(t)}-\mathbf{h}^*} \leqslant \sqrt{\kappa(\mathbf{P})}(\tau^{*})^{t}\norm{\mathbf{h}^{(0)}-\mathbf{h^*}}
\end{equation}
where $\kappa(\mathbf{P})$ is the condition number of $\mathbf{P}$ and we defined $\tau^{*} = \sup_{t}\tau_{(t)}$.
\end{proposition}
\begin{proof}
see appendix \ref{time_lyap_proof}
\end{proof}
\quad \\
We show in appendix \ref{conv_proof} how the additional ridge penalties from the constrained problem (\ref{smooth_student}) parametrized by $\lambda_{2}, \tilde{\lambda}_{2}$ can be used to make (\ref{the_LMI}) feasible and prove Lemma \ref{conv_lemma}. The core idea is to leverage on the Lipschitz constants (\ref{lip_const}), the operator norms of the matrices defined in (\ref{rec_op}) and the following upper and lower bounds on the $\hat{Q}$ parameters defined by the fixed point of state evolution equations:
\begin{align}
    &\lambda_{min}(\mathcal{H}_f) \leqslant \hat{Q}^{(t)}_{2x} \leqslant \lambda_{max} (\mathcal{H}_f) \\
    &\lambda_{min}(\mathcal{H}_g) \leqslant \hat{Q}^{(t+1)}_{2z} \leqslant \lambda_{max} (\mathcal{H}_g)   \\
    &\hat{Q}^{(t)}_{2z} \lambda_{min}(\mathbf{F}^T \mathbf{F}) \leqslant \hat{Q}^{(t+1)}_{1x} \leqslant \hat{Q}^{(t)}_{2z} \lambda_{max}(\mathbf{F}^T \mathbf{F}) \\
    &\dfrac{\hat{Q}^{(t)}_{2x}}{\lambda_{max}(\mathbf{F} \mathbf{F}^T)} \leqslant \hat{Q}^{(t)}_{1z} \leqslant \dfrac{\hat{Q}^{(t)}_{2x}}{\lambda_{min}(\mathbf{F} \mathbf{F}^T)},
\end{align}
where $\mathcal{H}_{f},\mathcal{H}_{g}$ are the Hessian of the loss and regularization functions taken at the fixed point. These bounds are obtained from the definitions of $\chi_{x},\chi_{z}$ in the state evolution equations (or equivalently in Theorem \ref{main_th}), and the fact that the derivative of a proximal operator reads, for a twice differentiable function:
\begin{equation}
    \mathcal{D}_{\mbox{$\eta_{\gamma f}$}}(\mathbf{x}) = (\text{Id}+\gamma \mathcal{H}_{f}(\mbox{$\eta_{\gamma f}$}(\mathbf{x})))^{-1}.
\end{equation}
Detail of this derivation can also be found in appendices \ref{appendix : prox_prop} and \ref{conv_proof}. For the constrained problem (\ref{smooth_student}), the maximum and minimum eigenvalues of the Hessians are directly augmented by $\tilde{\lambda}_{2},\lambda_{2}$, which allows us to control the scaling of the $\hat{Q}$ parameters. The rest of the convergence proof is then based on successive application of Schur's lemma \cite{horn2012matrix} on the linear matrix inequality (\ref{the_LMI}); and translating the resulting conditions on inequalities which can be verified by choosing the appropriate $\tilde{\lambda}_{2},\lambda_{2},\beta_{1}^{(t)}, \beta_{2}^{(t)}$. Convergence of gradient-based descent methods for sufficiently strongly-convex objectives is a coherent result from an optimization point of view. This is corroborated by the symbolic convergence rates derived for ADMM in \cite{nishihara2015general}, where a sufficiently strongly convex objective is also considered.
\subsection{Numerical experiments for Lemma \ref{conv_lemma}}
Here we provide numerical evidence for the linear convergence condition proved in Lemma 3. We consider a logistic regression penalized with the $\ell_{1}$ norm ($\lambda_{1}$ = 0.1) with an ill-conditioned design matrix, with i.i.d. standard normal elements. This corresponds to the setting of Figure \ref{figure3}. Since the logistic loss is strongly convex on any compact space, we do not need to add $\tilde{\lambda}_{2}$. We follow the convergence of 2-layer MLVAMP for this problem for increasing values of an additional ridge penalty $\lambda_{2} = 0,0.01,0.05,0.1$ and plot the average distance between successive iterates $\frac{1}{N}\norm{\mathbf{h}_{1x}^{(t+1)}-\mathbf{h}_{1x}^{(t+1)}}_{2}^{2}$ and the evolution of the reconstruction angle $\theta$ as a function of the number of iterations. We perform two experiments with aspect ratios $\alpha = 1$ and $\alpha = 0.2$. For $\alpha = 1$, 2-layer MLVAMP converges without any additional ridge penalty, and convergence is accelerated by larger values of $\lambda_{2}$. As a sanity check, note that the reconstruction angle of the estimator returned by the algorithm for $\lambda_{2} = 0$ (grey line on the lower left plot) converges to the value predicted at Figure \ref{figure3} for $\alpha=1,\lambda_{1}=0.1$ and a Gaussian matrix. For $\alpha=0.2$n the design matrix is ill-conditioned and we see that 2-layer MLVAMP diverges. Adding the ridge penalty leads to converging trajectories for a sufficiently large value of $\lambda_{2}$, as shown on the upper right block. Larger values of $\lambda_{2}$ again lead to faster convergence.
\begin{figure*}[!ht]
    \centering
    \includegraphics[scale=0.6]{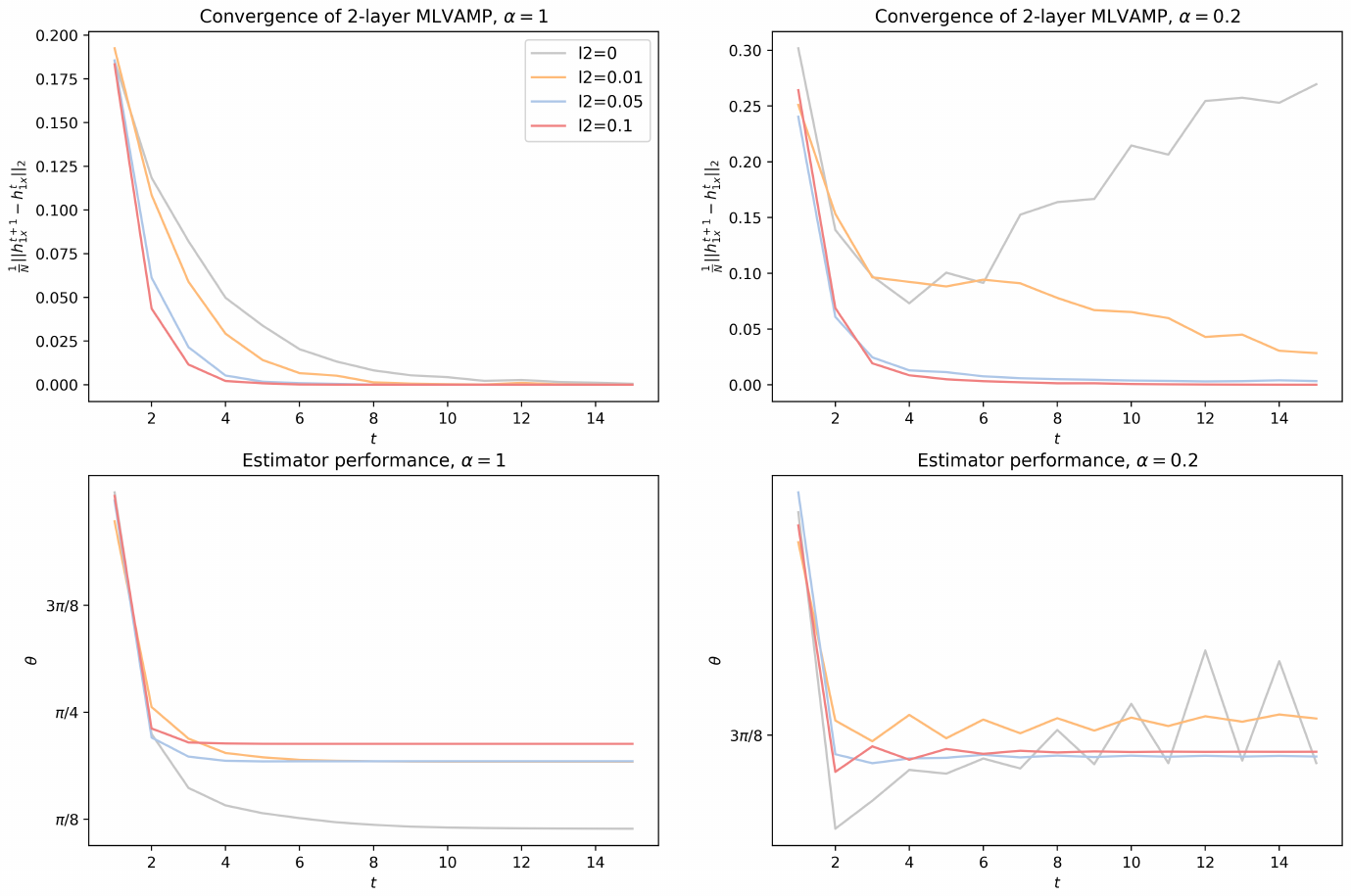}
    \caption{Convergence of 2-layer MLVAMP on a logistic regression with $\ell_{1}$ penalty with $\lambda_{1} = 0.1$, a Gaussian design matrix and two values of the aspect ratio $\alpha=1$ (left) and $\alpha=0.2$ (right). For $\alpha=1$, the algorithm converges regardless of the additional ridge penalty and we recover the performance predicted by Theorem \ref{main_th} for the plain $\ell_{1}$ regularization. For $\alpha=0.2$, the plain $\ell_{1}$ leads to an unstable iteration and a sufficiently large additional ridge indeed leads to convergence. In both cases, the larger the additional ridge, the faster the algorithm converges.}
    \label{convl2fig}
\end{figure*}\\
\quad \\
\section*{Acknowledgments}
The authors would like to thank Andrea Montanari, Benjamin Aubin, Yoshiyuki Kabashima and Lenka Zdeborov\'a for  discussions. This work is supported by the French Agence Nationale de la Recherche under grant ANR-17-CE23-0023-01 PAIL and ANR-19-P3IA-0001 PRAIRIE. Additional funding is acknowledged from ``Chaire de recherche sur les mod\`eles et sciences des donn\'ees'', Fondation CFM pour la Recherche-ENS.  

\newpage
\bibliographystyle{alpha}
\bibliography{references}
\newpage
\appendix
\section{Convergence of vector sequences}
\label{appendix:analysis_framework}
This section is a brief summary of the framework originally introduced in \cite{bayati2011dynamics} and used in \cite{fletcher2018inference, rangan2019vector}. We review the key definitions and verify that they apply in our setting. We remind the full set of state evolution equations from \cite{fletcher2018inference} at (\ref{rigorous_SE}), when applied to learning a GLM, in appendix \ref{SE_append}, along with the required assumptions for them to hold in appendix \ref{SE_assumptions}.  \\ The main building blocks are the notions of \emph{vector sequence} and \emph{pseudo-Lipschitz function}, which allow to define the \emph{empirical convergence with p-th order moment}.
Consider a vector of the form
\begin{equation}
    \mathbf{x}(N) = (\mathbf{x}_{1}(N),...,\mathbf{x}_{N}(N))
\end{equation}
where each sub-vector $\mathbf{x}_{n}(N) \in \mathbb{R}^{r}$ for any given $r \in \mathbb{N}^{*}$. For r=1, which we use in Theorem \ref{main_th}, $\mathbf{x}(N)$ is denoted a \emph{vector sequence}. \\
Given $p\geqslant 1$, a function $\mathbf{f}:\mathbb{R}^{r}\to \mathbb{R}^{s}$ is said to be \emph{pseudo-Lipschitz continuous of order p} if there exists a constant $C>0$ such that for all $\mathbf{x}_{1}, \mathbf{x}_{2} \in \mathbb{R}^{s}$:
\begin{equation}
    \norm{\mathbf{f}(\mathbf{x}_{1})-\mathbf{f}(\mathbf{x}_{2})} \leqslant C \norm{\mathbf{x}_{1}-\mathbf{x}_{2}}\left[1+\norm{\mathbf{x}_{1}}^{p-1}+\norm{\mathbf{x}_{2}}^{p-1}\right]
\end{equation}
Then, a given vector sequence $\mathbf{x}(N)$ \emph{converges empirically with p-th order moment} if there exists a random variable $X \in \mathbb{R}^{r}$ such that:
\begin{itemize}
    \item $\mathbb{E}\norm{X}_{p}^{p}<\infty$; and
    \item for any scalar-valued pseudo-Lipschitz continuous $\mathbf{f}:\mathbb{R}^{r} \to \mathbb{R}$ of order p,
    \begin{equation}
        \lim_{N\to \infty}\frac{1}{N}\sum_{n=1}^{N}\mathbf{f}(x_{n}(N))=\mathbb{E}[f(X)] \thickspace
    \end{equation}
\end{itemize}
Note that defining an empirically converging singular value distribution implicitly defines a sequence of matrices $\mathbf{F}(N)$ using the definition of rotational invariance from the introduction. This naturally brings us back to the original definitions from \cite{bayati2011dynamics}.
An important point is that the almost sure convergence of the second condition holds for random vector sequences, such as the ones we consider in the introduction. Note that the noise vector $\mathbf{\omega}_{0}$ must also satisfy these conditions, and naturally does when it is an i.i.d. Gaussian one. We also remind the definition of \emph{uniform Lipschitz continuity}.\\

For a given mapping $\phi(\mathbf{x},A)$ defined on $\mathbf{x} \in \mathcal{X}$ and $A \in \mathbb{R}$, we say it is \emph{uniformly Lipschitz continuous} in $\mathbf{x}$ at $A = \bar{A}$ if there exists constants $L_{1}$ and $L_{2} \geqslant 0$ and an open neighborhood U of $\bar{A}$ such that:
\begin{equation}
    \norm {\phi(\mathbf{x}_{1},A)-\phi(\mathbf{x}_{2},A)} \leqslant L_{1}\norm{\mathbf{x}_{1}-\mathbf{x}_{2}}
\end{equation}
for all $\mathbf{x}_{1}, \mathbf{x}_{2} \in \mathcal{X}$ and $A \in U$; and
\begin{equation}
    \norm{\phi(\mathbf{x},A_{1})-\phi(\mathbf{x},A_{2})} \leqslant L_{2}(1+\norm{\mathbf{x}})\abs{A_{1}-A_{2}}
\end{equation}
for all $\mathbf{x} \in \mathcal{X}$ and $A_{1},A_{2} \in U$. \\

We discuss the required assumptions for the state evolution equations to hold in detail, and why they are verified in our setting, in appendix \ref{SE_assumptions}.
\section{Convex analysis and properties of proximal operators}
\label{appendix : prox_prop}
We start this section with a few useful definitions from convex analysis, which can all be found in textbooks such as \cite{bauschke2011convex}.
We then remind important properties of proximal operators, which we use in appendix \ref{conv_proof} to derive upper bounds on the Lipschitz constants of the non-linear operators $\tilde{\mathcal{O}}_{1}, \tilde{\mathcal{O}}_{2}$. In what follows, we denote $\mathcal{X}$ the Hilbert space with scalar inner product serving as input and output space, here $\mathbb{R}^{N}$ or $\mathbb{R}^{M}$. For simplicity, we will write all operators as going from $\mathcal{X}$ to $\mathcal{X}$.
\begin{definition}(Strong convexity) A proper closed function is $\sigma$-strongly convex with $\sigma >0$ if ${f-\frac{\sigma}{2}\norm{.}^{2}}$ is convex. If f is differentiable,
    the definition is equivalent to
    \begin{equation}
        f(x) \geqslant f(y) + \langle \nabla f(y), x-y \rangle + \frac{\sigma}{2} \norm{x-y}^{2}
    \end{equation}
    for all $x,y \in \mathcal{X}$.
\end{definition}

\begin{definition}(Smoothness for convex functions) A proper closed function $f$ is $\beta$-smooth with $\beta >0$ if $\frac{\beta}{2}\norm{.}^{2}-f$ is convex. If f is 
    differentiable, the definition is equivalent to 
    \begin{equation}
        f(x) \leqslant f(y) + \langle \nabla f(y), x-y \rangle + \frac{\beta}{2} \norm{x-y}^{2}
    \end{equation}
    for all $x,y \in \mathcal{X}$.
\end{definition}
An immediate consequence of those definitions is the following second order condition: for twice differentiable functions, $f$ is $\sigma$-strongly convex and $\beta$-smooth if and only if:
\begin{equation}
    \sigma \rm{Id} \preceq \mathcal{H}_{f} \preceq \beta \rm{Id}.
\end{equation}
\begin{definition}(Co-coercivity)
Let $T :\mathcal{X} \to \mathcal{X}$ and $\beta \in \mathbb{R}^{*}_{+}$. Then $T$ is $\beta$ co-coercive if $\beta T$ is firmly-nonexpansive, i.e.
\begin{equation}
    \langle \mathbf{x}-\mathbf{y}, T( \mathbf{x})-T(\mathbf{y}) \rangle \geqslant \beta \norm{T(\mathbf{x})-T(\mathbf{y})}_{2}^{2}
\end{equation}
for all $\mathbf{x},\mathbf{y} \in \mathcal{X}$.
\end{definition}
Proximal operators are 1 co-coercive or equivalently firmly-nonexpansive.
\begin{corollary}(Remark 4.24 \cite{bauschke2011convex}) A mapping $T : \mathcal{X} \to \mathcal{X}$ is $\beta$-cocoercive if and only if 
    $\beta$T is half-averaged. This means that T can be expressed as:
    \begin{equation}
        T = \frac{1}{2\beta}(\rm{Id}+S)
    \end{equation}
    where $S$ is a nonexpansive operator.
\end{corollary}
\begin{proposition}(Resolvent of the sub-differential \cite{bauschke2011convex})
\label{prop_resolvent}
The proximal mapping of a convex function $f$ is the resolvent of the sub-differential $\partial f$ of $f$:
  \begin{equation}
      \mbox{Prox}_{\gamma f} = (\rm{Id}+\gamma \partial f)^{-1}.
  \end{equation}
\end{proposition}
The following proposition is due to \cite{giselsson2016linear}, and is useful to determine upper bounds on the Lipschitz constant of update functions involving proximal operators.
\begin{proposition}(Proposition 2 from \cite{giselsson2016linear})
\label{proposition : gisel_idh}
Assume that $f$ is $\sigma$-strongly convex and 
    $\beta$-smooth and that $\gamma \in ]0, \infty[$. Then $\mbox{Prox}_{\gamma f} - \frac{1}{1+\gamma \beta}\rm{Id}$ is $\frac{1}{\frac{1}{1+\gamma \beta}-\frac{1}{1+\gamma \sigma}}$-cocoercive
    if $\beta > \sigma$ and 0-Lipschitz if $\beta = \sigma$. If $f$ has no smoothness constant, the same holds by taking $\beta = + \infty$.
\end{proposition}
We will use these definitions and properties to derive the Lipschitz constants of $\tilde{\mathcal{O}}_{1},\tilde{\mathcal{O}}_{2}$ in appendix \ref{conv_proof}.
\begin{lemma}{Jacobian of the proximal} \\
Using proposition \ref{prop_resolvent}, the proximal operator can be written, for any parameter $\gamma \in \mathbb{R^{+}}$ and $\mathbf{x}$ in the input space $\mathcal{X}$:
\begin{equation}
    \mbox{$\mbox{Prox}_{\gamma f}$}(\mathbf{x}) = \left(\rm{Id}+\gamma \partial f\right)^{-1}(\mathbf{x}).
\end{equation}
For any convex and differentiable function $f$, we have: 
\begin{equation}
    \mbox{$\mbox{Prox}_{\gamma f}$}(\mathbf{x})+\gamma \nabla f(\mbox{Prox}_{\gamma f}(\mathbf{x})) = \mathbf{x}
\end{equation}
For a twice differentiable $f$, applying the chain rule then yields:
\begin{equation}
    \mathcal{D}_{\mbox{$\mbox{Prox}_{\gamma f}$}}(\mathbf{x})+\gamma \mathcal{H}_{f}(\mbox{$\mbox{Prox}_{\gamma f}$}(\mathbf{x})) \mathcal{D}_{\mbox{$\mbox{Prox}_{\gamma f}$}}(\mathbf{x})  = \rm{Id}
\end{equation}
where $\mathcal{D}$ is the Jacobian matrix and $\mathcal{H}$ the Hessian. Since f is a convex function, its Hessian is positive semi-definite, and, knowing that $\gamma$ is strictly positive, the matrix $(\rm{Id}+\gamma \mathcal{H}_{f}(\mbox{$\mbox{Prox}_{\gamma f}$}))$ is invertible. We thus have: 
\begin{equation}
    \mathcal{D}_{\mbox{$\mbox{Prox}_{\gamma f}$}}(\mathbf{x}) = (\rm{Id}+\gamma \mathcal{H}_{f}(\mbox{$\mbox{Prox}_{\gamma f}$}(\mathbf{x})))^{-1}
\end{equation}
\end{lemma}
\begin{lemma}{Proximal of ridge regularized functions} \\
\label{prox_ridge}
Since we consider only separable functions, we can work with scalar version of the proximal operators.
The scalar proximal of a given function with an added ridge regularization can be written:
\begin{align}
    \mbox{Prox}_{\gamma (f+\frac{\lambda_{2}}{2}\norm{.}_{2}^{2})}(x) &= (\rm{Id}+\gamma (\partial f+\lambda_{2}))^{-1}(x) \\
    &= ((1+\gamma\lambda_{2})Id+\gamma f')^{-1}(x)
\end{align}
where the second equality is true only for differentiable $f$.
If $f$ is real analytic, we can apply the analytic inverse function theorem \cite{krantz2002primer} and verify analyticity in $\lambda_{2}$ of the proximal.
\end{lemma}
Finally, we remind a result from \cite{bauschke2011convex} describing the limiting behavior of regularized estimators for vanishing regularization.
\begin{proposition}(Theorem 26.20 from \cite{bauschke2011convex})
\label{approx_l2}
  Let f and h be proper, lower semi-continuous, convex functions defined on $\mathcal{X}$. Suppose that $\argmin f \cap \mbox{dom}(h) \neq \emptyset$ and that $h$ is coercive and strictly convex. Then $h$ admits a unique minimizer $\mathbf{x}_{0}$ over $\argmin f$ and , for every $\epsilon \in ]0,1[$, the regularized problem \begin{equation}
      \argmin_{\mathbf{x} \in \mathcal{X}} f(\mathbf{x})+\epsilon h(\mathbf{x})
  \end{equation}
  admits a unique solution $\mathbf{x}_{\epsilon}$. If we assume further that $h$ is uniformly convex on any closed ball of the input space, then $\lim_{\epsilon \to 0} \mathbf{x}_{\epsilon} = \mathbf{x}_{0}$.
\end{proposition}
\section{From replica potentials to Moreau envelopes}
\label{app:rep_mor}
Here we show how the potentials defined for the replica free energy of corollary $\ref{cor_free}$ can be mapped to Moreau envelopes in the zero temperature limit, i.e. $\beta \to \infty$ where $\beta$ is the inverse temperature. We consider the scalar case since the replica expressions are scalar. All functions are separable here, so any needed generalization to the multidimensional case is immediate. We start by reminding the definition of the Moreau envelope \cite{bauschke2011convex,parikh2014proximal} $\mathcal{M}_{\gamma f}$ of a proper, closed and convex function $f$ for a given $\gamma \in \mathbb{R}^{*}_{+}$ and any $z \in \mathbb{R}$:
\begin{equation}
    \mathcal{M}_{\gamma f}(z) = \inf_{x \in \mathbb{R}}\left\lbrace f(x)+(1/2\gamma)\norm{x-z}_{2}^{2}\right\rbrace
\end{equation}
The Moreau envelope can be interpreted as a smoothed version of a given objective function with the same minimizer. For $\ell_{1}$ minimization for example, it allows to work with a differentiable objective. By definition of the proximal operator we have the following identity:
\begin{align}
    \mbox{Prox}_{\gamma f}(z) &= \argmin_{x \in \mathbb{R}} \left\lbrace f(x)+(1/2\gamma)\norm{x-z}_{2}^{2}\right\rbrace \\
    \mathcal{M}_{\gamma f}(z) &= f(\mbox{Prox}_{\gamma f}(z))+\frac{1}{2}\norm{\mbox{Prox}_{\gamma f}(z)-z}^{2}_{2}
\end{align}
We can now match the replica potentials with the Moreau envelope. We start from the definition of said potentials, to which we apply Laplace's approximation:
\begin{align}
    &\phi_x(\hat{m}_{1x}, \hat{Q}_{1x}, \hat{\chi}_{1x}; x_0, \xi_{1x})
    =\lim_{\beta \rightarrow \infty}...\notag\\
    &\dfrac{1}{\beta} \log \int e^{-\frac{\beta\hat{Q}_{1x}}{2}x^2 + \beta(\hat{m}_{1x} x_0 + \sqrt{\hat{\chi}_{1x}}\xi_{1x})x - \beta f(x)}dx \\
    &= -\frac{\hat{Q}_{1x}}{2}(x^{*})^2 + (\hat{m}_{1x} x_0 + \sqrt{\hat{\chi}_{1x}}\xi_{1x})x^{*}- f(x^{*})
\end{align}
where
\begin{align}
x^{*} = \argmin_{x} \bigg\lbrace -&\frac{\hat{Q}_{1x}}{2}x^2+...\notag\\ &(\hat{m}_{1x} x_0 + \sqrt{\hat{\chi}_{1x}}\xi_{1x})x - f(x) \bigg\rbrace
\end{align}
This is an unconstraint convex optimization problem, thus its optimality condition is enough to characterize its set of minimizers:
\begin{align}
    &-\hat{Q}_{1x}x^{*}+(\hat{m}_{1x} x_0 + \sqrt{\hat{\chi}_{1x}}\xi_{1x})-\partial f(x^{*}) = 0 \\
    &\iff x^{*} = (Id+\frac{1}{\hat{Q}_{1x}}\partial f)^{-1}\left(\frac{\hat{m}_{1x} x_0 + \sqrt{\hat{\chi}_{1x}}\xi_{1x}}{\hat{Q}_{1x}}\right)\\
    &\iff x^{*} = \mbox{Prox}_{\frac{f}{\hat{Q}_{1x}}}\left(\frac{\hat{m}_{1x} x_0 + \sqrt{\hat{\chi}_{1x}}\xi_{1x}}{\hat{Q}_{1x}}\right)
\end{align}
Replacing this in the replica potential and completing the square, we get:
\begin{align}
    &\phi_{x}(\hat{m}_{1x}, \hat{Q}_{1x}, \hat{\chi}_{1x}; x_0, \xi_{1x}) = -f(\mbox{Prox}_{\gamma f}(X))...\notag\\
    &\hspace{2cm}-\frac{\hat{Q}_{1x}}{2}\norm{X-\mbox{Prox}_{\gamma f}(X)}_{2}^{2}+\frac{X^{2}}{2}\hat{Q}_{1x} \\
    &= \hat{Q}_{1x}\frac{X^{2}}{2}-\mathcal{M}_{\frac{1}{\hat{Q}_{1x}}f}(X)
\end{align}
where we used the shorthand $X =\frac{\hat{m}_{1x} x_0 + \sqrt{\hat{\chi}_{1x}}\xi_{1x}}{\hat{Q}_{1x}}$. 
\section{Fixed point of multilayer vector approximate message passing}
\label{fixed_point_proof}
Here we show that the fixed point of 2-layer MLVAMP coincides with the optimality condition of the convex problem \ref{student}, proving Lemma \ref{VAMP_fixed}.
Writing the fixed point of the scalar parameters of algorithm (\ref{GVAMP}), we get the following prescriptions on the scalar quantities:
\begin{align}
    &\dfrac{1}{\chi_x} \equiv \dfrac{1}{\chi_{1x}} = \dfrac{1}{\chi_{2x}} = \hat{Q}_{1x} + \hat{Q}_{2x}\\
    &\dfrac{1}{\chi_z} \equiv \dfrac{1}{\chi_{1z}} = \dfrac{1}{\chi_{2z}} = \hat{Q}_{1z} + \hat{Q}_{2z}  \label{var_fixedpoint}\\
    &\hat{Q}_{1x} \chi_{1x} + \hat{Q}_{2x} \chi_{2x} = 1 \\
    &\hat{Q}_{1z} \chi_{1z} + \hat{Q}_{2z} \chi_{2z} = 1 \label{alpha_fixedpoint}
\end{align}
and the following ones on the estimates, as proved in \cite{pandit2020inference} section III:
\begin{align}
    \hat{\mathbf{x}}_{1} &= \hat{\mathbf{x}}_{2} \hspace{0.8cm}
    \hat{\mathbf{z}}_{1} = \hat{\mathbf{z}}_{2} \\
    \hat{\mathbf{z}}_{1} &= \mathbf{F}\hat{\mathbf{x}}_{1} \hspace{0.52cm} 
    \hat{\mathbf{z}}_{2} = \mathbf{F}\hat{\mathbf{x}}_{2}
\end{align}
We would like the fixed point of MLVAMP to satisfy the following first-order optimality condition
\begin{equation}
    \partial f(\mathbf{\hat{x}}) + \mathbf{F}^T \partial g(\mathbf{F \hat{x}})=0, \label{firstorder_condition}
\end{equation}
which characterizes the unique minimizer of the unconstraint convex problem (\ref{student}).
Replacing $\mathbf{h}_{1x}$'s expression inside $\mathbf{h}_{2x}$ reads
\begin{align}
    \mathbf{h}_{2x} &= \left(\dfrac{\mathbf{\hat{x}}_1}{\chi_x}-\hat{Q}_{1x} \mathbf{h}_{1x}\right)/\hat{Q}_{2x}\\
    &= \left(\dfrac{\mathbf{\hat{x}}_1}{\chi_x}-\left(\frac{\mathbf{\hat{x}}_2}{\chi_{x}}-\hat{Q}_{2x} \mathbf{h}_{2x}\right)\right)/\hat{Q}_{2x}
\end{align}
and using~\eqref{var_fixedpoint} we get $\mathbf{\hat{x}}_1=\mathbf{\hat{x}}_2$, and a similar reasoning gives  $\mathbf{\hat{z}}_2=\mathbf{\hat{z}}_1$. From~\eqref{g1+} and~\eqref{g1-}, we clearly find $\mathbf{\hat{z}}_2= \mathbf{F \hat{x}_2}$. Inverting the proximal operators in~\eqref{prox-f} and~\eqref{prox-g} yields
\begin{align}
    \mathbf{\hat{x}}_1+\frac{1}{\hat{Q}_{1x}}\partial g(\mathbf{\hat{x}}_1)&= \mathbf{h}_{1x} \label{r0-}\\
    \mathbf{\hat{z}}_1+\frac{1}{\hat{Q}_{1z}}\partial g(\mathbf{\hat{z}}_1)&= \mathbf{h}_{1z}.
\end{align}
Starting from the MLVAMP equation on $\mathbf{h}_{1x}$, we write
\begin{align}
    \mathbf{h}_{1x} &= \left(\frac{\mathbf{\hat{x}}_2}{\chi_x} - \hat{Q}_{2x} \mathbf{h}_{2x}\right) / \hat{Q}_{1x} \\
    &= \frac{\left(\frac{\mathbf{\hat{x}}_2}{\chi_x}  - (\hat{Q}_{2z}\mathbf{F}^T \mathbf{F} + \hat{Q}_{2x} \mbox{Id})\mathbf{\hat{x}}_2 + \hat{Q}_{2z} \mathbf{F}^T \mathbf{h}_{2z}\right)}{\hat{Q}_{1x}} \\
    &= - \frac{\left( \hat{Q}_{2z} \mathbf{F}^T \mathbf{F} + \hat{Q}_{2x} \left( 1 - \dfrac{1}{\chi_x \hat{Q}_{2x}}\right) \mbox{Id} \right) \mathbf{\hat{x}}_2}{ \hat{Q}_{2x}} \\
    &+\mathbf{F}^T \left( \hat{Q}_{1z} \left( \dfrac{1}{\chi_z \hat{Q}_{1z}}-1\right) \mathbf{\hat{z}}_1 - \partial\mathbf{g}(\mathbf{\hat{z}}_1) \right)
\end{align}
which is equal to the left-hand term in~\eqref{r0-}. Using this equality, as well as $\mathbf{\hat{z}}_1 = \mathbf{F \hat{x}_2}$ and relations~\eqref{var_fixedpoint} and~\eqref{alpha_fixedpoint} yields
\begin{equation}
    \partial f(\mathbf{\hat{x}}_2) + \mathbf{F}^T \partial g(\mathbf{F \hat{x}_2})=0.
\end{equation}
Hence, the fixed point of MLVAMP satisfies the optimality condition~\eqref{firstorder_condition} and is indeed the desired estimator: $\mathbf{\hat{x}}_1 = \mathbf{\hat{x}}_2 = \mathbf{\hat{x}}$.

\section{State evolution equations}
\label{SE_append}
This appendix is intended mainly for completeness, to show that the fixed point equations from Theorem \ref{main_th}, stemming from the heuristic state evolution written in \cite{takahashi2022macroscopic} are indeed made rigorous by the results presented in \cite{fletcher2018inference}.
\subsection{Heuristic state evolution equations}
The state evolution equations track the evolution of MLVAMP~\eqref{GVAMP} and provide statistical properties of its iterates. They are derived in~\cite{takahashi2022macroscopic} taking the heuristic assumption that $\mathbf{h_{1x}}, \mathbf{h_{1z}}, \mathbf{h_{2x}}, \mathbf{h_{2z}}$ behave as Gaussian estimates, which comes from the physics cavity approach:
\begin{subequations}
\label{SE-assumption}
\begin{align}
    \hat{Q}_{1x}^{(t)}\mathbf{h}_{1x}^{(t)} - \hat{m}_{1x}^{(t)} \mathbf{x_0} &\stackrel{PL2} = \sqrt{\hat{\chi}_{1x}^{(t)}} \boldsymbol{\xi}_{1x}^{(t)}\\
    \mathbf{V}^T (\hat{Q}_{2x}^{(t)} \mathbf{h}_{2x}^{(t)} - \hat{m}_{2x}^{(t)} \mathbf{x_0}) &\stackrel{PL2} = \sqrt{\hat{\chi}_{2x}^{(t)}} \boldsymbol{\xi}_{2x}^{(t)}
    \\
     \mathbf{U}^T (\hat{Q}_{1z}^{(t)}\mathbf{h}_{1z}^{(t)} - \hat{m}_{1z}^{(t)} \mathbf{z_0}) &\stackrel{PL2} = \sqrt{\hat{\chi}_{1z}^{(t)}} \boldsymbol{\xi}_{1z}^{(t)} \\
   \hat{Q}_{2z}^{(t)}\mathbf{h}_{2z}^{(t)} - \hat{m}_{2z}^{(t)} \mathbf{z_0} &\stackrel{PL2} = \sqrt{\hat{\chi}_{2z}^{(t)}} \boldsymbol{\xi}_{2z}^{(t)}
\end{align}
\end{subequations}
where $\stackrel{PL2} =$ denotes $PL2$ convergence. $\mathbf{U}$ and $\mathbf{V}$ come from the singular value decomposition $\mathbf{F}=\mathbf{U D V}^T$ and are Haar-sampled; $\mathbf{\xi}_{1x}^{(t)}, \mathbf{\xi}_{2x}^{(t)}, \mathbf{\xi}_{1z}^{(t)}, \mathbf{\xi}_{2z}^{(t)}$ are normal Gaussian vectors, independent from $\mathbf{x_0}, \mathbf{z_0}, \mathbf{V}^T \mathbf{x_0}$ and $\mathbf{U}^{T} \mathbf{z_0}$. Parameters $\hat{Q}_{1x}^{(t)}, \hat{Q}_{1z}^{(t)}$, $\hat{Q}_{2x}^{(t)}, \hat{Q}_{2z}^{(t)}$ are defined through MLVAMP's iterations~\eqref{GVAMP}; while parameters $\hat{m}_{1x}^{(t)}, \hat{m}_{1z}^{(t)}, \hat{m}_{2x}^{(t)}, \hat{m}_{2z}^{(t)}$ and $\hat{\chi}_{1x}^{(t)}, \hat{\chi}_{1z}^{(t)}, \hat{\chi}_{2x}^{(t)}, \hat{\chi}_{2z}^{(t)}$ are prescribed through SE equations. Other useful variables are the overlaps and squared norms of estimators, for $k\in \{ 1,2\}$:

\begin{align*}
    m_{kx}^{(t)} &= \frac{\mathbf{x}_0^\top \hat{\mathbf{x}}_k^{(t)}}{N} \quad \quad q_{kx}^{(t)} =\frac{ \|\hat{\mathbf{x}}_{k}^{(t)}\|_2^2}{N} \\
    m_{kz}^{(t)} &= \frac{\mathbf{z}_0^\top \hat{\mathbf{z}}_k^{(t)}}{M} \quad \quad q_{kz}^{(t)} = \frac{\|\hat{\mathbf{z}}_{k}^{(t)}\|_2^2}{M}.
\end{align*}
Starting from assumptions~\eqref{SE-assumption}, and following the derivation of~\cite{takahashi2022macroscopic} adapted to the iteration order from~\eqref{GVAMP}, the heuristic state evolution equations read: 
\begin{subequations}
\label{Kaba-fullSE}
\begin{align}
    \text{Initialize} \quad &\hat{Q}_{1x}^{(0)}, \hat{Q}_{2z}^{(0)}, \hat{m}_{1x}^{(0)}, \hat{m}_{2z}^{(0)}, \hat{\chi}_{1x}^{(0)}, \hat{\chi}_{2z}^{(0)}>0. \notag\\
    m_{1x}^{(t)}&=\mathbb{E}\left[x_{0}\eta_{f/\hat{Q}_{1x}^{(t)}}\left(\frac{\hat{m}_{1x}^{(t)}x_{0}+\sqrt{\hat{\chi}^{(t)}_{1x}}\xi_{1x}^{(t)}}{\hat{Q}_{1x}^{(t)}}\right)\right] \\
    \chi_{1x}^{(t)}&= \dfrac{1}{\hat{Q}_{1x}^{(t)}} \mathbb{E}\left[\eta'_{f/\hat{Q}_{1x}^{(t)}}\left(\frac{\hat{m}_{1x}^{(t)}x_{0}+\sqrt{\hat{\chi}^{(t)}_{1x}}\xi_{1x}^{(t)}}{\hat{Q}_{1x}^{(t)}}\right)\right] \\
    q_{1x}^{(t)}&= \mathbb{E}\left[\eta^{2}_{f/\hat{Q}_{1x}^{(t)}}\left(\frac{\hat{m}_{1x}^{(t)}x_{0}+\sqrt{\hat{\chi}^{(t)}_{1x}}\xi_{1x}^{(t)}}{\hat{Q}_{1x}^{(t)}}\right)\right] \label{q1x}\\
    \hat{Q}_{2x}^{(t)} &= \frac{1}{\chi_{1x}^{(t)}} - \hat{Q}_{1x}^{(t)} 
    \\
    \hat{m}_{2x}^{(t)} &= \frac{m_{1x}^{(t)}}{\rho_x \chi_{1x}^{(t)}} - \hat{m}_{1x}^{(t)} \label{m2hat}\\
     \hat{\chi}_{2x}^{(t)} &= \frac{q_{1x}^{(t)}}{(\chi_{1x}^{(t)})^2} - \frac{(m_{1x}^{(t)})^2}{\rho_x (\chi_{1x}^{(t)})^2} -  \hat{\chi}_{1x}^{(t)} \label{chi2xhat}\\
    m_{2z}^{(t)} &= \frac{\rho_x}{\alpha} \mathbb{E}\left[\frac{\lambda(\hat{m}_{2x}^{(t)} + \lambda \hat{m}_{2z}^{(t)})}{\hat{Q}_{2x}^{(t)} + \lambda \hat{Q}_{2z}^{(t)}}\right]\\
  \chi_{2z}^{(t)} &= \frac{1}{\alpha}\mathbb{E}\left[\frac{\lambda}{\hat{Q}_{2x}^{(t)} + \lambda\hat{Q}_{2z}^{(t)}}\right]\\
  q_{2z}^{(t)} &= \frac{1}{\alpha}\mathbb{E}\left[\frac{\lambda(\hat{\chi}_{2x}^{(t)} + \lambda\hat{\chi}_{2z}^{(t)})}{(\hat{Q}_{2x}^{(t)} + \lambda\hat{Q}_{2z}^{(t)})^2}\right]\\
  &\hspace{0.75cm}+\frac{\rho_x}{\alpha}\mathbb{E} \left[\frac{\lambda(\hat{m}_{2x}^{(t)} + \lambda\hat{m}_{2z}^{(t)})^2}{(\hat{Q}_{2x}^{(t)} + \lambda\hat{Q}_{2z}^{(t)})^2}\right]\\
  \hat{Q}_{1z}^{(t)} &= \frac{1}{\chi_{2z}^{(t)}} - \hat{Q}_{2z}^{(t)}\\
  \hat{m}_{1z}^{(t)} &= \frac{m_{2z}^{(t)}}{\rho_z \chi_{2z}^{(t)}} - \hat{m}_{2z}^{(t)}
      \end{align}
\begin{align}
   \hat{\chi}_{1z}^{(t)} &= \frac{q_{2z}^{(t)}}{(\chi_{2z}^{(t)})^2} - \frac{(m_{2z}^{(t)})^2}{\rho_z (\chi_{2z}^{(t)})^2} -  \hat{\chi}_{2z}^{(t)}\\
m_{1z}^{(t)}&=\mathbb{E}\left[z_{0}\eta_{g(y,.)/\hat{Q}_{1z}^{(t)}}\left(\frac{\hat{m}_{1z}^{(t)}z_{0}+\sqrt{\hat{\chi}^{(t)}_{1z}}\xi_{1z}^{(t)}}{\hat{Q}_{1z}^{(t)}}\right)\right]
\\
 \chi_{1z}^{(t)}&= \dfrac{1}{\hat{Q}_{1z}^{(t)}} \mathbb{E}\left[\eta'_{g(y,.)/\hat{Q}_{1z}^{(t)}}\left(\frac{\hat{m}_{1z}^{(t)}z_{0}+\sqrt{\hat{\chi}^{(t)}_{1z}}\xi_{1z}^{(t)}}{\hat{Q}_{1z}^{(t)}}\right)\right]
\\
q_{1z}^{(t)} &= \mathbb{E}\left[\eta^{2}_{g(y,.)/\hat{Q}_{1z}^{(t)}}\left(\frac{\hat{m}_{1z}^{(t)}z_{0}+\sqrt{\hat{\chi}^{(t)}_{1z}}\xi_{1z}^{(t)}}{\hat{Q}_{1z}^{(t)}}\right)\right]\\
\hat{Q}_{2z}^{(t+1)} &= \frac{1}{\chi_{1z}^{(t)}} - \hat{Q}_{1z}^{(t)}
\\
   \hat{m}_{2z}^{(t+1)}&= \frac{m_{1z}^{(t)}}{\rho_z \chi_{1z}^{(t)}} - \hat{m}_{1z}^{(t)} 
    \\
    \hat{\chi}_{2z}^{(t+1)} &= \frac{q_{1z}^{(t)}}{(\chi_{1z}^{(t)})^2} - \frac{(m_{1z}^{(t)})^2}{\rho_z (\chi_{1z}^{(t)})^2} -  \hat{\chi}_{1z}^{(t)}\\
    m_{2x}^{(t+1)} &= \rho_x \mathbb{E}
    \left[\frac{\hat{m}_{2x}^{(t)} + \lambda \hat{m}_{2z}^{(t+1)}}{\hat{Q}_{2x}^{(t)} + \lambda \hat{Q}_{2z}^{(t+1)}}\right] \label{chi2z}\\
    \chi_{2x}^{(t+1)} &= \mathbb{E}\left[\frac{1}{\hat{Q}_{2x}^{(t)} + \lambda \hat{Q}_{2z}^{(t+1)}}\right]\\
    q_{2x}^{(t+1)} &= \mathbb{E}
    \left[\frac{\hat{\chi}_{2x}^{(t)} + \lambda\hat{\chi}_{2z}^{(t+1)}}{(\hat{Q}_{2x}^{(t)} + \lambda\hat{Q}_{2z}^{(t+1)})^2}\right]\\
 &\hspace{0.75cm}+\rho_x \mathbb{E}\left[\frac{(\hat{m}_{2x}^{(t+1)} + \lambda\hat{m}_{2z}^{(t+1)})^2}{(\hat{Q}_{2x}^{(t)} + \lambda\hat{Q}_{2z}^{(t+1)})^2}\right] \\
  \hat{Q}_{1x}^{(t+1)} &= \frac{1}{\chi_{2x}^{(t+1)}} - \hat{Q}_{2x}^{(t)}
    \\
    \hat{m}_{1x}^{(t+1)} &= \frac{m_{2x}^{(t+1)}}{\rho_x \chi_{2x}^{(t+1)}} - \hat{m}_{2x}^{(t)}
    \\
    \hat{\chi}_{1x}^{(t+1)} &= \frac{q_{2x}^{(t+1)}}{(\chi_{2x}^{(t+1)})^2} - \frac{(m_{2x}^{(t+1)})^2}{\rho_x (\chi_{2x}^{(t+1)})^2} -  \hat{\chi}_{2x}^{(t)}.
\end{align}
\end{subequations}
We are interested in the fixed point of these state evolution equations, where  $\chi_{1x}^{(t)} = \chi_{2x}^{(t)} = \chi_x$, $q_{1x}^{(t)} = q_{2x}^{(t)} =q_x$, $m_{1x}^{(t)} = m_{2x}^{(t)} = m_x$, $\chi_{1z}^{(t)} = \chi_{2z}^{(t)} = \chi_z$, $q_{1z}^{(t)} = q_{2z}^{(t)} =q_z$, and $m_{1z}^{(t)} = m_{2z}^{(t)} = m_z$ are achieved. From there we easily recover eq. (\ref{thm1-equations}). However, these equations are not rigorous since the starting assumptions are not proven. Therefore, we will turn to a rigorous formalism to consolidate those results. 
\subsection{Necessary assumptions for the rigorous state evolution equations}
\label{SE_assumptions}
Here we remind the main assumptions needed for the rigorous state evolution equations to hold, as they are listed for Theorem 1 of \cite{fletcher2018inference}, and show they are verified in our setting.
\begin{assumption} 
\begin{itemize}
    \item [~]
    \item the empirical distributions of the underlying truth $\mathbf{x}_{0}$, eigenvalues of $\mathbf{F}^T \mathbf{F}$, and noise vector $w_{0}$, respectively converge with second order moments, as defined in appendix \ref{appendix:analysis_framework}, to independent scalar random variables $x_{0},w_{0},\lambda$ with distributions $p_{x_{0}}$, $p_{\lambda}$, $p_{w_{0}}$. We assume that the distribution $p_{\lambda}$ is not all-zero and has compact support. 
    \item the design matrix $\mathbf{F} = \mathbf{U}\mathbf{D}\mathbf{V}^{\top} \in \mathbb{R}^{M \times N}$ is rotationally invariant, as defined in the introduction, where the elements of the Haar distributed matrices $\mathbf{U},\mathbf{V}$ are independent of the random variables $x_{0},w_{0},\lambda$
    \item assume that $M,N \to \infty$ with fixed ratio $\alpha = M/N$ independent of $M,N$.
    \item the activation function $\phi(.,\mathbf{w}_{0})$ from Eq.\eqref{teacher} is pseudo-Lipschitz of order 2.
    \item the constants $\left\langle \partial_{\mathbf{h}_{1x}^{(t)}} g_{1x}(\mathbf{h}_{1x}^{(t)},\hat{Q}_{1x}^{(t)})\right\rangle,\left\langle \partial_{\mathbf{h}_{1z}^{(t)}} g_{1z}(\mathbf{h}_{1z}^{(t)},\hat{Q}_{1z}^{(t)})\right\rangle,\left\langle \partial_{\mathbf{h}_{2x}^{(t)}} g_{2x}(\mathbf{h}_{2x}^{(t)}, \mathbf{h}_{2z}^{(t+1)}, \hat{Q}_{2x}^{(t)}, \hat{Q}_{2z}^{(t+1)})\right\rangle$ \\ $\left\langle \partial_{\mathbf{h}_{2x}^{(t)}} g_{2z}(\mathbf{h}_{2x}^{(t)}, \mathbf{h}_{2z}^{(t)}, \hat{Q}_{2x}^{(t)}, \hat{Q}_{2z}^{(t)})\right\rangle$ from algorithm \eqref{GVAMP} are all in $[0,1]$.
    \item the component estimation functions $g_{1x}(\mathbf{h}_{1x}^{(t)},\hat{Q}_{1x}^{(t)}),g_{1z}(\mathbf{h}_{1z}^{(t)},\hat{Q}_{1z}^{(t)}),g_{2x}(\mathbf{h}_{2x}^{(t)}, \mathbf{h}_{2z}^{(t+1)}, \hat{Q}_{2x}^{(t)}, \hat{Q}_{2z}^{(t+1)}),$\\$g_{2z}(\mathbf{h}_{2x}^{(t)}, \mathbf{h}_{2z}^{(t)} \hat{Q}_{2x}^{(t)}, \hat{Q}_{2z}^{(t)})$ from algorithm \eqref{GVAMP} are uniformly Lipschitz continuous, at all time steps $t$, respectively in $\mathbf{h}_{1x}^{(t)}$ at $\hat{Q}_{1x}^{(t)}$, in $\mathbf{h}_{1z}^{(t)}$ at $\hat{Q}_{1z}^{(t)}$, $\mathbf{h}_{2x}^{(t)}$ at $\hat{Q}_{2x}^{(t)}$ and in $\mathbf{h}_{2z}^{(t)}$ at $\hat{Q}_{2z}^{(t)}$.
\end{itemize}
\end{assumption}
\quad \\
The first four points are included in the set of assumptions \ref{main_assum} and are therefore verified. We need to check the last two points, starting with the function $g_{1x}(\mathbf{h}_{1x}^{(t)},\hat{Q}_{1x}^{(t)})=\mbox{Prox}_{f/\hat{Q}_{1x}^{(t)}}(\mathbf{h}_{1x}^{(t)})$. Since proximal operators are firmly nonexpansive, they are 1-Lipschitz and we thus have, using the separability of the function $f$:
\begin{equation}
   \left\langle \partial_{\mathbf{h}_{1x}^{(t)}} g_{1x}(\mathbf{h}_{1x}^{(t)},\hat{Q}_{1x}^{(t)})\right\rangle =  \frac{1}{N}\sum_{i=1}^{N}\mbox{Prox}'_{f_{i}/\hat{Q}_{1x}^{(t)}}(\mathbf{h}_{1x,i}^{(t)}) \in [0,1]
\end{equation}
where each $f_{i} : \mathbb{R} \to \mathbb{R}$ is the same function applied to each coordinates. Now consider the restriction of $g_{1x}(\mathbf{h}_{1x}^{(t)},\hat{Q}_{1x}^{(t)})$ to its second argument. Its gradient w.r.t. $\hat{Q}_{1x}^{(t)}$ at a given point $\mathbf{h}_{1x}^{(t)}$ verifies, assuming the function $f$ is differentiable:
\begin{align}
    \norm{\nabla_{\hat{Q}_{1x}^{(t)}}\mbox{Prox}_{f/\hat{Q}_{1x}^{(t)}}(\mathbf{h}_{1x}^{(t)})}_{2} &= \norm{(Id+\frac{1}{\hat{Q}_{1x}^{(t)}}\mathcal{H}_{f}(\mbox{Prox}_{f/\hat{Q}_{1x}^{(t)}}(\mathbf{h}_{1x}^{(t)})))^{-1}\nabla f(\mathbf{h_{1x}^{(t)}})}_{2} \notag \\
    &\leqslant \norm{\nabla f(\mathbf{h}_{1x}^{(t)})}_{2} \notag \\
    &\leqslant C(1+\norm{\mathbf{h}_{1x}^{(t)}}_{2})
\end{align}
where the last line is obtained using the scaling conditions on the subdifferential of $f$ from assumption \ref{main_assum}.
Then, for any $\hat{Q}_{1x}^{(t)}, \hat{Q}_{1x}^{(t')}$, $\norm{\mbox{Prox}_{f/\hat{Q}_{1x}^{(t)}}-\mbox{Prox}_{f/\hat{Q}_{1x}^{(t')}}}_{2}\leqslant C(1+\norm{\mathbf{h}_{1x}^{(t)}}_{2})\abs{\hat{Q}_{1x}^{(t)}-\hat{Q}_{1x}^{(t')}}$ and $g_{1x}(\mathbf{h}_{1x}^{(t)},\hat{Q}_{1x}^{(t)})$ is uniformly Lipschitz in $\mathbf{h}_{1x}^{(t)}$ at $\hat{Q}_{1x}^{(t)}$, at any time index $t$. The argument is identical for $g_{1z}(\mathbf{h}_{1z}^{(t)},\hat{Q}_{1z}^{(t)})=\mbox{Prox}_{f/\hat{Q}_{1z}^{(t)}}(\mathbf{h}_{1z}^{(t)})$. The functions \\ $g_{2x}(\mathbf{h}_{2x}^{(t)}, \mathbf{h}_{2z}^{(t+1)}, \hat{Q}_{2x}^{(t)}, \hat{Q}_{2z}^{(t+1)}),g_{2z}(\mathbf{h}_{2x}^{(t)}, \mathbf{h}_{2z}^{(t)}, \hat{Q}_{2x}^{(t)}, \hat{Q}_{2z}^{(t)})$ have explicit expressions and it is straightforward to check the last two points using linear algebra and the assumptions on the spectrum of $\mathbf{F}^{\top}\mathbf{F}$.
\subsection{Rigorous state evolution formalism}
We now look into the state evolution equations derived for MLVAMP in~\cite{schniter2016vector}. Those equations are proven to be exact in the asymptotic limit, and follow the same algorithm as~\eqref{GVAMP}. In particular, they provide statistical properties of vectors $\mathbf{h}_{1x}, \mathbf{h}_{2x}, \mathbf{h}_{1z}, \mathbf{h}_{2z}$. We can read relations from~\cite{fletcher2018inference} using the following dictionary between our notations and theirs, valid at each iteration of the algorithm:
\begin{subequations}
\label{dictionary1}
\begin{align}
    \hat{Q}_{1x}, \hat{Q}_{2x}, \hat{Q}_{1z}, \hat{Q}_{2z} &\longleftrightarrow \gamma_0^-, \gamma_0^+, \gamma_1^+, \gamma_1^- \\
    \chi_{1x} \hat{Q}_{1x}, \chi_{2x} \hat{Q}_{2x} &\longleftrightarrow  \alpha_0^-, \alpha_0^+\\
    \chi_{1z} \hat{Q}_{1z}, \chi_{2z} \hat{Q}_{2z} &\longleftrightarrow  \alpha_1^-, \alpha_1^+ \\
    \mathbf{x_0}, \mathbf{z_0}, \rho_x, \rho_z &\longleftrightarrow \mathbf{Q}_0^0, \mathbf{Q}_1^0, \tau_0^0, \tau_1^0\\
    \mathbf{h}_{1x}, \mathbf{h}_{2x}, \mathbf{h}_{1z}, \mathbf{h}_{2z} &\longleftrightarrow \mathbf{r}_0^-, \mathbf{r}_0^+, \mathbf{r}_1^+, \mathbf{r}_1^-.
\end{align}
\end{subequations}
Placing ourselves in the asymptotic limit, \cite{fletcher2018inference} shows the following equalities:
\begin{subequations}
\label{r-distributions}
\begin{align}
    \mathbf{r}_0^- &= \mathbf{Q}_0^0 + \mathbf{Q}_0^-\\
    \mathbf{r}_0^+ &= \mathbf{Q}_0^0 + \mathbf{Q}_0^+ \\
    \mathbf{r}_1^- &= \mathbf{Q}_1^0 + \mathbf{Q}_1^- \\
    \mathbf{r}_1^+ &= \mathbf{Q}_1^0 + \mathbf{Q}_1^+
\end{align}
\end{subequations}
where $\mathbf{Q}_0^{-} \sim \mathcal{N}(0, \tau_0^{-})^N$ and $\mathbf{Q}_1^{-} \sim \mathcal{N}(0, \tau_1^{-})^N$ are i.i.d. Gaussian vectors. $\mathbf{Q}_0^+$, $\mathbf{Q}_1^+$ have the following norms and non-zero correlations with ground-truth vectors $\mathbf{Q}_0^0, \mathbf{Q}_1^0$:
\begin{align}
\tau_0^+ \equiv \dfrac{\norm{\mathbf{Q}_0^+}_2^2}{N}\hspace{1cm} c_0^+ \equiv \dfrac{\mathbf{Q}_0^{0T} \mathbf{Q_0^+}}{N}\\
\tau_1^+ \equiv \dfrac{\norm{\mathbf{Q}_1^+}_2^2}{M}\hspace{1cm}c_1^+ \equiv \dfrac{\mathbf{Q}_1^{0T} \mathbf{Q_1^+}}{M}.
\end{align}
With simple manipulations, we can rewrite~\eqref{r-distributions} as:
\begin{subequations}
\label{r-decorrelated}
\begin{align}
    \mathbf{r}_0^- &\deq \mathbf{Q}_{0} + \mathbf{Q}_0^- \\
    \mathbf{V}^T \mathbf{r}_0^+ &\deq \left( 1 + \dfrac{c_0^+}{\tau_0^0}\right) \mathbf{V}^T \mathbf{Q}_0^0 +\mathbf{V}^T \mathbf{\tilde{Q}}_0^+ \\
   \mathbf{r}_1^- &\deq \mathbf{Q}_1^0 + \mathbf{Q}_1^- \\
    \mathbf{U}^T \mathbf{r}_1^+ &\deq \left(1 + \dfrac{c_1^+}{\tau_1^0} \right) \mathbf{U}^T \mathbf{Q}_1^0 +\mathbf{U}^T \mathbf{\tilde{Q}}_1^+
\end{align}
\end{subequations}
where for $k \in \{1,2 \}$ vectors
\begin{equation}
    \mathbf{\tilde{Q}}_k^+ = - \dfrac{c_k^+}{\tau_k^0} \mathbf{Q}_k^0 + \mathbf{Q}_k^+
\end{equation}
and $\mathbf{Q}_0^-, \mathbf{Q}_1^-$
have no correlation with ground-truth vectors $\mathbf{Q}_0^0$, $\mathbf{Q}_1^0$, $\mathbf{U}^T \mathbf{Q}_0^0$, $\mathbf{V}^T \mathbf{Q}_1^0$. Besides, Lemma 5 from~\cite{rangan2019vector} states that $\mathbf{V}^T \mathbf{\tilde{Q}}_0^+$ and $\mathbf{U}^T \mathbf{\tilde{Q}}_1^+$ have components that converge empirically to Gaussian variables, respectively $\mathcal{N}(0, \tau_0^+)$ and $\mathcal{N}(0, \tau_1^+)$. Let us now translate this in our own terms, using the following relations that complete our dictionary with state evolution parameters:
\begin{subequations}
\label{dictionary2}
\begin{align}
    &\frac{\hat{m}_{1x}}{\hat{Q}_{1x}} \longleftrightarrow 1 \hspace{2.6cm} \frac{\hat{m}_{2z}}{\hat{Q}_{2z}} \longleftrightarrow 1\\ &\frac{\hat{m}_{2x}}{\hat{Q}_{2x}} \longleftrightarrow 1 + \frac{c_0^+}{\tau_0^0} \hspace{1.7cm}
    \frac{\hat{m}_{1z}}{\hat{Q}_{1z}} \longleftrightarrow  1 + \frac{c_1^+}{\tau_1^0}
    \\
    &\frac{\hat{\chi}_{1x}}{\hat{Q}_{1x}^2} \longleftrightarrow \tau_0^-\hspace{2.4cm} \frac{\hat{\chi}_{2z}}{\hat{Q}_{2z}^2} \longleftrightarrow \tau_1^-
    \\
    &\frac{\hat{\chi}_{2x}}{\hat{Q}_{2x}^2} \longleftrightarrow \tau_0^+ - \frac{(c_0^+)^2}{\tau_0^0} \hspace{1.1cm}  
    \frac{\hat{\chi}_{1z}}{\hat{Q}_{1z}^2} \longleftrightarrow \tau_1^+ - \frac{(c_1^+)^2}{\tau_1^0}.  
\end{align}
\end{subequations}
Simple bookkeeping transforms equations~\eqref{r-decorrelated} into a rigorous statement of starting assumptions~\eqref{r-distributions} from~\cite{takahashi2022macroscopic}. Since those assumptions are now rigorously 
established in the asymptotic limit, the remaining derivation of state evolution equations~\eqref{Kaba-fullSE} holds and provides a mathematically exact statement.
\subsection{Scalar equivalent model of state evolution}
For the sake of completeness, we will provide an overview of the explicit matching between the state evolution formalism from~\cite{fletcher2018inference} which was developed in a series of papers, and the replica formulation from~\cite{takahashi2022macroscopic} which relies on statistical physics methods. Although not necessary to our proof, it is interesting to develop an intuition about the correspondence between those two faces of the same coin. We have seen in the previous subsection that~\cite{fletcher2018inference} introduces ground-truth vectors $\mathbf{Q}_0^0, \mathbf{Q}_1^0$, estimates $\mathbf{r}_0^\pm, \mathbf{r}_1^\pm$ which are related to vectors $\mathbf{Q}_0^\pm, \mathbf{Q}_1
^\pm$. Let us introduce a few more vectors using matrices from the singular value decomposition $\mathbf{F}=\mathbf{U} \mathbf{D} \mathbf{V}^T$.
Let $\mathbf{s}_{\nu} \in \mathbb{R}^N$ be the vector containing all square roots of eigenvalues of $\mathbf{F}^T \mathbf{F}$ with $p_{\nu}$ its element-wise distribution; and  $\mathbf{s}_\mu \in \mathbb{R}^M$ the vector containing all square roots of eigenvalues of $\mathbf{F} \mathbf{F}^T$ with $p_{\mu}$ its element-wise distribution. Note that those two vectors contain the singular values of $\mathbf{F}$, but one of them also contains $\max(M,N) - \min(M,N)$ zero values. $p_\mu$ and $p_\nu$ are both well-defined since $p_\lambda$ is properly defined in Assumptions~\ref{main_assum}. We also define
\begin{align*}
    \mathbf{P}_0^0 &= \mathbf{V}^T \mathbf{Q}_0^0 \hspace{0.5cm} \mathbf{P}_0^+ = \mathbf{V}^T \mathbf{Q}_0^+ \hspace{0.5cm} \mathbf{P}_0^- = \mathbf{V}^T \mathbf{Q}_0^- \\
    \mathbf{P}_1^0 &= \mathbf{U} \mathbf{Q}_1^0 \hspace{0.7cm} \mathbf{P}_1^+ = \mathbf{U} \mathbf{Q}_1^+ \hspace{0.75cm} \mathbf{P}_1^- = \mathbf{U} \mathbf{Q}_1^-.
\end{align*}
By virtue of Lemma 5 from~\cite{rangan2019vector}, the six previous vectors have elements that converge empirically to a Gaussian variable. Hence, all defined vectors have an element-wise separable distribution, and we can write the state evolution as a scalar model on random variables sampled from those distributions. To do so, we will simply write the variables without the bold font: for instance $Z_0^0 \sim p_{x_0}$, $s_\nu \sim p_\nu$, and $Q_0^-$ refers to the random variable distributed according to the element-wise distribution of vector $\mathbf{Q}_0^-$. The scalar random variable state evolution from~\cite{fletcher2018inference} now reads:
\begin{subequations}
\label{rigorous_SE}
\begin{align}
    &\mbox{Initialize }  \gamma_{1}^{-(0)},\gamma_{0}^{-(0)},\tau_{0}^{-(0)},\tau_{1}^{-(0)},
    \\&Q_{0}^{-(0)}\hspace{-0.1cm} \sim \mathcal{N}(0,\tau_{0}^{-(0)}),Q_{1}^{-(0)}\hspace{-0.1cm}\sim \mathcal{N}(0,\tau_{1}^{-(0)}), \alpha_{0}^{-(0)},\alpha_{1}^{-(0)} \notag
    \\ 
    &\textit{Initial pass (ground truth only)} \notag \\ 
    &s_{\nu} \sim p_\nu, \hspace{0.5cm}s_\mu \sim p_\mu, \hspace{0.5cm} Q_{0}^{0}\sim p_{x_{0}}  \\
    &\tau_{0}^{0} = \mathbb{E}[(Q_{0}^{0})^{2}] \hspace{0.65cm} P_{0}^{0} \sim \mathcal{N}(0,\tau_{0}^{0}) \\
    &Q_{1}^{0} = s_{\mu} P_{0}^{0} \hspace{1.1cm} \tau_{1}^{0} = \mathbb{E}[(s_{\mu}P_{0}^{0})^{2}] = \mathbb{E}[(s_{\mu})^{2}]\tau_{0}^{0} \\ &P_{1}^{0} \sim \mathcal{N}(0,\tau_{1}^{0})\\ \notag \\
    &\textit{Forward Pass (estimation):} \notag \\
    &\alpha_{0}^{+(t)} = \mathbb{E}\left[\eta'_{f/\gamma_{0}^{-(t)}}(Q_0^0+Q_{0}^{-(t)})\right] \label{alpha0+} \\
    &\gamma_{0}^{+(t)} = \frac{\gamma_{0}^{(t)}}{\alpha_{0}^{+(t)}}-\gamma_{0}^{-(t)} \\
    &Q_{0}^{+(t)} = \frac{1}{1-\alpha_{0}^{+(t)}}\bigg\lbrace\eta_{f/\gamma_{0}^{-(t)}}(Q_0^0+Q_{0}^{-(t)})-... \notag
    \\&\hspace{5cm}Q_0^0-\alpha_{0}^{+}Q_{0}^{-(t)}\bigg\rbrace \label{up_SE1} \\ &\mathbf{K}_{0}^{+(t)} = Cov\left(Q_{0}^{0},Q_{0}^{+(t)}\right)\\ &\left(P_{0}^{0},P_{0}^{+(t)}\right) \sim \mathcal{N}\left(0,\mathbf{K}_{0}^{+(t)}\right)
        \end{align}
\begin{align}
    &\alpha_{1}^{+(t)} = \mathbb{E}\left[\frac{s_\mu^{2}\gamma_{1}^{-(t)}}{\gamma_{1}^{-(t)}s_{\mu}^{2}+\gamma_{0}^{+(t)}}\right] \\
    &\gamma_{1}^{+(t)} = \frac{\gamma_{1}^{-(t)}}{\alpha_{1}^{+(t)}}-\gamma_{1}^{-(t)} \\
    &Q_{1}^{+(t)} = \frac{1}{1-\alpha_{1}^{+(t)}}\bigg\lbrace \frac{s_\mu^{2}\gamma_{1}^{-(t)}}{\gamma_{1}^{-(t)}s_{\mu}^{2}+\gamma_{0}^{+(t)}}(Q_{1}^{-(t)}+Q_{1}^{0})+...\notag\\
    &\frac{s_{\mu}\gamma_{0}^{+(t)}}{\gamma_{1}^{-(t)}s_{\mu}^{2}+\gamma_{0}^{+(t)}}(P_{0}^{+(t)}+P_{0}^{0})-Q_{1}^{0}-\alpha_{1}^{+(t)}Q_{1}^{-(t)}\bigg\rbrace \label{up_SE2}\\
    &\mathbf{K}_{1}^{+(t)} = Cov\left(Q_{1}^{0},Q_{1}^{+(t)}\right)\\ &\left(P_{1}^{0},P_{1}^{+(t)}\right) \sim \mathcal{N}\left(0,\mathbf{K}_{1}^{+(t)}\right) 
    \\ \notag \\
    &\textit{Backward Pass (estimation):} \notag \\
    &\alpha_{1}^{-(t+1)}=\mathbb{E}\left[\eta_{g(y,.)/\gamma_{1}^{+(t)}}(P_{1}^{0}+P_{1}^{+(t)})\right] \\ &\gamma_{1}^{-(t+1)} = \frac{\gamma_{1}^{+(t)}}{\alpha_{1}^{-(t+1)}}-\gamma_{1}^{+(t)} \\
    & P_{1}^{-(t+1)} = \frac{1}{1-\alpha_{1}^{-(t+1)}}\bigg\lbrace\eta_{g(y,.)/\gamma_{1}^{+(t)}}(P_{1}^{0}+P_{1}^{+(t)}) \notag
    \\&\hspace{1.5cm}-P_{1}^{0}-\alpha_{1}^{-(t+1)}P_{1}^{+(t)}\bigg\rbrace \label{up_SE3} \\
    &\tau_{1}^{-(t+1)} = \mathbb{E}\left[(P_{1}^{-(t+1)})^{2}\right] \hspace{0.5cm}Q_{1}^{-(t+1)} \sim \mathcal{N}(0,\tau_{1}^{-(t+1)})\\
    & \alpha_{0}^{-(t+1)} = \mathbb{E}\left[\frac{\gamma_{0}^{+(t)}}{\gamma_{1}^{-(t+1)}s_{\nu}^{2}+\gamma_{0}^{+(t)}}\right]\\
    &\gamma_{0}^{-(t+1)} = \frac{\gamma_{0}^{+(t)}}{\alpha_{0}^{-(t+1)}}-\gamma_{0}^{+(t)} 
\\
    & P_{0}^{-(t+1)} = \frac{1}{1-\alpha_{0}^{-(t+1)}}\bigg\lbrace\frac{s_{\nu}\gamma_{1}^{-(t)}}{\gamma_{1}^{-(t+1)}s_{\nu}^{2}+\gamma_{0}^{+(t)}}(Q_{1}^{-(t+1)}+Q_{1}^{0})\notag\\
    &+ \frac{\gamma_{0}^{+(t)}}{\gamma_{1}^{-(t+1)}s_{\nu}^{2}+\gamma_{0}^{+(t)}}(P_{0}^{+(t)}+P_{0}^{0}) -P_{0}^{0}-\alpha_{0}^{-(t+1)}P_{0}^{+(t)}\bigg\rbrace  \label{up_SE4} \\
    & \tau_{0}^{-(t+1)} = \mathbb{E}\left[(P_{0}^{-(t+1)})^{2}\right]\hspace{0.5cm} Q_{0}^{-(t+1)} \sim \mathcal{N}(0,\tau_{0}^{-(t+1)}).
\end{align}
\end{subequations} 

\subsection{Direct matching of the state evolution fixed point equations}
To be consistent, we should be able to show that equations~\eqref{rigorous_SE} allow us to recover equations~\eqref{Kaba-fullSE} at their fixed point. Although somewhat tedious, this task is facilitated using dictionaries~\eqref{dictionary1} and~\eqref{dictionary2}. We shall give here an overview of this matching through a few examples.
\begin{itemize}
    \item Recovering equation~\eqref{m2hat}
\end{itemize}
Let us start from the rigorous scalar state evolution, in particular equation~\eqref{up_SE1} that defines variable $Q_0^{+}$. We get rid of time indices here since we focus on the fixed point. We first compute the correlation 
\begin{align}
    c_0^{+} &= \mathbb{E} \left[ Q_0^0 Q_0^+\right]\\ &=\dfrac{1}{1-\alpha_0^+} \left\lbrace \mathbb{E} \left[Q_0^0 \eta_{f/\gamma_0^-}(Q_0^0 + Q_0^-)\right] - \tau_0^0  \right\rbrace \label{c0+}
\end{align}
where we have used $\mathbb{E}[(Q_0^0)^2] = \tau_0^0$. At the fixed point, we know from MLVAMP or simply translating equations~\eqref{var_fixedpoint}, \eqref{alpha_fixedpoint} that 
\begin{equation*}
    1 - \alpha_0^+ = \alpha_0^-, \hspace{1cm} \dfrac{1}{\alpha_0^-} = \dfrac{\gamma_0^- + \gamma_0^+}{\gamma_0^+}, \hspace{1cm} \gamma_0^+ \alpha_0^+ = \gamma_0^- \alpha_0^-.
\end{equation*}
Simple manipulations take us to
\begin{align}
    &c_0^+ = \dfrac{\mathbb{E} \left[Q_0^0 \eta_{f/\gamma_0^-}(Q_0^0 + Q_0^-)\right]}{\alpha_0^-} - \tau_0^0 (1 + \dfrac{\gamma_0^-}{\gamma_0^+}) \\
    &\left( 1 + \dfrac{c_0^+}{\tau_0^0} \right) \gamma_0^+ = \dfrac{\mathbb{E} \left[Q_0^0 \eta_{f/\gamma_0^-}(Q_0^0 + Q_0^-)\right] \gamma_0^+}{\tau_0^0 \alpha_0^-} - \gamma_0^-. \label{Sundeep-eq1}
\end{align}
Now let us translate this back into our notations. The term $\mathbb{E} \left[Q_0^0 \eta_{f/\gamma_0^-}(Q_0^0 + Q_0^-)\right]$ simply translates into $m_{1x}$, and the rest of the terms can all be changed according to our dictionary. \eqref{Sundeep-eq1} exactly becomes
\begin{equation}
    \hat{m}_{2x} = \dfrac{m_{1x}}{\rho_x \chi_x} - \hat{m}_{1x},
\end{equation}
hence we perfectly recover equations~\eqref{m2hat} at the fixed point.
\begin{itemize}
    \item Recovering equation~\eqref{chi2xhat}
\end{itemize}
We start again from~\eqref{up_SE1} and square it:
\begin{align}
    &\mathbb{E}\left[ (Q_0^+)^2 \right] = \dfrac{1}{(1-\alpha_0^+)^2} \bigg\lbrace \mathbb{E}\left[ \eta_{f/\gamma_0^-}^2(Q_0^0 + Q_0^-)\right] +...\notag\\
    &
    (\alpha_0^+)^2 \mathbb{E}\left[ (Q_0^-)^2\right] - 2 \mathbb{E} \left[Q_0^0 \eta_{f/\gamma_0^-}(Q_0^0 + Q_0^-)\right]\notag\\
    &-2 \alpha_0^+ \mathbb{E}\left[Q_0^-\eta_{f/\gamma_0^-}^2(Q_0^0 + Q_0^-) +\mathbb{E}\left[ (Q_0^0)^2\right]\right] \bigg\rbrace \\
    &\tau_0^+ = \dfrac{1}{(1-\alpha_0^+)^2} \bigg\lbrace \mathbb{E}\left[ \eta_{f/\gamma_0^-}^2(Q_0^0 + Q_0^-)\right] + \tau_0^0 + ...\notag\\
    &(\alpha_0^+)^2 \tau_0^- - 2 \mathbb{E} \left[Q_0^0 \eta_{f/\gamma_0^-}(Q_0^0 + Q_0^-)\right]-...\notag\\
   &\hspace{2cm} 2 \alpha_0^+ \mathbb{E}\left[Q_0^-\eta_{f/\gamma_0^-}^2(Q_0^0 + Q_0^-) \right]  \bigg\rbrace. \label{beginning}
\end{align}
Since $Q_0^-$ is a Gaussian variable, independent from $Q_0^0$, we can use Stein's lemma and use  equation~\eqref{alpha0+} to get
\begin{equation}
    \mathbb{E}\left[Q_0^-\eta_{f/\gamma_0^-}^2(Q_0^0 + Q_0^-) \right] = \alpha_0^+ \tau_0^-. \label{steinQ0}
\end{equation}
Moreover, from~\eqref{c0+} we have
\begin{align}
    &(c_0^+)^2 (\alpha_0^-)^2= \left( \mathbb{E} \left[Q_0^0 \eta_{f/\gamma_0^-}(Q_0^0 + Q_0^-)\right] - \tau_0^0 \right)^2 \\
    &\dfrac{(c_0^+)^2 (\alpha_0^-)^2}{\tau_0^0} - \dfrac{(\mathbb{E} \left[Q_0^0 \eta_{f/\gamma_0^-}(Q_0^0 + Q_0^-)\right])^2}{\tau_0^0} = ... \notag\\
    &\hspace{1.5cm}-2 \mathbb{E} \left[Q_0^0 \eta_{f/\gamma_0^-}(Q_0^0 + Q_0^-)\right] + \tau_0^0. \label{usefulc0+}
\end{align}
Replacing~\eqref{steinQ0} and~\eqref{usefulc0+} into~\eqref{beginning}, we reach
\begin{align}
\label{finalform}
    &\left(\tau_0^+ - \dfrac{(c_0^+)^2}{\tau_0^0}\right) (\alpha_0^-)^2 = \mathbb{E}\left[ \eta_{f/\gamma_0^-}^2(Q_0^0 + Q_0^-)\right] \notag\\
    &- \dfrac{\left(\mathbb{E}\left[Q_0^0 \eta_{f/\gamma_0^-}(Q_0^0 + Q_0^-)\right]\right)^2}{\tau_0^0} - (\alpha_0^+)^2 \tau_0^- \\
    &\left( \tau_0^+ - \dfrac{(c_0^+)^2}{\tau_0^0}\right) (\gamma_0^+)^2 = \dfrac{\mathbb{E}\left[ \eta_{f/\gamma_0^-}^2(Q_0^0 + Q_0^-)\right] (\gamma_0^+)^2}{(\alpha_0^-)^2} \notag\\
    &- \dfrac{\left(\mathbb{E} \left[Q_0^0 \eta_{f/\gamma_0^-}(Q_0^0 + Q_0^-)\right]\right)^2 (\gamma_0^+)^2}{\tau_0^0 (\alpha_0^-)^2} - (\gamma_0^-)^2 \tau_0^-.
\end{align}
Notice that $\mathbb{E}\left[ \eta_{f/\gamma_0^-}^2(Q_0^0 + Q_0^-)\right]$ simply translates into our variable $q_{1x}$ from its definition~\eqref{q1x}, and our dictionary directly transforms~\eqref{finalform} into equation~\eqref{chi2xhat}:
\begin{equation}
    \hat{\chi}_{2x} = \frac{q_{1x}}{\chi_{1x}^2} - \frac{m_{1x}^2}{\rho_x \chi_{1x}^2} -  \hat{\chi}_{1x}.
\end{equation}
\begin{itemize}
    \item Recovering equation~\eqref{chi2z}
\end{itemize}
We first note that for any function $h$,
\begin{equation}
    \mathbb{E}[h(s_\nu)] = \min(1,\alpha) \mathbb{E}[h(s_\mu)] + \max(0,1- \alpha) h(0).
\end{equation}
and $s_\nu^2 \sim p_\lambda$. Applying this to $h(s) = \dfrac{\gamma_1^- s^2}{\gamma_1^- s^2 + \gamma_0^+}$ and starting from~\eqref{up_SE2}, we rewrite
\begin{align}
    \alpha_{1}^{+} &= \mathbb{E}\left[\frac{\gamma_{1}^{-} s_\mu^{2}}{\gamma_{1}^{-}s_{\mu}^{2}+\gamma_{0}^{+}}\right]\\
    &= \dfrac{1}{\alpha}\mathbb{E} \left[ \frac{\gamma_{1}^{-}\lambda}{\gamma_{1}^{-}\lambda+\gamma_{0}^{+}}\right]
\end{align}
with $\lambda \sim p_\lambda$, which translates into equation~\eqref{chi2z}:
\begin{equation}
    \chi_{2z} = \dfrac{1}{\alpha} \mathbb{E} \left[ \dfrac{\lambda}{\hat{Q}_{2x}+\lambda \hat{Q}_{2z}} \right].
\end{equation}
In a similar fashion, we can recover all equations~\eqref{Kaba-fullSE} by writing variances and correlations between scalar random variables defined in~\eqref{rigorous_SE}, and using the independence properties established in~\cite{fletcher2018inference}; thus directly showing the matching between the two state evolution formalisms at their fixed point.

\section{Numerical implementation details}
The plots were generated using the toolbox available at \url{https://github.com/cgerbelo/Replica_GLM_orth.inv.git} \\
\quad \\
\label{num_append}
Here we give a few derivation details for implementation of the equations presented in Theorem \ref{main_th}. We provide the Python script used to produce the figures in the main body of the paper as an example. The experimental points were obtained using the convex optimization tools of \cite{pedregosa2011scikit}, with a data matrix of dimension $N = 200, M = \alpha N$, for $\alpha \in [0.1,3]$. Each point is averaged 100 times to get smoother curves. The theoretical prediction was simply obtained by iterating the equations from Theorem \ref{main_th}. This can lead to unstable numerical schemes, and we include a few comments about stability in the code provided with this version of the paper. For Gaussian data, the design matrices were simply obtained by sampling a normal distribution $\mathcal{N}(0,\sqrt{1/M})$, effectively yielding the Marchenko-Pastur distribution \cite{tulino2004random} for averaging on the eigenvalues of $\mathbf{F}
^{T}\mathbf{F}$ in the state evolution equations :

\begin{equation}
    \lambda_{\mathbf{F}^{T}\mathbf{F}} \sim \max(0,1-\alpha)\delta(\lambda-0)+\alpha\frac{\sqrt{(0,\lambda-a)^{+}(0,b-\lambda)^{+}}}{2\pi\lambda}
\end{equation}
where $a = \sqrt{1-\left(\frac{1}{\alpha}\right)^{2}}, b= \sqrt{1+\left(\frac{1}{\alpha}\right)^{2}}$, and $(0,x)^{+}=\max(0,x)$.
For the example of orthogonally invariant matrix with arbitrary spectrum, we chose to sample the singular values of $\mathbf{F}$ from the uniform distribution $\mathcal{U}(\left[(1-\alpha)^{2},(1+\alpha)^{2}\right])$. This leads to the following distribution for the eigenvalues of $\mathbf{F}^{T}\mathbf{F}$:
\begin{equation}
\label{indicator}
    \lambda_{\mathbf{F}^{T}\mathbf{F}} \sim \max(0,1-\alpha)\delta(0)+\min(1,\alpha)d(\lambda,\alpha)
\end{equation}
where $d(\lambda,\alpha)=\left(\frac{1}{2((1+\alpha)^{2}-(1-\alpha)^{2})}\mathbb{I}_{\{\sqrt{\lambda}\in [(1-\alpha)^{2},(1+\alpha)^{2}]\}}\frac{1}{\sqrt{\lambda}}\right)$, and $\mathbb{I}$ is the indicator function. \\ 
\quad \\
The only quantities that need additional calculus are the averages of proximals, squared proximals and derivatives of proximals. Here we give the corresponding expressions for the losses and regularizations that were used to make the figures. Note that the stability and convergence of the state evolution equations closely follow the result of Lemma \ref{conv_lemma}. For example, a ridge regularized logistic regression, which is a strongly convex objective in both the loss (on compact spaces) and regularization  will lead to more stable iterations than a LASSO SVC.
\subsection{Regularization : elastic net}
For the elastic net regularization, we can obtain an exact expression, avoiding any numerical integration.
The proximal of the elastic net reads:
\begin{equation}
    \mbox{Prox}_{\frac{1}{\hat{Q}_{1x}}(\lambda_{1} \vert \mathbf{x} \vert_{1}+\frac{\lambda_{2}}{2}\norm{\mathbf{x}}_{2}^{2})}(.) = \frac{1}{1+\frac{\lambda_{2}}{\hat{Q}_{1x}}}\hspace{.1cm}s\hspace{-0.1cm}\left(.,\frac{\lambda_{1}}{\hat{Q}_{1x}}\right)
\end{equation}
where $s\left(.,\frac{\lambda_{1}}{\hat{Q}_{1x}}\right)$ is the soft-thresholding function: 
\begin{equation}
s\left(r_{1k},\frac{\lambda_{1}}{\hat{Q}_{1x}}\right) =
\left\lbrace
\begin{array}{ccc}
r_{1k}+\frac{\lambda_{1}}{\hat{Q}_{1x}}  & \mbox{if} & r_{1k}<-\frac{\lambda_{1}}{\hat{Q}_{1x}}\\
0 & \mbox{if} & -\frac{\lambda_{1}}{\hat{Q}_{1x}}<r_{1k}<\frac{\lambda_{1}}{\hat{Q}_{1x}}\\
r_{1k}-\frac{\lambda_{1}}{\hat{Q}_{1x}} & \mbox{if} & r_{1k}>\frac{\lambda_{1}}{\hat{Q}_{1x}}.
\end{array}\right.
\end{equation}
We assume that the ground-truth $x_{0}$ is pulled from a Gauss-Bernoulli law of the form:
\begin{equation}
    \phi(x_{0}) = (1-\rho)\delta(0)+\rho\frac{1}{\sqrt{2\pi \sigma^{2}}}\exp{(-x_{0}^{2}/(2\sigma^{2}))}.
\end{equation}
Note that we did our plots with $\rho =1$, but this form can be used to study the effect of sparsity in the model.
Writing $X =\frac{\hat{m}_{1x} x_0 + \sqrt{\hat{\chi}_{1x}}\xi_{1x}}{\hat{Q}_{1x}}$, and remembering that $\xi_{1x} \sim \mathcal{N}(0,1)$, some calculus then shows that:
\begin{figure*}[!t]
\begin{align}
&\mathbb{E}[\mbox{Prox}^{2}_{\mathbf{f}/\hat{Q}_{1x}}(X)] \notag \\  &=\left(\frac{1}{1+\frac{\lambda_{2}}{\hat{Q}_{1x}}}\right)^{2}\left[(1-\rho)\left(\frac{\lambda_{1}^{2}+\hat{\chi}_{1x}}{(\hat{Q}_{1x})^{2}}\erfc\left(\frac{\lambda_{1}}{\sqrt{2\hat{\chi}_{1x}}}\right)-\frac{\lambda_{1}\sqrt{2\hat{\chi}_{1x}}\exp(-\frac{\lambda_{1}^{2}}{2(\hat{\chi}_{1x})})}{\sqrt{\pi}}\right) \right.\\
&+ \left. \rho\left(\frac{\lambda_{1}^{2}+\hat{\chi}_{1x}+\sigma^{2}\hat{m}_{1x}^{2}}{(\hat{Q}_{1x})^{2}}\erfc\left(\frac{\lambda_{1}}{\sqrt{2(\hat{\chi}_{1x}+\sigma^{2}\hat{m}_{1x}^{2})}}\right)-\frac{\lambda_{1}\sqrt{2(\hat{\chi}_{1x}+\sigma^{2}\hat{m}_{1x}^{2})}\exp(-\frac{\lambda_{1}^{2}}{2(\hat{Q}_{1x})^{2}(\hat{\chi}_{1x}+\sigma^{2}\hat{m}_{1x}^{2})})}{\sqrt{\pi}}\right)\right]. \notag
\end{align}
Similarly, we have
\begin{align}
    \mathbb{E}[\mbox{Prox}^{'}_{\mathbf{f}/\hat{Q}_{1x}}(X)] = \frac{1}{1+\frac{\lambda_{2}}{\hat{Q}_{1x}}}\left[(1-\rho)\erfc\left(\frac{\lambda_{1}}{\sqrt{2\hat{\chi}_{1x}}}\right)+\rho\erfc\left(\frac{\lambda_{1}}{\sqrt{2(\hat{\chi}_{1x}+\sigma^{2}\hat{m}_{1x}^{2})}}\right)\right]
\end{align}
and 
\begin{align}
    \mathbb{E}[x_{0}\mbox{Prox}_{\mathbf{f}/\hat{Q}_{1x}}(X)] = \frac{\rho \abs{\sigma\hat{m}_{1x}}}{\hat{Q}_{1x}+\lambda_{2}}\erfc\left(\frac{\lambda_{1}}{\sqrt{2(\hat{\chi}_{1x}+\sigma^{2}\hat{m}_{1x}^{2})}}\right)
\end{align}
\hrulefill
\end{figure*}
We now turn to the loss functions.
    
\subsection{Loss functions}

The loss functions sometimes have no closed form, as is the case for the logistic loss. In that case, numerical integration cannot be avoided, and we recommend marginalizing all the possible variables that can be averaged out. In the present model, if the teacher $y$ is chosen as a sign, one-dimensional integrals can be reached, leading to stable and reasonably fast implementation (a few minutes to generate a curve comparable to those of Figure \ref{fig_data} for the non-linear models, the ridge regression being very fast). The interested reader can find the corresponding marginalized prefactors in the code jointly provided with this paper. \\
\paragraph{Square loss}
The square loss is defined as:
\begin{equation}
 f(x,y) = \frac{1}{2}(x-y)^{2},
\end{equation}
its proximal and partial derivative then read: 
\begin{align}
    \mbox{Prox}_{\frac{1}{\gamma}f}(p) &= \frac{\gamma}{1+\gamma}p+\frac{1}{1+\gamma}y \\
    \frac{\partial}{\partial p}\mbox{Prox}_{\frac{1}{\gamma}f}(p) &= \frac{\gamma}{1+\gamma}.
\end{align}

Using this form with a plain ridge penalty (elastic net with $\ell_{1}=0$) leads to great simplification in the equations of Theorem \ref{main_th} and we recover the classical expressions obtained for ridge regression in papers such as \cite{hastie2022surprises,gerbelot2020asymptotic}. \\

\paragraph{Hinge loss}
The hinge loss reads: 
\begin{equation}
    f(x,y) = \max(0,1-yx).
\end{equation}

Assuming $y \in \{-1,+1\}$, its proximal and partial derivative then read: 
\begin{align}
     \mbox{Prox}_{\frac{1}{\gamma}f}(p)&=\left\lbrace
    \begin{array}{ccc}
     p+\frac{y}{\gamma}& \mbox{if} & \gamma(1-yp)\geqslant 1\\
    y & \mbox{if} & 0\leqslant \gamma(1-yp) \leqslant 1\\
    p & \mbox{if} & \gamma(1-yp)\leqslant 0
    \end{array}\right. \\
   \frac{\partial}{\partial p}\mbox{Prox}_{\frac{1}{\gamma}f}(p)&= \left\lbrace
    \begin{array}{ccc}
     1 & \mbox{if} & \gamma(1-yp)\geqslant 1\\
    0 & \mbox{if} & 0\leqslant \gamma(1-yp) \leqslant 1\\
    1 & \mbox{if} & \gamma(1-yp)\leqslant 0.
    \end{array}\right.
\end{align}

\paragraph{Logistic loss}
\begin{equation}
    f(x,y) = \log(1+\exp(-yx))
\end{equation}

Its proximal (at point p) is the solution to the fixed point problem:
\begin{equation}
    x = p+\frac{y}{\gamma(1+\exp(yx))},
\end{equation}

and its derivative, given that the logistic loss is twice differentiable, reads:
\begin{align}
    \frac{\partial}{\partial p}\mbox{Prox}_{\frac{1}{\gamma}f}(p) &= \frac{1}{1+\frac{1}{\gamma}\frac{\partial^{2}}{\partial p^{2}}f(\mbox{Prox}_{\frac{1}{\gamma}f}(p))} \\
    &= \frac{1}{1+\frac{1}{\gamma}\frac{1}{(2+2\mbox{cosh}(\mbox{Prox}_{\frac{1}{\gamma}f}(p))}}. 
\end{align}

\section{Proof of Lemma \ref{conv_lemma}: Convergence analysis of 2-layer MLVAMP}
\label{conv_proof}
\noindent
In this section, we give the detail of the convergence proof of 2-layer MLVAMP.
\subsection{Proof of Proposition \ref{find_Lyap}}
\label{time_lyap_proof}
\noindent
This proof is quite straightforward and close to the one of Theorem 4 from \cite{lessard2016analysis}. \newline
\quad \\
Multiplying Eq.\eqref{the_LMI} on the left and right by $[(\mathbf{h}^{(t)}-\mathbf{h}^{(t-1)})^{\top} \quad(\mathbf{u}^{(t)}-\mathbf{u}^{(t-1)})^{\top}]$ and its transpose respectively, we get
\begin{align*}
    &(\mathbf{A}^{(t)}(\mathbf{h}^{(t)}-\mathbf{h}^{(t-1)})+\mathbf{B}^{(t)}(\mathbf{u}^{(t)}-\mathbf{u}^{(t-1)}))^{\top}\mathbf{P}(\mathbf{A}^{(t)}(\mathbf{h}^{(t)}-\mathbf{h}^{(t-1)})+\mathbf{B}^{(t)}(\mathbf{u}^{(t)}-\mathbf{u}^{(t-1)}))\\
    &-(\tau_{(t)})^{2}(\mathbf{h}^{(t)}-\mathbf{h}^{(t-1)})^\top\mathbf{P}(\mathbf{h}^{(t)}-\mathbf{h}^{(t-1)}) \\
    &+\beta_{1}^{(t)}(\mathbf{C}_{1}^{(t)}(\mathbf{h}^{(t)}-\mathbf{h}^{(t-1)})+\mathbf{D}_{1}(\mathbf{u}^{(t)}-\mathbf{u}^{(t-1)}))^{\top}\mathbf{M}_{1}^{(t)}(\mathbf{C}_{1}^{(t)}(\mathbf{h}^{(t)}-\mathbf{h}^{(t-1)})+\mathbf{D}_{1}(\mathbf{u}^{(t)}-\mathbf{u}^{(t-1)})) \\
    &+\beta_{2}^{(t)}(\mathbf{C}_{2}^{(t)}(\mathbf{h}^{(t)}-\mathbf{h}^{(t-1)})+\mathbf{D}_{2}(\mathbf{u}^{(t)}-\mathbf{u}^{(t-1)}))^{\top}\mathbf{M}_{2}^{(t)}(\mathbf{C}_{2}^{(t)}(\mathbf{h}^{(t)}-\mathbf{h}^{(t-1)})+\mathbf{D}_{2}(\mathbf{u}^{(t)}-\mathbf{u}^{(t-1)})) \leqslant 0
\end{align*}
Using the definition of the iteration \eqref{dyn_syst1}-\eqref{dyn_syst2}, this simplifies to 
\begin{align*}
    &(\mathbf{h}^{(t+1)}-\mathbf{h}^{(t)})^\top\mathbf{P}(\mathbf{h}^{(t+1)}-\mathbf{h}^{(t)})-(\tau_{(t)})^{2}(\mathbf{h}^{(t)}-\mathbf{h}^{(t-1)})^\top\mathbf{P}(\mathbf{h}^{(t)}-\mathbf{h}^{(t-1)}) \\
    &+\beta_{1}(\mathbf{w}_{1}^{(t)}-\mathbf{w}_{1}^{(t-1)})^{\top}\mathbf{M}_{1}^{(t)}(\mathbf{w}_{1}^{(t)}-\mathbf{w}_{1}^{(t-1)})+\beta_{2}(\mathbf{w}_{2}^{(t)}-\mathbf{w}_{2}^{(t-1)})^{\top}\mathbf{M}_{2}^{(t)}(\mathbf{w}_{2}^{(t)}-\mathbf{w}_{2}^{(t-1)})\leqslant 0
\end{align*}
Owing to the Lipschitz properties of $\tilde{\mathcal{O}}_{1}^{(t)},\tilde{\mathcal{O}}_{2}^{(t)}$ and the definitions of $\mathbf{w}_{1}^{(t)},\mathbf{w}_{2}^{(t)}$, the terms factoring $\beta_{1}, \beta_{2}$ are both non-negative. We thus have, at each time step $t$:
\begin{align}
    (\mathbf{h}^{(t+1)}-\mathbf{h}^{(t)})^{\top}\mathbf{P}(\mathbf{h}^{(t+1)}-\mathbf{h}^{(t)})\leqslant \tau_{(t)}(\mathbf{h}^{(t)}-\mathbf{h}^{(t-1)})^{\top}\mathbf{P}(\mathbf{h}^{(t)}-\mathbf{h}^{(t-1)})
\end{align}
Letting $\tau^{*} = \sup_{t} \tau_{(t)}$, an immediate induction concludes the proof.
\subsection{Bounds on \texorpdfstring{$\hat{Q}_{1x}^{(t+1)}, \hat{Q}_{1z}^{(t)}, \hat{Q}_{2x}^{(t)}, \hat{Q}_{2z}^{(t+1)}$}{Bounds on Q-parameters}}
\label{var_append}
We remind that, since the functions $f$ and $g$ are separable, their Hessians are diagonal matrices. For any time index $t$, the following bounds hold: \\
\quad \\
\noindent
$\hat{Q}_{2x}^{(t)}:$ \\
\begin{align}
 &\hat{Q}_{2x}^{(t)} =1/\chi_{1x}^{(t)} - \hat{Q}_{1x}^{(t)} \qquad \mbox{where} \qquad  \chi_{1x}^{(t)} = \left\langle \partial_{\mathbf{h}_{1x}^{(t)}} g_{1x} (...)\right\rangle/\hat{Q}_{1x}^{(t)}, \\
 &\mbox{then} \qquad \frac{1}{\hat{Q}_{2x}^{(t)}+\hat{Q}_{1x}^{(t)}} = \frac{1}{N}\left(\mbox{Tr}\left[(\hat{Q}^{(t)}_{1x}Id+\mathcal{H}_{f}(\mbox{prox}))^{-1}\right]\right), \\
 &\hat{Q}_{1x}^{(t)}+\lambda_{min}(\mathcal{H}_{f})\leqslant \hat{Q}_{1x}^{(t)}+\hat{Q}_{2x}^{(t)}\leqslant \hat{Q}_{1x}^{(t)}+\lambda_{max}(\mathcal{H}_{f}).
\end{align}
$\hat{Q}_{2z}^{(t+1)}:$
\begin{align}
    &\hat{Q}_{2z}^{(t+1)} =1/\chi_{1z}^{(t)} - \hat{Q}_{1z}^{(t)} \qquad \mbox{where} \qquad \chi_{1z}^{(t)} =  \left\langle \partial_{\mathbf{h}_{1z}^{(t)}} g_{1z} (...)\right\rangle/\hat{Q}_{1z}^{(t)}, \\
    &\mbox{then} \qquad \frac{1}{\hat{Q}_{2z}^{(t+1)}+\hat{Q}_{1z}^{(t)}} = \frac{1}{M}\left(\mbox{Tr}\left[(\hat{Q}_{1z}^{(t)}Id+\mathcal{H}_{g}(\mbox{prox}))^{-1}\right]\right), \\
     &\hat{Q}_{1z}^{(t)}+\lambda_{min}(\mathcal{H}_{g})\leqslant \hat{Q}_{1z}^{(t)}+\hat{Q}_{2z}^{(t+1)}\leqslant \hat{Q}_{1z}^{(t)}+\lambda_{max}(\mathcal{H}_{g}).
\end{align} 
$\hat{Q}_{1z}^{(t)}:$
\begin{align}
    &\hat{Q}_{1z}^{(t)} = 1/\chi_{2z}^{(t)} -\hat{Q}_{2z}^{(t)} \qquad  \chi_{2z}^{(t)} = \left\langle \partial_{\mathbf{h}_{2z}^{(t)}} g_{2z}(...) \right\rangle/ \hat{Q}_{2z}^{(t)} \\
    &\mbox{then} \qquad \frac{1}{\hat{Q}_{1z}^{(t)}+\hat{Q}_{2z}^{(t)}} = \frac{1}{M}\mbox{Tr}\left[\mathbf{F}\mathbf{F}^{\top}\left(\hat{Q}_{2z}^{(t)}\mathbf{F}\mathbf{F}^{\top}+\hat{Q}_{2x}^{(t)}Id\right)^{-1}\right]
\end{align}
The matrices on the r.h.s. of the previous equation are all diagonalisable in the same basis. Then each eigenvalue has the form 
\begin{equation}
   \frac{\lambda_{min}(\mathbf{F}\mathbf{F}^{\top})}{\hat{Q}_{2z}^{(t)}\lambda_{min}(\mathbf{F}\mathbf{F}^{\top})+\hat{Q}_{2x}^{(t)}} \leqslant \frac{\lambda_{k}(\mathbf{F}\mathbf{F}^{\top})}{\hat{Q}_{2z}^{(t)}\lambda_{k}(\mathbf{F}\mathbf{F}^{\top})+\hat{Q}_{2x}^{(t)}} \leqslant \frac{\lambda_{max}(\mathbf{F}\mathbf{F}^{\top})}{\hat{Q}_{2z}^{(t)}\lambda_{max}(\mathbf{F}\mathbf{F}^{\top})+\hat{Q}_{2x}^{(t)}},
\end{equation}
which leads to the bound
\begin{equation}
    \hat{Q}_{2z}^{(t)}+\frac{\hat{Q}_{2x}^{(t)}}{\lambda_{max}(\mathbf{F}\mathbf{F}^{\top})} \leqslant \hat{Q}_{1z}^{(t)}+\hat{Q}_{2z}^{(t)}\leqslant \hat{Q}_{2z}^{(t)}+\frac{\hat{Q}_{2x}^{(t)}}{\lambda_{min}(\mathbf{F}\mathbf{F}^{\top})}.
\end{equation}
$\hat{Q}_{1x}^{(t+1)}:$
\begin{align}
&\hat{Q}_{1x}^{(t+1)} = 1/\chi_{2x}^{(t+1)} -\hat{Q}_{2x}^{(t)} \qquad \chi_{2x}^{(t+1)} = \left\langle \partial_{\mathbf{h}_{2x}^{(t)}} g_{2x}(...) \right\rangle/\hat{Q}_{2x}^{(t)}, \\
&\mbox{then} \qquad \frac{1}{\hat{Q}_{1x}^{(t+1)}+\hat{Q}_{2x}^{(t)}} = \frac{1}{N}\mbox{Tr}\left[\left(\hat{Q}_{2z}^{(t+1)}\mathbf{F}^{\top}\mathbf{F}+\hat{Q}_{2x}^{(t)}Id\right)^{-1}\right],
\end{align}
which leads to 
\begin{equation}
   \hat{Q}_{2x}^{(t)}+\lambda_{min}(\mathbf{F}^{\top}\mathbf{F})\hat{Q}_{2z}^{(t+1)} \leqslant \hat{Q}_{1x}^{(t+1)}+\hat{Q}_{2x}^{(t)} \leqslant \hat{Q}_{2x}^{(t)}+\lambda_{max}(\mathbf{F}^{\top}\mathbf{F})\hat{Q}_{2z}^{(t+1)}.
\end{equation}
\subsection{Operator norms and Lipschitz constants}
\label{operator-bounds}
\subsubsection{Operator norms of matrices \texorpdfstring{$\mathbf{W_1}^{(t)}, \mathbf{W_2}^{(t)}, \mathbf{W_3}^{(t)}, \mathbf{W_4}^{(t)}$}{operator-norms}} \quad \\
The norms of the linear operators $\mathbf{W_1}^{(t)}, \mathbf{W_2}^{(t)}, \mathbf{W_3}^{(t)}, \mathbf{W_4}^{(t)}$ can be computed or bounded with respect to the singular values of the matrix $\mathbf{F}$. The derivations are straightforward and do not require any specific mathematical result. Denoting $\norm{\mathbf{W}}$ the operator norm of a given matrix $\mathbf{W}$, we have the following:
\begin{align}
    \norm{\mathbf{W_1}^{(t)}}&=\dfrac{\hat{Q}_{2x}^{(t)}}{\hat{Q}_{1x}^{(t+1)}}\max \bigg( \dfrac{\vert \hat{Q}_{1x}^{(t+1)} - \hat{Q}_{2z}^{(t+1)} \lambda_{min}(\mathbf{F}^T \mathbf{F}) \vert}{\hat{Q}_{2x}^{(t)} + \hat{Q}_{2z}^{(t+1)} \lambda_{min}(\mathbf{F}^T \mathbf{F})},\\
    &\hspace{2.5cm}\dfrac{\vert \hat{Q}_{1x}^{(t+1)} - \hat{Q}_{2z}^{(t+1)} \lambda_{max}(\mathbf{F}^T \mathbf{F}) \vert}{\hat{Q}_{2x}^{(t)} + \hat{Q}_{2z}^{(t+1)} \lambda_{max}(\mathbf{F}^T \mathbf{F})} \bigg) \\
    \norm{\mathbf{W_2}^{(t)}} &= \dfrac{\hat{Q}_{2z}^{(t+1)}}{\chi_{2x}^{(t+1)} \hat{Q}_{1x}^{(t+1)}} \dfrac{\sqrt{\lambda_{max}(\mathbf{F}^T \mathbf{F})}}{\hat{Q}_{2x}^{(t)} + \hat{Q}_{2z}^{(t+1)} \lambda_{min}(\mathbf{F}^T \mathbf{F})} \\
    \norm{\mathbf{W_3}^{(t)}}&=\dfrac{\hat{Q}_{2z}^{(t)}}{\hat{Q}_{1z}^{(t)}} \max\bigg( \dfrac{\vert \hat{Q}_{2x}^{(t)} - \hat{Q}_{1z}^{(t)} \lambda_{min}(\mathbf{F} \mathbf{F}^T) \vert}{\hat{Q}_{2x}^{(t)} + \hat{Q}_{2z}^{(t)} \lambda_{min}(\mathbf{F} \mathbf{F}^T)}, \\
    &\hspace{2.5cm}\dfrac{\vert \hat{Q}_{2x}^{(t)} - \hat{Q}_{1z}^{(t)} \lambda_{max}(\mathbf{F} \mathbf{F}^T) \vert}{\hat{Q}_{2x}^{(t)} + \hat{Q}_{2z}^{(t)} \lambda_{max}(\mathbf{F} \mathbf{F}^T)} \bigg)\\
    \norm{\mathbf{W_4}^{(t)}} &= \dfrac{\hat{Q}_{2x}^{(t)}}{\chi_{2z}^{(t)} \hat{Q}_{1z}^{(t)}} \dfrac{\sqrt{\lambda_{max}(\mathbf{F}^T \mathbf{F})}}{\hat{Q}_{2x}^{(t)} + \hat{Q}_{2z}^{(t)} \lambda_{min}(\mathbf{F}^T \mathbf{F})}.
\end{align}
\subsubsection{Lispchitz constants of \texorpdfstring{$\tilde{\mathcal{O}}_{1}^{(t)},\tilde{\mathcal{O}}_{2}^{(t)}$}{Lipschitz-constants}} \quad \\
\label{subsec:app_Lip_const}
We now derive upper bounds of the Lipschitz constants of $\tilde{\mathcal{O}}_{1}^{(t)},\tilde{\mathcal{O}}_{2}^{(t)}$ using the convex analysis reminder in appendix \ref{appendix : prox_prop}. We give detail for $\tilde{\mathcal{O}}_{1}^{(t)}$, the derivation is identical for $\tilde{\mathcal{O}}_{2}^{(t)}$. Let $(\sigma_{1},\beta_{1}) \in \mathbb{R}_{+}^{*2}$ be the strong-convexity and smoothness constants of $f$, if they exist. If $f$ has no strong convexity constant, we set $\sigma_1 = 0$, and if it holds no smoothness assumption, we set $\beta_1 = + \infty$. Note that, from the upper and lower bounds obtained in appendix \ref{var_append}, we have $\sigma_{1} \leqslant \hat{Q}_{2x}^{(t)}\leqslant \beta_{1}$.\\
\paragraph{Case 1: $0<\sigma_{1}<\beta_{1}$}
Proposition \ref{proposition : gisel_idh} gives the following expression:
\begin{align}
    \mbox{Prox}_{\frac{1}{\hat{Q}_{1x}^{(t)}}f} &= \frac{1}{2}\left(\frac{1}{1+\sigma_{1}/\hat{Q}_{1x}^{(t)}}+\frac{1}{1+\beta_{1}/\hat{Q}_{1x}^{(t)}}\right)\rm{Id}\\&+\frac{1}{2}\left(\frac{1}{1+\sigma_{1}/\hat{Q}_{1x}^{(t)}}-\frac{1}{1+\beta_{1}/\hat{Q}_{1x}^{(t)}}\right)S_{1}
\end{align}
where $S_{1}$ is a non-expansive operator. Replacing in the expression of $\tilde{\mathcal{O}}_{1}$ leads to:
\begin{align}
\label{eq_137}
    \tilde{\mathcal{O}}_{1}^{(t)} &=\frac{\hat{Q}_{1x}^{(t)}}{\hat{Q}_{2x}^{(t)}}\left( \bigg(\frac{1}{2\chi_{1x}^{(t)}}\left(\frac{1}{\hat{Q}_{1x}^{(t)}+\sigma_{1}}+\frac{1}{\hat{Q}_{1x}^{(t)}+\beta_{1}}\right)-1\right)\rm{Id}\\
    &+\frac{1}{2\chi_{1x}^{(t)}}\left(\frac{1}{\hat{Q}_{1x}^{(t)}+\sigma_{1}}-\frac{1}{\hat{Q}_{1x}^{(t)}+\beta_{1}}\right)S_{1}\bigg)
\end{align}
which, knowing that $\hat{Q}_{1x}^{(t)}+\hat{Q}_{2x}^{(t)} = \frac{1}{\chi_{1x}^{(t)}}$, and separating the case where the first term of the sum in Eq.\eqref{eq_137} is negative or positive, $\tilde{\mathcal{O}}_{1}$ has Lipschitz constant:
\begin{equation}
\label{Lip_1}
    \omega_{1}^{(t)} = \frac{\hat{Q}_{1x}^{(t)}}{\hat{Q}_{2x}^{(t)}}\max \left(\frac{\hat{Q}_{2x}^{(t)}-\sigma_{1}}{\hat{Q}_{1x}^{(t)}+\sigma_{1}},\frac{\beta_{1}-\hat{Q}_{2x}^{(t)}}{\hat{Q}_{1x}^{(t)}+\beta_{1}}\right).
\end{equation} 
\paragraph{Case 2: $0<\sigma_{1}=\beta_{1}$}
In this case, we have from Proposition \ref{proposition : gisel_idh}:
\begin{equation}
    \norm{\mbox{Prox}_{\frac{1}{\hat{Q}_{1x}^{(t)}}f}(x)-\mbox{Prox}_{\frac{1}{\hat{Q}_{1x}^{(t)}}f}(y)}_{2}^{2} = \left(\frac{1}{1+\sigma_{1}/\hat{Q}_{1x}^{(t)}}\right)^{2} \norm{x-y}_{2}^{2}
\end{equation}
which, with the firm non-expansiveness of the proximal operator gives, for any $x,y \in \mathbb{R}$:
\begin{figure*}[!t]
\begin{align}
    \norm{\tilde{\mathcal{O}}_{1}^{(t)}(x)-\tilde{\mathcal{O}}^{(t)}_{1}(y)}_{2}^{2} &=\left(\frac{\hat{Q}_{1x}^{(t)}}{\hat{Q}_{2x}^{(t)}}\right)^{2}\bigg( \frac{1}{(\hat{Q}_{1x}^{(t)})^{2}(\chi_{1x}^{(t)})^{2}} \norm{\mbox{Prox}_{\frac{1}{\hat{Q}_{1x}^{(t)}}f}(x)-\mbox{Prox}_{\frac{1}{\hat{Q}_{1x}^{(t)}}f}(y)}_{2}^{2}\\
    &\hspace{3cm}-2\frac{1}{\hat{Q}_{1x}^{(t)}\chi_{1x}^{(t)}}\left\langle x-y,\mbox{Prox}_{\frac{1}{\hat{Q}_{1x}^{(t)}}f}(x)-\mbox{Prox}_{\frac{1}{\hat{Q}_{1x}^{(t)}}f}(y)\right\rangle+\norm{x-y}_{2}^{2} \bigg) \\
    & \label{equation : stuff} \leqslant \left(\frac{\hat{Q}_{1x}^{(t)}}{\hat{Q}_{2x}^{(t)}}\right)^{2}\bigg(\hspace{-0.1cm}\left( \frac{1}{(\hat{Q}_{1x}^{(t)})^{2}(\chi_{1x}^{(t)})^{2}}-2\frac{1}{\hat{Q}_{1x}^{(t)}\chi_{1x}^{(t)}}\right) \norm{\mbox{Prox}_{\frac{1}{\hat{Q}_{1x}}f}(x)-\mbox{Prox}_{\frac{1}{\hat{Q}_{1x}^{(t)}}f}(y)}_{2}^{2} +\norm{x-y}_{2}^{2}\bigg) \\
    &=\left(\frac{\hat{Q}_{1x}^{(t)}}{\hat{Q}_{2x}^{(t)}}\right)^{2} \left(\left( \frac{1}{(\hat{Q}_{1x}^{(t)})^{2}(\chi_{1x}^{(t)})^{2}}-2\frac{1}{\hat{Q}_{1x}^{(t)}\chi_{1x}^{(t)}}\right)\left(\frac{1}{1+\sigma_{1}/\hat{Q}_{1x}^{(t)}}\right)^{2}+1\right)\norm{x-y}_{2}^{2} \\
    &= \left(\frac{\hat{Q}_{1x}^{(t)}}{\hat{Q}_{2x}^{(t)}}\right)^{2}\left(\frac{(\hat{Q}_{2x}^{(t)})^{2}-(\hat{Q}_{1x}^{(t)})^{2}}{(\hat{Q}_{1x}^{(t)}+\sigma_{1})^{2}}+1\right)\norm{x-y}_{2}^{2}.
\end{align}
\hrulefill
\end{figure*}
The upper bound on the Lipschitz constant is therefore:
\begin{equation}
\label{Lip_2}
    \omega_{1} = \frac{\hat{Q}_{1x}^{(t)}}{\hat{Q}_{2x}^{(t)}}\sqrt{1+\frac{((\hat{Q}_{2x}^{(t)})^{2}-(\hat{Q}_{1x}^{(t)})^{2})}{(\hat{Q}_{1x}^{(t)}+\sigma_{1})^{2}}}.
\end{equation}

\paragraph{Case 3: no strong convexity or smoothness assumption}
This setting is not necessary for our proof, because we only handle penalty functions which have a strictly positive strong convexity constant, by adding a ridge term. However, we list it for completeness. In this case, the only information we have is the firm nonexpansiveness of the proximal operator, which leads us to the same derivation as the previous one up to (\ref{equation : stuff}), where the first term in the sum can be positive or negative. This yields the Lipschitz constant:
\begin{equation}
    \omega_{1}^{(t)} = \frac{\hat{Q}_{1x}^{(t)}}{\hat{Q}_{2x}^{(t)}}\max \left(1,\frac{\hat{Q}_{2x}^{(t)}}{\hat{Q}_{1x}^{(t)}}\right).
\end{equation}

\paragraph{Recovering (\ref{lip_const})}\quad In our proof, we make no assumption on the strong-convexity or smoothness of the function, but adding the ridge penalties $\lambda_{2},\tilde{\lambda}_{2}$ brings us for both $\tilde{\mathcal{O}}_{1}^{(t)}$ and $\tilde{\mathcal{O}}_{2}^{(t)}$ to either the first of the second case above. It is straightforward to see that the Lipschitz constant (\ref{Lip_2}) is an upper bound of (\ref{Lip_1}). We thus use (\ref{Lip_2}) for generality, and recover the expressions (\ref{lip_const}) shown in the main body of the paper.
\begin{align}
    \omega_{1}^{(t)}&= \dfrac{\hat{Q}_{1x}^{(t)}}{\hat{Q}_{2x}^{(t)}} \sqrt{1+\dfrac{(\hat{Q}_{2x}^{(t)})^{2}-(\hat{Q}_{1x}^{(t)})^{2}}{(\hat{Q}_{1x}^{(t)} + \lambda_2)^2}} \\
    \omega_{2}^{(t)}&= \dfrac{\hat{Q}_{1z}^{(t)}}{\hat{Q}_{2z}^{(t)}} \sqrt{1+\dfrac{(\hat{Q}_{2z}^{(t)})^{2}-(\hat{Q}_{1z}^{(t)})^{2}}{(\hat{Q}_{1z}^{(t)} + \tilde{\lambda}_2)^2}}.
\end{align}
\subsection{Dynamical system convergence analysis}
We are now ready to prove Lemma \ref{conv_lemma}.\\
\quad \\
We will use the bounds derived above to prove the convergence lemma. Since we have proved the required bounds at any time step, we drop the time indices in the remainder of this proof for simplicity. The choice of additional regularization is $\lambda_{2}$ arbitrarily large, and $\tilde{\lambda}_{2}$ fixed but finite and non-zero. $\hat{Q}_{2x}, \hat{Q}_{1z}$ can thus be made arbitrarily large, and $\hat{Q}_{2z}, \hat{Q}_{1x}$ remain finite. We write the corresponding linear matrix inequality (\ref{the_LMI}) and expand the constraint term. Some algebra shows that:
\begin{align}
    \mathbf{C}_{1}^{T}\mathbf{M}_{1}\mathbf{C}_{1} &= \begin{bmatrix}\mathbf{0}_{M\times M} & 0_{M\times N} \\ \mathbf{0}_{
    N\times M} & \omega_{1}^{2}I_{N\times N}\end{bmatrix}\\
    \mathbf{C}_{2}^{T}\mathbf{M}_{2}\mathbf{C}_{2} &= \begin{bmatrix}\omega_{2}^{2}\mathbf{W}_{3}^{T}\mathbf{W}_{3} & \mathbf{0}_{M\times N} \\ \mathbf{0}_{N\times M} & \mathbf{0}_{N\times N}\end{bmatrix} \\
    \mathbf{C}_{1}^{T}\mathbf{M}_{1}\mathbf{D}_{1} &= \mathbf{0}_{(M+N)\times(M+N)} \\
    \mathbf{D}_{1}^{T}\mathbf{M}_{1}\mathbf{C}_{1}&= \mathbf{0}_{(M+N)\times(M+N)} \\
    \mathbf{C}_{2}^{T}\mathbf{M}_{2}\mathbf{D}_{2}&=\begin{bmatrix} \mathbf{0}_{M \times M} & \omega_{2}^{2}\mathbf{W}_{3}^{T}\mathbf{W}_{4} \\ \mathbf{0}_{N\times M} & \mathbf{0}_{N\times N}\end{bmatrix} \\
    \mathbf{D}_{2}^{T}\mathbf{M}_{2}\mathbf{C}_{2}&=\begin{bmatrix} \mathbf{0}_{M \times M} & \mathbf{0}_{M\times N} \\ \omega_{2}^{2}\mathbf{W}_{4}^{T}\mathbf{W}_{3} & \mathbf{0}_{N\times N}\end{bmatrix} \\
    \mathbf{D}_{1}^{T}\mathbf{M}_{1}\mathbf{D}_{1}&=\begin{bmatrix} \mathbf{0}_{M \times M} & \mathbf{0}_{M\times N} \\ \mathbf{0}_{N\times M} & -\mathbf{I}_{N\times N}\end{bmatrix} \\
    \mathbf{D}_{2}^{T}\mathbf{M}_{2}\mathbf{D}_{2}&=\begin{bmatrix}-\mathbf{I}_{M \times M} & \mathbf{0}_{M\times N} \\ \mathbf{0}_{N\times M} & \omega_{2}^{2}\mathbf{W}_{4}^{T}\mathbf{W}_{4}\end{bmatrix}
\end{align}
where all the matrices constituting the blocks have been defined in section \ref{section:conv_sec}. This gives the following form for the constraint matrix:
\begin{equation}
    \begin{bmatrix}\mathbf{H}_{1}&\mathbf{H}_{2} \\ \mathbf{H}_{2}^{T}&\mathbf{H}_{3}\end{bmatrix}
\end{equation}
where 
\begin{align}
    \mathbf{H}_{1} &= \begin{bmatrix}\beta_{1}\omega_{2}^{2}\mathbf{W}_{3}^{T}\mathbf{W}_{3} & \mathbf{0}_{M\times N} \\ \mathbf{0}_{N \times M} & \beta_{0}\omega_{1}^{2} \mathbf{I}_{N \times N}\end{bmatrix} \\
    \mathbf{H}_{2} &= \begin{bmatrix} \mathbf{0}_{M\times M} & \beta_{1}\omega_{2}^{2}\mathbf{W}_{3}^{T}\mathbf{W}_{4} \\
    \mathbf{0}_{N \times M} & \mathbf{0}_{N \times N} \end{bmatrix} \\
    \mathbf{H}_{3} &= \begin{bmatrix}-\beta_{1} \mathbf{I}_{M \times M} & \mathbf{0}_{M \times N} \\ \mathbf{0}_{N \times M} & -\beta_{0} \mathbf{I}_{N \times N}+ \beta_{1}\omega_{2}^{2}\mathbf{W}_{4}^{T}\mathbf{W}_{4} \end{bmatrix}
\end{align}
thus the LMI (\ref{the_LMI}) becomes:
 \begin{equation}
     0 \succeq \begin{bmatrix} -\tau^{2}\mathbf{P}+\mathbf{H}_{1} & \mathbf{H}_{2} \\\mathbf{H}_{2}^{T} & \mathbf{B}^{T}\mathbf{P}\mathbf{B}+\mathbf{H}_{3}
     \end{bmatrix}.
 \end{equation}
 We take $\mathbf{P}$ as block diagonal:
 \begin{equation}
     \mathbf{P} = \begin{bmatrix}\mathbf{P}_{1} & \mathbf{0}_{M \times N} \\
     \mathbf{0}_{N \times M} & \mathbf{P}_{2} \end{bmatrix}
 \end{equation}
 where $\mathbf{P}_{1} \in \mathbb{R}^{M \times M}$ and $\mathbf{P}_{2} \in \mathbb{R}^{N \times N}$ are positive definite (no zero eigenvalues) and diagonalizable in the same basis as $\mathbf{F}^{T}\mathbf{F}$, which is also the eigenbasis of $\mathbf{W}_{1}, \mathbf{W}_{3}, \mathbf{W}_{2}^{T}\mathbf{W}_{2}, \mathbf{W}_{4}^{T}\mathbf{W}_{4}$. We then have:
 \begin{equation}
     \mathbf{B}^{T}\mathbf{P}\mathbf{B} = \begin{bmatrix} \mathbf{P}_{1}+\mathbf{W}_{2}^{T}\mathbf{P}_{2}\mathbf{W}_{2} & \mathbf{W}_{2}^{T}\mathbf{P}_{2}\mathbf{W}_{1} \\
     \mathbf{W}_{1}^{T}\mathbf{P}_{2}\mathbf{W}_{2} & \mathbf{W}_{1}^{T}\mathbf{P}_{2}\mathbf{W}_{1} \end{bmatrix}.
 \end{equation}
We are then trying to find the conditions for the following problem to be feasible with $0<\tau<1$:
\begin{equation}
\label{refom_LMI}
    \begin{bmatrix} \tau^{2}\mathbf{P}-\mathbf{H}_{1} & -\mathbf{H}_{2} \\-\mathbf{H}_{2}^{T} & -(\mathbf{B}^{T}\mathbf{P}\mathbf{B}+\mathbf{H}_{3})
     \end{bmatrix} \succeq 0
\end{equation}
Schur's lemma then gives that the strict version of (\ref{refom_LMI}), which we will consider, is equivalent \cite{horn2012matrix} to:
\begin{align}
\label{LMI_split}
    &-(\mathbf{B}^{T}\mathbf{P}\mathbf{B}+\mathbf{H}_{3}) \succ 0 \quad \mbox{and}\\ &\tau^{2}\mathbf{P}-\mathbf{H}_{1}+\mathbf{H}_{2}(\mathbf{B}^{T}\mathbf{P}\mathbf{B}+\mathbf{H}_{3})^{-1}\mathbf{H}_{2}^{T} \succ 0 
\end{align}
We start with $-(\mathbf{B}^{T}\mathbf{P}\mathbf{B}+\mathbf{H}_{3})$.  \\
\subsubsection{Conditions for \texorpdfstring{$-(\mathbf{B}^{T}\mathbf{P}\mathbf{B}+\mathbf{H}_{3}) \succ 0$}{condition1}}\quad \\
\quad \\
Expanding $-(\mathbf{B}^{T}\mathbf{P}\mathbf{B}+\mathbf{H}_{3}) \succ 0$ and applying Schur's lemma again gives the equivalent problem:
\begin{align}
    &\beta_{1} \mathbf{I}_{N \times N}- \beta_{2}\omega_{2}^{2}\mathbf{W}_{4}^{T}\mathbf{W}_{4}-\mathbf{W}_{1}^{T}\mathbf{P}_{2}\mathbf{W}_{1}\succ 0 \label{cond11} \quad \mbox{and} \\ \notag
    &\beta_{2} \mathbf{I}_{M \times M}-\mathbf{P}_{1}-\mathbf{W}_{2}^{T}\mathbf{P}_{2}\mathbf{W}_{2}\\ 
    &\hspace{2cm}-\mathbf{W}_{2}^{T}\mathbf{P}_{2}\mathbf{W}_{1}\mathbf{K}_{1}\mathbf{W}_{1}^{T}\mathbf{P}_{2}\mathbf{W}_{2} \succ 0. \label{cond12}
\end{align}
where $\mathbf{K}_{1}=(\beta_{1} \mathbf{I}_{N \times N}- \beta_{2}\omega_{2}^{2}\mathbf{W}_{4}^{T}\mathbf{W}_{4}-\mathbf{W}_{1}^{T}\mathbf{P}_{2}\mathbf{W}_{1})^{-1}$. We start with (\ref{cond11}). A sufficient condition for it to hold true is:
\begin{equation}
    \beta_{1} > \beta_{2}\omega_{2}^{2}\lambda_{max}(\mathbf{W}_{4}^{T}\mathbf{W}_{4})+\lambda_{max}(\mathbf{P}_{2})\lambda_{max}(\mathbf{W}_{1}^{T}\mathbf{W}_{1}).
\end{equation}
Using the bounds from appendix \ref{operator-bounds}, we have:
\begin{align}
    \lambda_{max}(\mathbf{W}_{1}^{T}\mathbf{W}_{1}) &\leqslant \left(\dfrac{\hat{Q}_{2x}}{\hat{Q}_{1x}}\right)^{2} \max \bigg( ... \notag
    \\&\hspace{-2cm}\dfrac{\vert \hat{Q}_{1x} - \hat{Q}_{2z} \lambda_{min}(\mathbf{F}^T \mathbf{F}) \vert}{\hat{Q}_{2x} + \hat{Q}_{2z} \lambda_{min}(\mathbf{F}^T \mathbf{F})},
    \dfrac{\vert \hat{Q}_{1x} - \hat{Q}_{2z} \lambda_{max}(\mathbf{F}^T \mathbf{F}) \vert}{\hat{Q}_{2x} + \hat{Q}_{2z} \lambda_{max}(\mathbf{F}^T \mathbf{F})} \bigg)^{2} \\
    &\hspace{-2cm}\leqslant \max\bigg(\left(1-\dfrac{\hat{Q}_{2z}}{\hat{Q}_{1x}}\lambda_{min}(\mathbf{F}^{T}\mathbf{F})\right)^{2}, \notag
    \\&\hspace{0.5cm}\left(1-\dfrac{\hat{Q}_{2z}}{\hat{Q}_{1x}}\lambda_{max}(\mathbf{F}^{T}\mathbf{F})\right)^{2}\bigg) = b_{1}
\end{align}
and 
\begin{align}
    \omega_{2}^{2}\lambda_{max}(\mathbf{W}_{4}^{T}\mathbf{W}_{4})\hspace{-0.05cm}&\leqslant \hspace{-0.05cm} \left(\dfrac{\hat{Q}_{1z}}{\hat{Q}_{2z}}\right)^{2}\left(\frac{\hat{Q}_{2x}}{\chi_{2z}\hat{Q}_{1z}}\right)^{2}\times ... \notag\\
    &\hspace{-2.5cm}\left(1+\frac{(\hat{Q}_{2z})^{2}-(\hat{Q}_{1z})^{2}}{(\hat{Q}_{1z}+\tilde{\lambda}_{2})^{2}}\right)\frac{\lambda_{max}(\mathbf{F}^{T}\mathbf{F})}{(\hat{Q}_{2x}+\hat{Q}_{2z}\lambda_{min}(\mathbf{F}^{T}\mathbf{F}))^{2}} \\
    &\hspace{-2.5cm}\leqslant \hat{Q}_{1z}\left(2\tilde{\lambda}_{2}+\frac{\tilde{\lambda}_{2}^{2}}{\hat{Q}_{1z}}+\frac{(\hat{Q}_{2z})^{2}}{\hat{Q}_{1z}}\right)\times... \notag\\
    &\left(\frac{\hat{Q}_{1z}+\hat{Q}_{2z}}{\hat{Q}_{2z}(\hat{Q}_{1z}+\tilde{\lambda}_{2})}\right)^{2}\lambda_{max}(\mathbf{F}^{T}\mathbf{F}). 
\end{align}
For arbitrarily large $\hat{Q}_{1z}$, the quantity $\left(2\tilde{\lambda}_{2}+\frac{\tilde{\lambda}_{2}^{2}}{\hat{Q}_{1z}}+\frac{(\hat{Q}_{2z})^{2}}{\hat{Q}_{1z}}\right)\left(\frac{\hat{Q}_{1z}+\hat{Q}_{2z}}{\hat{Q}_{2z}(\hat{Q}_{1z}+\tilde{\lambda}_{2})}\right)^{2}\lambda_{max}(\mathbf{F}^{T}\mathbf{F})$ is trivially bounded above whatever the value of $\tilde{\lambda}_{2},\hat{Q}_{2z}$. Let $b_{2}$ be such an upper bound independent of $\lambda_{2}, \hat{Q}_{2x}, \hat{Q}_{1z}$. The sufficient condition for (\ref{cond11}) to hold thus becomes:
\begin{equation}
\label{cond1reform}
    \beta_{1} > \beta_{2}\hat{Q}_{1z}b_{2}+\lambda_{max}(\mathbf{P}_{2})b_{1}
\end{equation}
where $b_{1},b_{2}$ are constants independent of $\lambda_{2},\hat{Q}_{2x},\hat{Q}_{1z}$. \\
\newline
We now turn to (\ref{cond12}). A sufficient condition for it to hold is:
\begin{align}
\label{eq_cond12}
    \beta_{2}>\lambda_{max}(\mathbf{P}_{1})&+\lambda_{max}(\mathbf{W}_{2}^{T}\mathbf{W}_{2})\lambda_{max}(\mathbf{P}_{2}) \notag\\
    &\hspace{-2.5cm}+\frac{(\lambda_{max}(\mathbf{P}_{2}))^{2}\lambda_{max}(\mathbf{W}_{2}^{T}\mathbf{W}_{2})\lambda_{max}(\mathbf{W}_{1}^{T}\mathbf{W}_{1})}{\beta_{1}-\beta_{2}\omega_{2}^{2}\lambda_{max}(\mathbf{W}_{4}^{T}\mathbf{W}_{4})-\lambda_{max}(\mathbf{P}_{2})\lambda_{max}(\mathbf{W}_{1}^{T}\mathbf{W}_{1})} 
\end{align}
Note that condition (\ref{cond11}) ensures that the denominator in (\ref{eq_cond12}) is non-zero. We then have:
\begin{align}
    \lambda_{max}(\mathbf{W}_{2}^{T}\mathbf{W}_{2}) &\leqslant \left(\frac{\hat{Q}_{2z}}{\chi_{2x}\hat{Q}_{1x}}\right)^{2}\frac{\lambda_{max}(\mathbf{F}^{T}\mathbf{F})}{(\hat{Q}_{2x}+\hat{Q}_{2z}\lambda_{min}(\mathbf{F}^{T}\mathbf{F}))^{2}}\\
    &\leqslant \left(\frac{\hat{Q}_{2z}(1+\frac{\hat{Q}_{1x}}{\hat{Q}_{2x}})}{\hat{Q}_{1x}}\right)^{2}\lambda_{max}(\mathbf{F}^{T}\mathbf{F}) 
\end{align}
This quantity can be bounded above by a constant independent of $\lambda_{2},\hat{Q}_{2x},\hat{Q}_{1z}$ for arbitrarily large $\hat{Q}_{2x}$. Let $b_{3}$ be such a constant . Then a sufficient condition for condition (\ref{cond12}) to hold is:
\begin{align}
\label{cond12reform}
    \beta_{2}>\lambda_{max}(\mathbf{P}_{1})+b_{3}\lambda_{max}(\mathbf{P}_{2})\notag \\
    &\hspace{-2.5cm}+\frac{b_{1}b_{3}(\lambda_{max}(\mathbf{P}_{2}))^{2}}{\beta_{1}-\beta_{2}\hat{Q}_{1z}b_{2}-\lambda_{max}(\mathbf{P}_{2})b_{1}}
\end{align}
we see that $\beta_{1}$ must scale linearly with $\hat{Q}_{1z}$ which is one of the parameters that is made arbitrarily large. Then $\beta_{1}$ also needs to become arbitrarily large for the conditions to hold. We choose $\beta_{1} = 2\beta_{2}\hat{Q}_{1z}b_{2}+\lambda_{max}(\mathbf{P}_{2})b_{1}$ for the rest of the proof. Condition (\ref{cond1reform}) is then verified, and $\beta_{2}$ needs to be chosen according to condition (\ref{cond12reform}), which becomes:
\begin{equation}
    \beta_{2}>\lambda_{max}(\mathbf{P}_{1})+b_{3}\lambda_{max}(\mathbf{P}_{2})+\frac{b_{1}b_{3}\lambda_{max}^{2}(\mathbf{P}_{2})}{\beta_{2}\hat{Q}_{1z}b_{2}}
\end{equation}
This has a bounded solution for large values of $\hat{Q}_{1z}$.
We now turn to the second part of (\ref{LMI_split}).\\
\subsubsection{Conditions for \texorpdfstring{$\tau^{2}\mathbf{P}-\mathbf{H}_{1}+\mathbf{H}_{2}(\mathbf{B}^{T}\mathbf{P}\mathbf{B}+\mathbf{H}_{3})^{-1}\mathbf{H}_{2}^{T} \succ 0$}{condition2}} \quad \\ 
\quad \\
We need to study the term $-\mathbf{H}_{2}(\mathbf{B}^{T}\mathbf{P}\mathbf{B}+\mathbf{H}_{3})^{-1}\mathbf{H}_{2}^{T}$ (we study it with the $-$ sign since the middle matrix is negative definite from conditions (\ref{cond11},\ref{cond12}) which are now verified). As we will see, because of the form of $\mathbf{H}_{2}$, we don't need to explicitly compute the whole inverse. Let $\mathbf{Z}=-(\mathbf{B}^{T}\mathbf{P}\mathbf{B}+\mathbf{H}_{3})^{-1} = \begin{bmatrix}\mathbf{Z}_{1} & \mathbf{Z}_{2} \\ \mathbf{Z}_{2}^{T} & \mathbf{Z}_{3}\end{bmatrix}$ ($\mathbf{Z}$ has the same block dimensions as $(\mathbf{B}^{T}\mathbf{P}\mathbf{B}+\mathbf{H}_{3})$).
We then have:
\begin{align}
    -\mathbf{H}_{2}(\mathbf{B}^{T}\mathbf{P}\mathbf{B}+\mathbf{H}_{3})^{-1}\mathbf{H}_{2}^{T} &= \mathbf{H}_{2}\mathbf{Z}\mathbf{H}_{2}^{T} \\
    &\hspace{-2cm}=\begin{bmatrix}\beta_{2}^{2}\omega_{2}^{4}\mathbf{W}_{3}^{T}\mathbf{W}_{4}\mathbf{Z}_{3}\mathbf{W}_{4}^{T}\mathbf{W}_{3} & \mathbf{0}_{M\times N} \\ \mathbf{0}_{N \times M}& \mathbf{0}_{N \times N}\end{bmatrix}.
\end{align}
We thus only need to characterize the lower right block of $\mathbf{Z}$. It is easy to see that conditions (\ref{cond11}) and (\ref{cond12}) also enforce that both the Schur complements associated with the upper left and lower right blocks of $-(\mathbf{B}^{T}\mathbf{P}\mathbf{B}+\mathbf{H}_{3})$ are invertible, thus giving the following form for $\mathbf{Z}_{3}$ using the block matrix inversion lemma \cite{horn2012matrix}:
\begin{align}
    \mathbf{Z}_{3} = (\beta_{1} \mathbf{I}_{N}&-\beta_{2}\omega_{2}^{2}\mathbf{W}_{4}^{T}\mathbf{W}_{4} \notag\\
    &\hspace{-1cm}-\mathbf{W}_{1}^{T}\mathbf{P}_{2}\mathbf{W}_{1}-\mathbf{W}_{1}^{T}\mathbf{P}_{2}\mathbf{W}_{2}\mathbf{K}_{2}\mathbf{W}_{2}^{T}\mathbf{P}_{2}\mathbf{W}_{1})^{-1}.
\end{align}
where $\mathbf{K}_{2}=(\beta_{1} \mathbf{I}_{M}-\mathbf{P}_{1}-\mathbf{W}_{2}^{T}\mathbf{P}_{2}\mathbf{W}_{2})^{-1}$. We thus have the following upper bound on the largest eigenvalue of $\mathbf{Z}_{3}$:
\begin{equation}
    \lambda_{max}(\mathbf{Z}_{3}) \leqslant \frac{1}{\beta_{1}-\beta_{2}\hat{Q}_{1z}b_{2}-\lambda_{max}(\mathbf{P}_{2})b_{1}-k},
\end{equation}
where $k=\frac{b_{1}b_{3}\lambda_{max}^{2}(\mathbf{P}_{2})}{\beta_{2}-\lambda_{max}(\mathbf{P}_{1})-b_{2}\lambda_{max}(\mathbf{P}_{2})}$. Using the prescription $\beta_{1} = 2\beta_{2}\hat{Q}_{1z}b_{2}+\lambda_{max}(\mathbf{P}_{1})b_{1}$, we get:
\begin{equation}
    \lambda_{max}(\mathbf{Z}_{3}) = \frac{1}{\beta_{1}\hat{Q}_{1z}b_{2}-\frac{b_{1}b_{3}\lambda_{max}^{2}(\mathbf{P}_{2})}{\beta_{1}-\lambda_{max}(\mathbf{P}_{1})-b_{2}\lambda_{max}(\mathbf{P}_{2})}} \leqslant \frac{b_{4}}{\hat{Q}_{1z}}
\end{equation}
where $b_{4}$ is a constant independent of the arbitrarily large parameters $\lambda_{2}, \hat{Q}_{2x}, \hat{Q}_{1z}$. Thus $\lambda_{max}(\mathbf{Z}_{3})$ can be made arbitrarily small by making $\lambda_{2}$ arbitrarily large.\\
\quad \\
We now want to find conditions for $\tau^{2}\mathbf{P}-\mathbf{H}_{1}+\mathbf{H}_{2}(\mathbf{B}^{T}\mathbf{P}\mathbf{B}+\mathbf{H}_{3})^{-1}\mathbf{H}_{2}^{T} \succ 0$ which is equivalent to:
\begin{align}
\label{last_LMI}
    &\tau^{2}\mathbf{P}_{1}-\beta_{2}\omega_{2}^{2}\mathbf{W}_{3}^{T}\mathbf{W}_{3}-\beta_{2}^{2}\omega_{2}^{4}\mathbf{W}_{3}^{T}\mathbf{W}_{4}\mathbf{Z}_{3}\mathbf{W}_{4}^{T}\mathbf{W}_{3} \succeq 0 \notag\\  &\tau^{2}\mathbf{P}_{2}-\beta_{1}\omega_{1}^{2}\mathbf{I}_{N} \succeq 0
\end{align}
We start with the upper matrix inequality, for which a sufficient condition is:
 \begin{align}
     &\tau^{2}\lambda_{min}(\mathbf{P}_{1})-\beta_{2}\omega_{2}^{2}\lambda_{max}(\mathbf{W}_{3}^{T}\mathbf{W}_{3}) \notag\\
     &-\beta_{2}^{2}\omega_{2}^{4}\lambda_{max}(\mathbf{W}_{3}^{T}\mathbf{W}_{3})\lambda_{max}(\mathbf{W}_{4}^{T}\mathbf{W}_{4})\lambda_{max}(\mathbf{Z}_{3})>0
 \end{align}
Using the bounds from appendix \ref{operator-bounds}, we have:
 \begin{align}
     &\omega_{2}^{2}\lambda_{max}(\mathbf{W}_{3}^{T}\mathbf{W}_{3}) \leqslant ... \notag \\ &\left(\dfrac{\hat{Q}_{1z}}{\hat{Q}_{2z}}\right)^{2}\bigg(1+\frac{(\hat{Q}_{2z})^{2}-(\hat{Q}_{1z})^{2}}{(\hat{Q}_{1z}+\tilde{\lambda}_{2})^{2}}\bigg)\lambda_{max}(\mathbf{W}_{3}^{T}\mathbf{W}_{3}) \\
     &\leqslant \frac{2\tilde{\lambda}_{2}\hat{Q}_{1z}+\tilde{\lambda}_{2}^{2}+(\hat{Q}_{2z})^{2}}{(\hat{Q}_{1z}+\tilde{\lambda}_{2})^{2}}\times ... \notag\\ &\max((1-\dfrac{\hat{Q}_{1z}}{\hat{Q}_{2x}}\lambda_{min}(\mathbf{F}^{T}\mathbf{F}))^{2},(1-\dfrac{\hat{Q}_{1z}}{\hat{Q}_{2x}}\lambda_{max}(\mathbf{F}^{T}\mathbf{F}))^{2}) \\
     &\leqslant \frac{1}{\hat{Q}_{1z}}(2\tilde{\lambda}_{2}+\frac{(\tilde{\lambda}_{2}^{2}+(\hat{Q}_{2z})^{2})}{\hat{Q}_{1z}})\times...\notag\\
     &\max((1-\dfrac{\hat{Q}_{1z}}{\hat{Q}_{2x}}\lambda_{min}(\mathbf{F}^{T}\mathbf{F}))^{2},(1-\dfrac{\hat{Q}_{1z}}{\hat{Q}_{2x}}\lambda_{max}(\mathbf{F}^{T}\mathbf{F}))^{2}) 
 \end{align}
 Thus there exists a constant $b_{5}$, independent of $\lambda_{2}, \hat{Q}_{1z}, \hat{Q}_{2x}$ such that, for sufficiently large $\hat{Q}_{1z}$:
 \begin{equation}
     \omega_{2}^{2}\lambda_{max}(\mathbf{W}_{3}^{T}\mathbf{W}_{3}) \leqslant \frac{b_{5}}{\hat{Q}_{1z}}.
 \end{equation}
Remember that we had:
 \begin{equation}
     \omega_{2}^{2}\lambda_{max}(\mathbf{W}_{4}^{T}\mathbf{W}_{4}) \leqslant \hat{Q}_{1z}b_{2},
 \end{equation}
which gives the following sufficient condition for the upper left block in (\ref{last_LMI}):
\begin{equation}
    \tau^{2}\lambda_{min}(\mathbf{P}_{1})-\beta_{2}\frac{b_{5}}{\hat{Q}_{1z}}-\beta_{2}^{2}\frac{b_{2}b_{5}b_{4}}{\hat{Q}_{1z}} >0.
\end{equation}
A sufficient condition for the lower right block in (\ref{last_LMI}) then reads:
\begin{equation}
    \tau^{2}\lambda_{min}(\mathbf{P}_{2})-\beta_{1}\omega_{1}^{2}>0,
\end{equation}
where we have:
\begin{align}
    \beta_{1}\omega_{1}^{2}&=\left(\dfrac{\hat{Q}_{1x}}{\hat{Q}_{2x}}\right)^{2}\left( 1+\dfrac{(\hat{Q}_{2x})^{2}-(\hat{Q}_{1x})^{2}}{(\hat{Q}_{1x} + \lambda_2)^2}\right)\times ... \notag\\
    &\hspace{1.5cm}(2\beta_{1}\hat{Q}_{1z}b_{2}+\lambda_{max}(\mathbf{P}_{2})b_{1}) \\
    &=\frac{1}{\hat{Q}_{2x}}(\hat{Q}_{1x})^{2}\left( 1+\dfrac{(\hat{Q}_{2x})^{2}-(\hat{Q}_{1x})^{2}}{(\hat{Q}_{1x} + \lambda_2)^2}\right)\times... \notag\\
    &\hspace{1.5cm}\left(2\beta_{1}\frac{\hat{Q}_{1z}}{\hat{Q}_{2x}}b_{2}+\lambda_{max}(\mathbf{P}_{2})\frac{b_{1}}{\hat{Q}_{2x}}\right)
\end{align}
We remind the reader that $\hat{Q}_{1z}, \hat{Q}_{2x}$ grow linearly with $\lambda_{2}$. Thus the dominant scaling at large $\lambda_{2}$ is (exchanging $\hat{Q}_{2x}$ with $\hat{Q}_{1z}$ up to a constant): 
\begin{equation}
    \beta_{1}\omega_{1}^{2} \leqslant \frac{b_{6}}{\hat{Q}_{1z}},
\end{equation}
where $b_{6}$ is a constant independent of the arbitrarily large quantities. The final condition becomes:
\begin{align}
    \tau^{2}\lambda_{min}(\mathbf{P}_{1})-\beta_{2}\frac{b_{5}}{\hat{Q}_{1z}}-\beta_{2}^{2}\frac{b_{2}b_{5}b_{4}}{\hat{Q}_{1z}} &>0 \\
    \tau^{2}\lambda_{min}(\mathbf{P}_{2})-\frac{b_{6}}{\hat{Q}_{1z}}&>0
\end{align}
where we want $\tau < 1$. We now choose $\tau^{2} = \tilde{c}/\hat{Q}_{1z}$ with a constant $\tilde{c}$ independent of $\lambda_{2}, \hat{Q}_{1z}, \hat{Q}_{2x}$ that verifies $\tilde{c}>\max\left(\frac{\beta_{2}b_{5}+\beta_{2}^{2}b_{2}b_{5}b_{4}}{\lambda_{min}(\mathbf{P}_{1})},\frac{b_{6}}{\lambda_{min}(\mathbf{P}_{2})}\right)$, such that:
\begin{align}
    \frac{\tilde{c}}{\hat{Q}_{1z}}\lambda_{min}(\mathbf{P}_{1})-\beta_{2}\frac{b_{5}}{\hat{Q}_{1z}}-\beta_{2}^{2}\frac{b_{2}b_{5}b_{4}}{\hat{Q}_{1z}} &>0 \\
    \frac{\tilde{c}}{\hat{Q}_{1z}}\lambda_{min}(\mathbf{P}_{2})-\frac{b_{6}}{\hat{Q}_{1z}}&>0.
\end{align}
Since $\beta_{2}$ is bounded for large values of $\hat{Q}_{1z}$, and the $b_{i}$ and $c$ are constants independent of $\lambda_{2},\hat{Q}_{2x},\hat{Q}_{1z}$, we can then enforce $\tilde{c}<\hat{Q}_{1z}$ using the additional ridge penalty parametrized by $\lambda_{2}$ on the regularization to obtain $\tau<1 $ and a linear convergence rate proportional to $\sqrt{\frac{\tilde{c}}{\lambda_{2}}}$. We see that the eigenvalues of the matrix $\mathbf{P}$ are of little importance as long as they are non-vanishing. We choose $\mathbf{P}$ as the identity. In the statement of Lemma \ref{conv_lemma}, we write $c$ the exact constant which comes linking $\hat{Q}_{1z}$ to $\lambda_{2}$. \\
This proves Lemma \ref{conv_lemma}.
\section{Analytic continuation}
\label{analytic_continuation}
In this section, we prove the validity of the analytic continuation and approximation argument used to prove Theorem \ref{main_th}, under the required set of assumptions 
\ref{main_assum}.
According to Lemma 4, for any $\tilde{\lambda}_{2}>0$ and $\lambda_{2}>\lambda_{2}^{*}$, any scalar pseudo-Lipschitz observable of order 2 $\phi$, we have almost surely
\begin{equation}
\label{to_be_cont}
    \lim_{N \to \infty} \frac{1}{N}\sum_{i=1}^{N}\phi(x_{0,i},\hat{x}_{i}(\lambda_{2})) = \mathbb{E}[\phi(x_{0},\mbox{Prox}_{f/\hat{Q}_{1x}^{(t)}}(H_{x}))]
\end{equation}
where $H_{x}=\frac{\hat{m}_{1x}^*x_{0}+\sqrt{\hat{\chi}_{1x}^*}\xi_{1x}}{\hat{Q}_{1x}}$ is defined in Theorem \ref{main_th}. We would like to show that this equality still holds for any $\lambda_{2}>0$. To do so we will show  that, for a real analytic approximation of problem Eq.\eqref{student}, both sides of Eq.\eqref{to_be_cont} are real analytic in $\lambda_{2}$. We may then use the real analytic continuation theorem, as given in \cite{krantz2002primer} to extend to any $\lambda_{2}>0$. We will treat the case $\lambda_{2}=0$ separately.
In what follows, we will write the dependency in $\lambda_{2}$ of the estimator explicitly, i.e., $\hat{\mathbf{x}} = \hat{\mathbf{x}}(\lambda_{2})$.
\subsection{Real analyticity of the left hand side of Eq.\texorpdfstring{\eqref{to_be_cont}}{real-analytic}} 
We remind a useful characterization of real analytic functions from \cite{krantz2002primer}:
\begin{proposition}[Proposition 1.2.10 from \cite{krantz2002primer}]
\label{analytic_charac}
Let $f\in \mathcal{C}^{\infty}(I)$ for some open interval I. The function f is in fact 
real analytic on I if and only if, for each $\alpha \in I$, there are an open
interval J, with $\alpha \in J \subset I$, and finite constants $C>0$ and $R>0$ such that the derivatives
of f satisfy :
\begin{equation}
\abs{f^{(j)}(\alpha)}\leqslant C\frac{j!}{R^{j}}, \quad \forall \alpha \in J
\end{equation}
\end{proposition}
We also remind the formula for the higher order derivatives of a composition of two infinitely differentiable functions:
\begin{proposition}(Faa di Bruno's formula, \cite{krantz2002primer} Theorem 1.3.2.)
\label{prop_fdb}
Consider two scalar functions $f$ and $g$ defined on an open interval $I \in \mathbb{R}$. Assume that both functions are infinitely differentiable on $I$ and taking value in $I$. Then the derivatives of $h = g \circ f$ are given by
\begin{equation}
    h^{(n)}(t) = \sum \frac{n!}{k_{1}!k_{2}!...k_{n}!}g^{(k)}\left(f(t)\right)\left(\frac{f^{(1)}(t)}{1!}\right)^{k_{1}}\left(\frac{f^{(2)}(t)}{2!}\right)^{k_{2}}...\left(\frac{f^{(n)}(t)}{n!}\right)^{k_{n}}
\end{equation}
where $k = k_{1}+k_{2}+...+k_{n}$ and the sum is taken over all $k_{1},k_{2},...,k_{n}$ for which $k_{1}+2k_{2}+...+nk_{n} = n$.
\end{proposition}
The following lemma establishes bounds on the higher order derivatives of $\hat{\mathbf{x}}(\lambda_{2})$ with respect to $\lambda_{2}$.
\begin{lemma}
\label{lemma:der_bound}
$\hat{\mathbf{x}}(\lambda_{2})$ is infinitely differentiable w.r.t. $\lambda_{2}$ and, for any integer $p$, there exists a constant $K'$ such that its elementwise p-th derivative, denoted $D_{\lambda_{2}}^{(p)}\hat{\mathbf{x}}(\lambda_{2})$ verifies, almost surely
\begin{equation}
\label{use_bound}
\frac{1}{N}\norm{D_{\lambda_{2}}^{(p)}\hat{\mathbf{x}}(\lambda_{2})}_{2}^{2} \leqslant K' 
\end{equation}
Furthermore, $D^{(p)}_{\lambda_{2}}\hat{\mathbf{x}}(\lambda_{2})$ is a Lipschitz function of $\hat{\mathbf{x}}(\lambda_{2})$.
\end{lemma}
\begin{proof}
Recall the strongly convex problem, for any finite N,
\begin{equation}
\mathbf{\hat{x}}(\lambda_{2},\tilde{\lambda}_{2}) = \argmin_{\mathbf{x} \in \mathcal{X}} \tilde{g}(\mathbf{F}\mathbf{x},\mathbf{y})+f(\mathbf{x})+\frac{\lambda_{2}}{2}\norm{\mathbf{x}}_{2}^{2}
\end{equation}
where we absorbed $\tilde{\lambda}_{2}$ in $\tilde{g}$ as we are only interested in prolonging on $\lambda_{2}$. \\
The optimality condition then uniquely defines $\mathbf{\hat{x}}(\lambda_{2})$ of each value of $\lambda_{2}$ and reads : 
\begin{equation}
\mathbf{F}^{\top}\nabla \tilde{g}(\mathbf{F}\mathbf{\hat{x}}(\lambda_{2}),\mathbf{y})+\nabla f(\mathbf{\hat{x}}(\lambda_{2}))+\lambda_{2}\mathbf{\hat{x}}(\lambda_{2})=0
\end{equation}
The function $\mathbf{F}^{\top}\nabla \tilde{g}(\mathbf{F}\cdot,\mathbf{y})+\nabla f(\cdot)+\lambda_{2} \cdot$ is real analytic in $\mathbb{R}^{N}$ and its Jacobian $ \mathbf{F}^{\top}\mathcal{H}_{\tilde{g}}\mathbf{F}+\mathcal{H}_{f}+\lambda_{2}\mathbb{I}_{N}$ is non singular since f and $\tilde{g}$ are convex. The implicit function theorem \cite{krantz2002primer} then ensures that, at any finite $N >0$, the function $\mathbf{\hat{x}}(\lambda_{2})$ is elementwise real analytic in $\lambda_{2}$. We can now prove the lemma with an induction. \\
\paragraph{Initialization}
Owing to assumption \ref{main_assum}, we have almost surely
\begin{equation}
    \lim_{N \to \infty} \frac{1}{N}\norm{\hat{\mathbf{x}}(\lambda_{2})}_{2}^{2} \leqslant K' \quad
\end{equation}
and the identity is a Lipshchitz function of $\hat{\mathbf{x}}(\lambda_{2})$
The function of $\lambda_{2}$ defined by :
\begin{equation}
\lambda_2 \mapsto \nabla \tilde{g}(\mathbf{F}\mathbf{\hat{x}}(\lambda_{2}),\mathbf{y})+\nabla f(\mathbf{\hat{x}}(\lambda_{2}))+\lambda_{2}\mathbf{\hat{x}}(\lambda_{2})
\end{equation}
is always zero valued from the definition of $\mathbf{\hat{x}}(\lambda_{2})$, thus all its derivatives are zero.
Taking the first derivative with respect to $\lambda_{2}$ yields:
\begin{align}
\label{partial_init}
    &(\mathbf{F}^{T}\mathcal{H}_{\tilde{g}}(\mathbf{F}\mathbf{\hat{x}}(\lambda_{2}),\mathbf{y})\mathbf{F}+\mathcal{H}_{f}(\mathbf{\hat{x}}(\lambda_{2}))+\lambda_{2} \mathbf{I}_{N})D\mathbf{\hat{x}}(\lambda_{2}) \notag\\
    &\hspace{5.5cm}+\mathbf{\hat{x}}(\lambda_{2}) = 0 
\end{align}
where $D^{p}$ is the $(N \times 1)$ dimensional element-wise p-th differential of $\hat{\mathbf{x}}(\lambda_{2})$. We then define the operator
\begin{equation*}
    \mathcal{O} : \bigg\{
    \begin{array}{l}
        \mathbb{R} \to \mathbb{R}^{N \times N} \\
        \lambda_2 \mapsto  \mathbf{F}^{T}\mathcal{H}_{\tilde{g}}(\mathbf{F}\mathbf{\hat{x}}(\lambda_{2}),\mathbf{y})\mathbf{F}+\mathcal{H}_{f}(\mathbf{\hat{x}}(\lambda_{2}))+\lambda_{2} \mathbf{I}_{N}.
    \end{array}
\end{equation*}
We obtain a simple expression for $D\mathbf{\hat{x}}(\lambda_{2})$
\begin{equation}
D\mathbf{\hat{x}}(\lambda_{2}) = -\mathcal{O}^{-1}(\lambda_{2})\mathbf{\hat{x}}(\lambda_{2})
\end{equation}
Since $f$ and $g$ are convex, the operator norm of $\mathcal{O}^{-1}(\lambda_{2})$ is bounded with probability one, and $D\mathbf{\hat{x}}(\lambda_{2})$ is a Lipschitz function of $\mathbf{\hat{x}}(\lambda_{2})$ where $\frac{1}{N}\norm{D\mathbf{\hat{x}}(\lambda_{2})}_{2}^{2}$ is almost surely bounded. \quad \\
\paragraph{Induction step}
Assume the property is verified up to $p-1$.
For higher order derivatives, applying
Leibniz's rule on Eq.\eqref{partial_init} gives, denoting $\mathcal{O}^{(i)}(\lambda_{2})$ the i-th derivative of $\mathcal{O}(\lambda_{2})$, for the (p-1)-th derivative of \eqref{partial_init} : 
\begin{equation}
    \sum_{i=0}^{p-1}\binom{p-1}{i}\mathcal{O}^{(i)}(\lambda_{2})D^{(p-i)}\mathbf{\hat{x}}(\lambda_{2})+D^{(p-1)}\mathbf{\hat{x}}(\lambda_{2}) = 0,
\end{equation}
such that
\begin{align}
    \sum_{i=1}^{p-1}\binom{p-1}{i}\mathcal{O}^{(i)}(\lambda_{2})D^{(p-i)}\mathbf{\hat{x}}(\lambda_{2})&+\mathcal{O}(\lambda_{2})D^{(p)}\mathbf{\hat{x}}(\lambda_{2})\notag\\
    &+D^{(p-1)}\mathbf{\hat{x}}(\lambda_{2})= 0
\end{align}
We obtain the recursion on the differentials of $\mathbf{\hat{x}}(\lambda_{2})$ :
\begin{align}
\label{partial_rec}
    D^{p}\mathbf{\hat{x}}(\lambda_{2}) = - \mathcal{O}^{-1}(\lambda_{2})\bigg(\sum_{i=1}^{p-1}\binom{p-1}{i}\mathcal{O}^{(i)}(\lambda_{2})D^{(p-i)}\mathbf{\hat{x}}(\lambda_{2})\notag\\
    &\hspace{-2cm}+D^{(p-1)}\mathbf{\hat{x}}(\lambda_{2})\bigg).
\end{align}
where the matrix inverse $\mathcal{O}^{-1}(\lambda_{2})$ is well defined for any $\lambda_{2}>0$ since $f$ and $g$ are convex.
Using proposition \ref{prop_fdb}, the assumption on the fast decay of the higher-order (larger than 2) derivatives of
$f$ and $g$, the bounded spectrum of the matrix $\mathbf{F}$, and the induction hypothesis, the operator norm of $\mathcal{O}^{(p)}(\lambda_{2})$ is bounded with probability one for any $p \in \mathbb{N}$, $D^{(p)}\hat{\mathbf{x}}(\lambda_{2})$ is a Lipschitz function of $\hat{\mathbf{x}}(\lambda_{2})$ as a finite sum of Lipschitz functions of $\hat{\mathbf{x}}(\lambda_{2})$, and its averaged squared norm is bounded almost surely. This concludes the induction.
\end{proof}
\begin{lemma}
Under assumption \ref{main_assum}, the function $\psi(\lambda_{2})$ defined as
\begin{align}
    \psi : \mathbb{R} &\to \mathbb{R} \\
    \lambda_{2} &\to  \lim_{N \to \infty} \frac{1}{N}\sum_{i=1}^{N}\phi(x_{0,i},\hat{x}_{i}(\lambda_{2}))
\end{align}
is real analytic for $\lambda_{2}>0$.
\end{lemma}
\begin{proof}
Since $\phi$ is pseudo Lipschitz of order $2$, there exists a constant $C_{\phi}$ such that, for any $x \in \mathbb{R}$, $\phi(x) \leqslant C_{\phi}(1+x^{2})$. Thus :
\begin{equation}
   \lim_{N \to \infty} \abs{\psi(\lambda_{2})}\leqslant \lim_{N \to \infty}\frac{C_{\phi}}{N}(1+\norm{\hat{\mathbf{x}}(\lambda_{2})}_{2}^{2})
\end{equation}
which is almost surely bounded. By assumption, the boundedness of $\psi$ is enough to obtain its convergence.
For the first derivative, the pseudo-Lipschitz property ensures that there exists a constant $C_{\phi}'$ such that, for any $x \in \mathbb{R}$, $\abs{\frac{d\phi}{dx}(x)} \leqslant C_{\phi}'(1+\abs{x})$. Then
\begin{equation}
    \abs{\frac{d}{d\lambda_{2}}\phi(\hat{x}(\lambda_{2}))} \leqslant C^{'}_{\phi}\abs{\frac{d}{d\lambda_{2}}\hat{x}(\lambda_{2})}\left(1+\abs{\hat{x}}(\lambda_{2})\right)
\end{equation}
so there exists a constant $C_{\psi}'$ such that
\begin{equation}
    \lim_{N \to \infty}D\psi(\lambda_{2}) \leqslant \lim_{N \to \infty}\frac{1}{N}C_{\psi}'\left(\norm{D\hat{\mathbf{x}}(\lambda_{2})}_{2}+\norm{D\hat{\mathbf{x}}(\lambda_{2})}_{2}\norm{\hat{\mathbf{x}}(\lambda_{2})}_{2}\right)
\end{equation}
which is almost surely bounded. We have also proved in the previous lemma that $D\hat{\mathbf{x}}(\lambda_{2})$ is a Lipschitz function of $\lambda_{2}$, thus $D\psi(\lambda_{2})$ is a PL2 function of $\hat{\mathbf{x}}(\lambda_{2})$ and its limit exists according to Assumption \ref{main_assum} (c). For the higher order derivatives, we use proposition \ref{prop_fdb} to obtain, for any coordinate $1\leqslant i \leqslant n$ :
\begin{align}
    \abs{\frac{d^{(p)}}{d\lambda_{2}^{(p)}}\phi(\hat{x}_{i}(\lambda_{2}))} &= \sum \frac{p!}{k_{1}!k_{2}!...k_{p}!}\phi^{(k)}\left(\hat{x}_{i}(\lambda_{2})\right)\left(\frac{\hat{x}_{i}^{(1)}(\lambda_{2})}{1!}\right)^{k_{1}}\left(\frac{\hat{x}_{i}^{(2)}(\lambda_{2})}{2!}\right)^{k_{2}}...\left(\frac{\hat{x}_{i}^{(p)}(\lambda_{2})}{p!}\right)^{k_{p}} \notag 
\end{align}
The assumption on the higher order derivatives of $\phi$ from Theorem \ref{main_th} and Lemma \ref{lemma:der_bound} implies that the term \\
$\phi^{(k)}\left(\hat{x}_{i}(\lambda_{2})\right)\left(\frac{\hat{x}_{i}^{(1)}(\lambda_{2})}{1!}\right)^{k_{1}}\left(\frac{\hat{x}_{i}^{(2)}(\lambda_{2})}{2!}\right)^{k_{2}}...\left(\frac{\hat{x}_{i}^{(p)}(\lambda_{2})}{p!}\right)^{k_{p}}$
has bounded absolute value with probability one, for all coordinates $i$. Using the characterization of real analytic functions and assumption \ref{main_assum} (c) from proposition \ref{analytic_charac}, this concludes the proof.
\end{proof}
\subsection{Analytic continuation to \texorpdfstring{$(\tilde{\lambda}_{2},\lambda_{2}) \in \mathbb{R}^{*}_{+}\times \mathbb{R}_{+}^{*}$}{real-analytic2}}
From assumption \ref{main_assum}, the set of fixed point equations from Theorem \ref{main_th} admit a unique solution for any $\lambda_{2},\tilde{\lambda}_{2}$. Additionally, the implicit function theorem \cite{krantz2002primer} can also be applied to the set of fixed point equations from Theorem \ref{main_th} regarding the dependencies in $\lambda_{2},\tilde{\lambda}_{2}$ to show that each quantity involved is real analytic in $\lambda_{2},\tilde{\lambda}_{2}$. At this point, we have two analytic functions, the observable and the one defined by the fixed point of the state evolution equations, that coincide for any $\lambda_{2} \in \left[\lambda_{2}^{*}, +\infty \right[$ and any $\tilde{\lambda}_{2}> 0$. We can now use the analytic continuation theorem \cite{krantz2002primer} to show that these functions remain equal for any $\lambda_{2}>0$ and for $\tilde{\lambda}_{2}>0$. This concludes the proof of Lemma \ref{analytic_error_lemma}.
\subsection{Real analytic approximation of strongly convex problems}
Consider
\begin{align}
&\hat{\mathbf{x}}_{\epsilon}(\lambda_{2}) = \argmin_{\mathbf{x} \in \mathbb{R}^{N}} \thickspace \tilde{g}_{\epsilon}(\mathbf{F}\mathbf{x},\mathbf{y})+f_{\epsilon}(\mathbf{x})+\frac{\lambda_{2}}{2}\norm{\mathbf{x}}_{2}^{2} \\
&\hat{\mathbf{x}}(\lambda_{2}) = \argmin_{\mathbf{x} \in \mathbb{R}^{N}} \thickspace \tilde{g}(\mathbf{F}\mathbf{x},\mathbf{y})+f(\mathbf{x})+\frac{\lambda_{2}}{2}\norm{\mathbf{x}}_{2}^{2} 
\end{align}
where $g_{\epsilon},f_{\epsilon}$ are real analytic approximations of the loss $g$ and regularizer $f$ verifying assumption \ref{main_assum}(e).
To relax the analytic approximation, we need to prove the following equality.
\begin{equation}
    \lim_{\epsilon \to 0}\lim_{N \to \infty} \frac{1}{N}\sum_{i=1}^{N}\phi(\hat{x}_{\epsilon,i}(\lambda_{2})) = \lim_{N \to \infty} \frac{1}{N}\sum_{i=1}^{N}\phi(\hat{x}_{i}(\lambda_{2}))
\end{equation}
Under assumption \ref{main_assum} (c) and owing to the definition of PL2 functions, it is sufficient to prove 
\begin{equation}
\lim_{\epsilon \to 0}\lim_{N \to \infty}\frac{1}{N}\norm{\mathbf{\hat{x}}_{\epsilon}(\lambda_{2})-\mathbf{\hat{x}}(\lambda_{2})}_{2}^{2} = 0
\end{equation}
Denote $\mathcal{C}$ the cost function $\tilde{g}(\mathbf{F}.,\mathbf{y})+f(.)$ and its real analytic counterpart $\mathcal{C}_{\epsilon}$ the cost function $\tilde{g}_{\epsilon}(\mathbf{F}.,\mathbf{y})+f_{\epsilon}(.)$.
\begin{align}
    \forall \mathbf{x} \in \mathbb{R}^{d} \quad \lim_{\epsilon \to 0} \mathcal{C}_{\epsilon}(\mathbf{x}) = \mathcal{C}(\mathbf{x})
\end{align}
Since minimizers of convex functions are fixed points of the corresponding proximity operators, it holds that
\begin{align}
    &\frac{1}{N}\norm{\hat{\mathbf{x}}_{\epsilon}(\lambda_{2})-\hat{\mathbf{x}}(\lambda_{2})}_{2}^{2} = \frac{1}{N}\norm{\mbox{prox}_{\mathcal{C}_{\epsilon}(.)+\frac{\lambda_{2}}{2}\norm{.}_{2}^{2}}(\hat{\mathbf{x}}_{\epsilon}(\lambda_{2}))-\mbox{prox}_{\mathcal{C}(.)+\frac{\lambda_{2}}{2}\norm{.}_{2}^{2}}(\hat{\mathbf{x}}(\lambda_{2}))}_{2}^{2} \\
    &\leqslant \frac{1}{N}\norm{\mbox{prox}_{\mathcal{C}_{\epsilon}(.)+\frac{\lambda_{2}}{2}\norm{.}_{2}^{2}}(\hat{\mathbf{x}}_{\epsilon}(\lambda_{2}))-\mbox{prox}_{\mathcal{C}_{\epsilon}(.)+\frac{\lambda_{2}}{2}\norm{.}_{2}^{2}}(\hat{\mathbf{x}}(\lambda_{2}))}_{2}^{2} \notag \\
    &\hspace{5cm}+\frac{1}{N}\norm{\mbox{prox}_{\mathcal{C}_{\epsilon}(.)+\frac{\lambda_{2}}{2}\norm{.}_{2}^{2}}(\hat{\mathbf{x}}(\lambda_{2}))-\mbox{prox}_{\mathcal{C}(.)+\frac{\lambda_{2}}{2}\norm{.}_{2}^{2}}(\hat{\mathbf{x}}(\lambda_{2}))}_{2}^{2}
\end{align}
The results from appendix \ref{subsec:app_Lip_const} show that proximity operators of strongly convex functions are contractions, thus their exists a positive constant $L_{\lambda_{2}}<1$ such that for any realisation of $\mathbf{F},\mathbf{x}^{0},\boldsymbol{\omega}_{0}$
\begin{align}
    \frac{1}{N}\norm{\hat{\mathbf{x}}_{\epsilon}(\lambda_{2})-\hat{\mathbf{x}}(\lambda_{2})}_{2}^{2} \leqslant \frac{1}{N}L_{\lambda_{2}}\norm{\hat{\mathbf{x}}_{\epsilon}(\lambda_{2})-\hat{\mathbf{x}}(\lambda_{2})}_{2}^{2}+\frac{1}{N}\norm{\mbox{prox}_{\mathcal{C}_{\epsilon}(.)+\frac{\lambda_{2}}{2}\norm{.}_{2}^{2}}(\hat{\mathbf{x}}(\lambda_{2}))-\mbox{prox}_{\mathcal{C}(.)+\frac{\lambda_{2}}{2}\norm{.}_{2}^{2}}(\hat{\mathbf{x}}(\lambda_{2}))}_{2}^{2}
\end{align}
Furthermore, the function $\mbox{prox}_{\mathcal{C}_{\epsilon}(.)+\frac{\lambda_{2}}{2}\norm{.}_{2}^{2}}(.)$ converges uniformly to $\mbox{prox}_{\mathcal{C}(.)+\frac{\lambda_{2}}{2}\norm{.}_{2}^{2}}(.)$ when $\epsilon \to 0$, and thus 
\begin{equation}
    \lim_{\epsilon \to 0} \lim_{N \to \infty} \frac{1}{N}\norm{\mbox{prox}_{\mathcal{C}_{\epsilon}(.)+\frac{\lambda_{2}}{2}\norm{.}_{2}^{2}}(\hat{\mathbf{x}}(\lambda_{2}))-\mbox{prox}_{\mathcal{C}(.)+\frac{\lambda_{2}}{2}\norm{.}_{2}^{2}}(\hat{\mathbf{x}}(\lambda_{2}))}_{2}^{2} = 0
\end{equation}
which gives 
\begin{equation}
    \lim_{\epsilon \to 0} \lim_{N \to \infty} \frac{1}{N}\norm{\hat{\mathbf{x}}_{\epsilon}(\lambda_{2})-\hat{\mathbf{x}}(\lambda_{2})}_{2}^{2} \leqslant L_{\lambda_{2}}\lim_{\epsilon \to 0} \lim_{N \to \infty} \frac{1}{N}\norm{\hat{\mathbf{x}}_{\epsilon}(\lambda_{2})-\hat{\mathbf{x}}(\lambda_{2})}_{2}^{2}.
\end{equation}
Since $L_{\lambda_{2}} <1$, this implies
\begin{equation}
    \lim_{\epsilon \to 0} \lim_{N \to \infty} \frac{1}{N}\norm{\hat{\mathbf{x}}_{\epsilon}(\lambda_{2})-\hat{\mathbf{x}}(\lambda_{2})}_{2}^{2} = 0
\end{equation}
\subsection{Continuous extension to $\tilde{\lambda}_{2} = 0$}
Making the dependence on $\tilde{\lambda}_{2}$ explicit, define
\begin{align}
&\hat{\mathbf{x}}(\tilde{\lambda}_{2},\lambda_{2}) = \argmin_{\mathbf{x} \in \mathbb{R}^{N}} \thickspace g(\mathbf{F}\mathbf{x},\mathbf{y})+f(\mathbf{x})+\frac{\lambda_{2}}{2}\norm{\mathbf{x}}_{2}^{2}+\frac{\tilde{\lambda}_{2}}{2}\norm{\mathbf{F}\mathbf{x}}_{2}^{2}\ \\
&\hat{\mathbf{x}}(0,\lambda_{2}) = \argmin_{\mathbf{x} \in \mathbb{R}^{N}} \thickspace g(\mathbf{F}\mathbf{x},\mathbf{y})+f(\mathbf{x})+\frac{\lambda_{2}}{2}\norm{\mathbf{x}}_{2}^{2} 
\end{align}
Both cost functions defining $\hat{\mathbf{x}}(\tilde{\lambda}_{2},\lambda_{2}),\hat{\mathbf{x}}(0,\lambda_{2})$ are strongly convex for any $\lambda_{2}>0$. We can then use the same argument as in the previous subsection C to conclude 
\begin{equation}
\lim_{\tilde{\lambda}_{2} \to 0} \lim_{N \to \infty} \frac{1}{N}\norm{\hat{\mathbf{x}}(\tilde{\lambda}_{2},\lambda_{2})-\hat{\mathbf{x}}(0,\lambda_{2})}_{2}^{2} = 0
\end{equation}
\subsection{Continuous extension to \texorpdfstring{$\lambda_{2}=0$}{cont-ext}}
For $\tilde{\lambda}_{2} = 0$, the estimator $\hat{\mathbf{x}}(\lambda_{2})$ is still unique for any $\lambda_{2}>0$. We now need to study the limiting ridgeless estimator
\begin{equation}
    \lim_{\lambda_{2}\to 0} \argmin_{\mathbf{x} \in \mathcal{X}} g(\mathbf{F}\mathbf{x},\mathbf{y})+f(\mathbf{x})+\frac{\lambda_{2}}{2}\norm{\mathbf{x}}_{2}^{2}
\end{equation}
for functions $f,g$ that may not be strictly convex. To do so we will use Theorem 26.20 from \cite{bauschke2011convex}, which is reminded in appendix \ref{appendix : prox_prop}, proposition \ref{approx_l2}. Under assumption \ref{main_assum} and since the $l_{2}$ norm is strongly convex thus uniformly convex, we have, denoting $\hat{\mathbf{x}}_{0}$ the unique least $l_{2}$ norm element in $\argmin_{\mathbf{x} \in \mathcal{X}} g(\mathbf{F}\mathbf{x},\mathbf{y})+f(\mathbf{x})$,
\begin{equation}
    \lim_{\lambda_{2} \to 0} \hat{\mathbf{x}}(\lambda_{2}) = \hat{\mathbf{x}}_{0}
\end{equation}
We can therefore uniquely define the continuous extension of any continuous observable $\phi$ of $\hat{\mathbf{x}}(\lambda_{2})$ such that $\phi(\lambda_{2}=0) = \phi(\hat{\mathbf{x}}_{0})$. Then this observable and the corresponding function implicitly defined by the set of fixed point equations are continuous on $[0,+\infty[$ and equal for any $\lambda_{2} \in ]0,+\infty[$, and thus also equal at $\lambda_{2} = 0$ using the definition of continuity and the fact that $]0,+\infty[$ is dense in $[0,+\infty[$.

\subsection{Real analytic approximation of usual cost functions with fast decaying higher-order derivatives}
\label{subsec:app_approx_fast}
In this section, we show that any combination of the square, logistic and hinge loss with $\ell_{1}$ or $\ell_{2}$ verifies Assumption \ref{main_assum} (e), i.e. they can be approximated with real analytic functions whose second derivatives have higher-order derivatives that decrease faster than any polynomial. The square loss and $\ell_{2}$ immediately verify these assumptions. Assuming $y=1$ without loss of generality, the second derivative of the logistic loss is given by 
\begin{equation}
g''(x) = \frac{\mbox{exp}(x)}{(1+\mbox{exp}(x))}.
\end{equation}
All higher order derivatives will be a polynomial in $\mbox{exp}(x)$ divided by a higher order polynomial in $\mbox{exp}(x)$ plus one. Thus, for any sign of $x$, higher-order derivatives of the logistic loss will decrease exponentially fast when the absolute value of $x$ goes to infinity.
We now turn to the $\ell_{1}$ penalty. Real analytic approximations of functions may be constructed by considering their convolution with a Gaussian kernel, which is also known as the Weierstrass transform. Denoting $\mathcal{W}_{\epsilon}\left[f\right]$ the Weierstrass transform of a function $f$ with parameter $\epsilon>0$, we obtain for the $\ell_{1}$ penalty
\begin{align}
    \mathcal{W}_{\epsilon}\left[\abs{.}\right](x) &= \frac{1}{\sqrt{2\pi\epsilon}}\int_{-\infty}^{+\infty}\abs{u}\exp(-\frac{1}{2\epsilon}(u-x)^{2})du \\
    & = \frac{1}{\sqrt{2\pi\epsilon}}\left(2 \epsilon \exp(-\frac{1}{2\epsilon}x^{2})+2x\int_{0}^{x}\exp(-\frac{1}{2\epsilon}u^{2})du\right)
\end{align}
whose second derivative reads 
\begin{equation}
    \frac{d^{2}}{dx^{2}}\mathcal{W}_{\epsilon}\left[\abs{.}\right](x)=\frac{\sqrt{2}}{\sqrt{\pi\epsilon}}\exp(-\frac{1}{2\epsilon}x^{2})
\end{equation}
thus $\mathcal{W}_{\epsilon}\left[\abs{.})\right]$ is strongly convex and its higher order derivatives all decay faster than any finite order polynomial. A similar computation shows that, for the hinge loss, 
\begin{align}
    \mathcal{W}_{\epsilon}\left[\max(0,1-.)\right](x) &= \frac{1}{\sqrt{2\pi\epsilon}}\int_{-\infty}^{+\infty}\max(0,1-u)\exp(-\frac{1}{2\epsilon}(u-x)^{2})du \\
    & = \frac{1}{\sqrt{2\pi\epsilon}}\left((1-x)\sqrt{\frac{\pi\epsilon}{2}}+\epsilon \exp(-\frac{1}{2\epsilon}(1-x)^{2})+(1-x)\int_{0}^{x}\exp(-\frac{1}{2\epsilon}(1-x)^{2})du\right)
\end{align}
whose second derivative reads 
\begin{equation}
    \frac{d^{2}}{dx^{2}}\mathcal{W}_{\epsilon}\left[\max(0,1-.)\right](x)=\frac{1}{\sqrt{2\pi\epsilon}}\exp(-\frac{1}{2\epsilon}(1-x)^{2})
\end{equation}
Thus the hinge loss and $\ell_{1}$ penalty verify Assumption \ref{main_assum} (e).
\end{document}